%% file: ms.tex
\documentclass[12pt, lot, lof, los, loa]{southalabamathesis}
\pdfoutput=1
\usepackage{cite}
\usepackage{array} 
\usepackage[labelsep=period]{caption}
\usepackage{fixltx2e}
\usepackage{soul}
\usepackage[hidelinks]{hyperref}
\usepackage{amsmath, amsthm, amssymb, amsfonts}
\usepackage{algorithmic}
\usepackage{algorithm}
\usepackage{multirow}
\setul{1pt}{.4pt}
\usepackage{afterpage}
  
\usepackage{graphics}
\usepackage{graphicx}
\graphicspath{{static//}}
\usepackage{subfig}
\usepackage{indentfirst}
\usepackage{url}
\hypersetup{colorlinks=false,bookmarksnumbered,pdftitle=\@title,pdfauthor=\@author}

\newcommand{\T}{\mathrm{T}}
\newcommand{\lvec}[2]{\begin{bmatrix} \underbrace{#1 \ldots #1}_{#2\text{-times}} \end{bmatrix}}

\newtheorem{lemma}{Lemma}
\newtheorem{theorem}{Theorem}
\newtheorem{corollary}{Corollary}

\abstractopening{Brown, Christopher, Scott, Ph.D., University of South Alabama, May 2020. Local Model Feature Transformations. Chair of Committee: Ryan Benton, Ph.D.}

\abstract{
    \input{sections/abstract/abstract}
}

\acknowledgements{
    \input{sections/acknowledgements/acknowledgements}

}

\bio{
    \input{sections/bio/bio}
}

\title{Local Model Feature Transformations}

\submitted{April 2019}
\expectedgraddate{May 2020}
\copyrightyear{2019}  
\author{Christopher Scott Brown}
\priordegrees{
    B.S. Mathematics: University of South Alabama 2006,
    M.S. Mathematics: University of South Alabama 2007
}
\documenttype{A Dissertation}
\authorreversed{Brown, C. Scott}
\adviser{Dr. Ryan Benton}  
\degreetitle{Doctor of Philosophy}
\degreeabbreviation{Ph.D.}

\degreetype{Computing}
\institution{The University of South Alabama}
\school{School of Computing}
\department{The Department of Computer Science}

\signatories{
    Chair of Thesis Committee: Dr. Ryan Benton,
    Committee Member: Dr. David Bourrie,
    Committee Member: Dr. H. Frazier Bindele,
    Committee Member: Dr. Tom Johnsten,
    Director of Graduate Studies: Dr. Debra Chapman,
    Dean of the Graduate School: Dr. J. Harold Pardue
}

\makenomenclature

\begin{document}

\makefrontmatter

\include{sections/intro/intro}

\include{sections/background/background}

\chapter{LOCAL GAUSSIAN PROCESS FEATURES} \label{chapter:localgpr}
\input{local_gpr/abstract}

\input{local_gpr/intro/intro}
\input{local_gpr/background/background}

\input{local_gpr/methodology/methodology}

\input{local_gpr/results/results}

\input{local_gpr/conclusion/conclusion}

\chapter{DECISION SURFACES OF LOCALIZED CLASSIFIERS} \label{chapter:localsvm}
\input{local_SVM_projection/abstract}

\input{local_SVM_projection/intro/intro}
\input{local_SVM_projection/background/background}

\input{local_SVM_projection/methodology/methodology}

\input{local_SVM_projection/results/results}

\input{local_SVM_projection/conclusion/conclusion}

\chapter{LOCAL WORD2VEC MODELS OVER TIME} \label{chapter:localw2v}
\input{local_W2V_v_time/abstract}

\input{local_W2V_v_time/intro/intro}
\input{local_W2V_v_time/background/background}

\input{local_W2V_v_time/methodology/methodology}

\input{local_W2V_v_time/results/results}

\input{local_W2V_v_time/conclusion/conclusion}

\chapter{LOCAL ORTHOGONAL QUADRIC REGRESSION} \label{chapter:localodq}
\input{local_ODQ_regression/abstract}

\input{local_ODQ_regression/intro/intro}

\input{local_ODQ_regression/background/background}

\input{local_ODQ_regression/methodology/methodology}

\input{local_ODQ_regression/results/results}

\input{local_ODQ_regression/conclusion/conclusion}

\include{sections/conclusion/conclusion}

\bibliographystyle{ieeetr}
\bibliography{citations,local_ODQ_regression/citations,local_gpr/citations,local_SVM_projection/citations,local_W2V_v_time/citations}

\appendix

\section{Gaussian Process Weighting} \label{appendix:gprweights}
\input{local_gpr/appendix1/weighted_gpr}

\makebio

\end{document}

%% file: sections/abstract/abstract.tex
Local learning methods are a popular class of machine learning algorithms.
 The basic idea for the entire cadre is to choose some non-local model family, to train many of them on small sections of neighboring data, and then to `stitch' the resulting models together in some way.
 Due to the limits of constraining a training dataset to a small neighborhood, research on locally-learned models has largely been restricted to simple model families.
 Also, since simple model families have no complex structure by design, this has limited use of the individual local models to predictive tasks.
 Recently, methods that do not directly employ local models, but that extract local information from data have proven to be very powerful in a number of tasks.
 Unfortunately, there has been little research on extending the local modeling schema to more complex model families and non-predictive tasks may prove valuable in many contexts.
 We hypothesize that, using a sufficiently complex local model family, various properties of the individual local models, such as their learned parameters, can be used as features for further learning.
 This dissertation improves upon the current state of research and works toward establishing this hypothesis by investigating algorithms for localization of more complex model families and by studying their applications beyond mere predictions.

We summarize this generic technique of using local models as a feature extraction step with the term ``local model feature transformations.''
 The basic idea is that models that are trained on a local neighborhood of data `encode' some information about the local structure of that neighborhood.
 This information can be extracted and used as features for input to other machine learning algorithms in a transfer learning scheme.
 The general intuition outlined above gives rise to a recurring prediction of our hypothesis that we will test in many experiments throughout this document.
 Namely, that existing machine learning methods will provide superior results when supplied additional features extracted from context-appropriate local models.

As a precursor to utilizing local models for feature extraction, we must first establish the means to localize various model families under consideration.
 Specifically, in this document, we extend the local modeling paradigm to Gaussian processes, orthogonal quadric models and word embedding models.
 This document also extends the existing theory for localized linear classifiers.
 In order to test the prediction of our hypothesis stated above, we will use the local model-based features as inputs to baseline machine learning models in a variety of domains.
 This document demonstrates applications of local model feature transformations to epileptic event classification from EEG readings, activity monitoring via chest accelerometry, 3D surface reconstruction, 3D point cloud segmentation, handwritten digit classification and event detection from Twitter feeds.
 
The individual papers included herein may be taken as techniques for extracting features from raw data in the particular application context in which they are studied.
 However, we encourage the reader to consider them as individual pieces of evidence of the unifying hypothesis of the usefulness of local model feature transformations in general.

%% file: sections/acknowledgements/acknowledgements.tex
I am very grateful to the faculty and everyone involved with the graduate program in the school of computing for their support: educationally, emotionally and financially for the duration of my graduate work.
 I am especially grateful to my advisor Ryan Benton for his assistance, his advocacy, and most of all for his loyalty and firm faith in me and my research.

I am also indebted beyond words to my mother and father, Louise and Robert Brown, who have always incited my pursuit of knowledge, especially with regards to computers and mathematics.
 Lastly, I would like to thank my Libby for being a never ending source of inspiration to be my best self, and for her encouragement during my continued education from beginning to end and beyond.

%% file: sections/bio/bio.tex
Christopher Scott Brown was born and raised in Mobile, AL.
 He received his B.S and M.S in Mathematics, and PhD in Computing from the University of South Alabama.

%% file: sections/intro/intro.tex
\chapter{INTRODUCTION} \label{chapter:intro}

Feature engineering is a common and important step in any machine learning task.
 Feature engineering is the process of extracting meaningful variables from raw data that are essential for machine learning algorithms to successfully learn to make predictions \cite{anderson2013brainwash}.
 It has been noted that simple models based on millions of features can outperform more complex models that try to discover general rules \cite{halevy2009unreasonable}.
 Classically, feature engineering is a human-in-the-loop process requiring extensive domain knowledge.
 Recently, however, much of the focus has shifted to feature \emph{learning}, in which the engineering of useful features is an automated process.
 Feature learning has been used to great effect in image processing where models capable of learning local relationships in an image have far outpaced other methods for image classification tasks.
 Direct parities of these same methods have been less effective for generic datasets, but their success still suggests that methods for learning local structure are possibly the best route in many other fields.
 
Mostly, existing feature learning methods focus on global attributes of a dataset.
 For example, principal component analysis finds features in which the variance computed over the entire training set is large.
 Neural networks learn features based on non-linear relationships between dimensions in the raw data.
 However, the ability to extract such global features relies on an implicit assumption that data are generated from a single underlying process that is reasonably well-behaved globally.
 For principal component analysis, well-behaved means that learned features are based on global linear relationships between raw data variables.

This `globality' assumption can be problematic when building supervised models, especially in the presence of potentially unobserved confounding variables.
 Although such variables need not necessarily manifest themselves in predictable ways, the effects of such variables can often be seen in the distribution and behavior of observed variables.
 The discovery of these variables and their effects are often the subject of unsupervised learning methods such as clustering, anomaly/outlier detection and change point detection.
 These unsupervised learning tasks are popular because often the assumption that a single simple mathematical formula in the observed variables holds everywhere in a domain is too strong.
 For example, the assumption of linear regression that variables are related linearly, globally, is very restrictive \cite{bender2009introduction}. 


One way to skirt these issues when building supervised models is to employ nonparametric methods such as local learning \cite{atkeson1997locally}.
 Local learning is used when learning a global function is too complex or too costly, and hinges on the fact that making a prediction at $x$ typically only requires knowledge of points `near' $x$.
 K-nearest-neighbors is one very popular example of a local learning model \cite{altman1992introduction}.
 More complex versions of these methods have seen success in robotics with motion-planning applications \cite{atkeson1991using}.
 This avoids the `globality' assumption problem since local assumptions are much less restrictive than global ones.
 For local linear regression, the assumption that variables are related linearly \emph{locally} is tantamount to an assumption that the function of interest is globally differentiable \cite{newton1687philosophiae}.
 This is a much weaker assumption than global linearity.

It is well-known that the predictions of local models can be used as an initial `smoothing' step on data \cite{cleveland1979robust}.
 So far, local learning theory has largely been confined to making predictions directly, and utilizing only simple model families.
 Unfortunately, this process is not without drawbacks, the most obvious of which is that, just as the predictions of global models may fail to consider the local context of the data, local models fail to consider the global context of the data.
 This dichotomy seems to suggest that we might desire to combine these two approaches to obtain a global model that uses the local models in some way to account for the local structure of training data.

In some circumstances, a potential way to achieve this would be to simply make an ensemble of global and local model predictions.
 Although the predictions of local models are useful, they discard much of the information contained in a local model.
 We propose that the parameters or other summarizing statistics of local models encode information about the local neighborhood of a point that can be useful in its own right, and which is not captured by simply obtaining a prediction.
 This concept of repurposing the intermediate calculations of a model trained for one purpose as a feature for a model created for a different but related purpose is known as transfer learning \cite{torrey2010transfer}.
 Unfortunately, simple families of models contain relatively little information that can be transferred to a new task.
 On the other hand, more complex model families have not been deeply studied in the context of local modelling.
 Thus, there has been some \cite{bay2004framework}, but very little research in using local models as a source task to improve capabilities of subsequent global modeling tasks.

We term the general process of learning features via the training of local models and the subsequent casting of those local models into real vectors a `local model feature transformation' (LMFT\nomenclature{LMFT}{Local Model Feature Transformation}).
 Since the LMFT process is designed as a feature transformation step, it can hypothetically be used as a precursor in conjunction with any number of existing learning methods, for any number of ultimate purposes, including supervised and unsupervised applications.
 The degree to which such features are useful for real world applications, outside of directly obtaining a prediction, remains largely unknown.

The a priori benefit of LMFT is that local model parameters and predictions include information about the local structure of the data.
 As a simple example of how local structure can be useful in learning, consider the task of predicting whether or not someone is in an urban or rural location from their latitude and longitude.
 Any simple supervised learner trained on data from the U.S. would have very poor accuracy when tested on data from, say, the Netherlands.
 On the other hand, suppose that, for each individual, we engineer the feature $d :=$ ``the average distance to the 1000 nearest neighbors".
 Using this feature, a global linear learner would have no trouble at all predicting if an individual is in an urban or rural area, regardless of the origin of the training and test data.
 If simplistic, this example illustrates the power of learning locally-relevant features for further learning, even when the local feature under consideration is not the primary focus of the problem.
 This suggests that ways to extract a priori meaningful properties summarizing local structure may be useful more broadly.
 Specifically, the parameters of locally learned models fit this description of `properties summarizing local structure'.
 This intuition is the subject of the current research.
 
In the proceeding chapters, we aim to provide a firm foundation for the use of local models as a feature extraction technique.
 This involves theoretical considerations for applying the local modeling paradigm to model families that are complex enough to give relevant yet non-trivial local features.
 This also involves investigation of possible use cases for such features, and the use of these features to improve upon existing methods.

Specifically, in chapters \ref{chapter:localgpr}-\ref{chapter:localodq} we will extend the local modeling paradigm to several model families.
 We will demonstrate that local model feature transformations are feasible for a wide range of data types.
 We will demonstrate that local model feature transformations are applicable to both supervised and unsupervised applications.
 We will show that features extracted via local modeling often align with concepts that we might a priori expect to be useful and informative in a particular context, and that are not trivially reflected in the raw data.
 Lastly, we will provide numerous examples where local model features improve upon the results of subsequent modeling when used in conjunction with raw data.

Chapter \ref{chapter:background} provides a background on the subject and formally defines local model feature transformations.
 Readers unfamiliar with the concept of local learning may find Sections \ref{section:locallearning} and \ref{section:loess} particularly informative.
 The particularly curious may also be interested in the introduction to the subject provided in \cite{atkeson1997locally}.
 Section \ref{section:localmodelparameters} describes work that is most closely related to our proposed research.
 Sections \ref{section:kernelmethods} and \ref{section:featurelearning} describe broad fields that provide some relevant background to the tasks of local learning and feature learning respectively.
 Section \ref{section:computervision} describes one particular task at which locally-derived features are the cornerstone of state-of-the-art machine learning algorithms.
 
Chapter \ref{chapter:localgpr} includes a paper detailing our research with local Gaussian process regression.
 The work extends the local modeling scheme to Gaussian process regression, requiring a non-trivial means for locally weighting training data.
 We employ local Gaussian process models trained to forecast time-series as a transfer learning step to extract features for time series classification and anomaly detection.

Chapter \ref{chapter:localsvm} presents a paper concerning our research with local support vector machines.
 Although local support vector machines are not entirely new to the literature, they were not thoroughly studied, and a means to find the decision surface of the localized classifier was notably missing.
 We devise an algorithm to discover the surface, and to find orthogonal projections onto the surface.
 We utilize these distances of individual points to the decision surface of the system of localized classifiers to derive pseudo-probabilistic features relevant to that point.
 We then use these features to extend local support vector machines to multiclass problems.

Chapter \ref{chapter:localw2v} presents a paper detailing our research with local word2vec models.
 The work extends the local modeling recipe to word embedding models, requiring novel methods to circumvent problems with non-convexity of the optimization surface.
 We gather a new Twitter dataset specifically designed to test models geared toward detecting emerging events.
 We employ local word2vec models trained to predict word context as a transfer learning step to extract features for time series anomaly detection in word usage.

Chapter \ref{chapter:localodq} includes a paper detailing our research with local orthogonal quadric regression models.
 The work extends the local modeling scheme to apply to quadric models trained with a sum of orthogonal distances loss objective.
 We use the method to extend existing surface reconstruction techniques.
 We further use the method as a transfer learning step to extract features for point cloud segmentation.

Lastly, chapter \ref{chapter:conclusions} provides a summarizing context to the work in the four preceding chapters.
 It includes a discussion of the challenges of employing the local modeling recipe for complex model families, how some of these challenges are inherent to the task of local modeling, and how some hinge crucially on the specific model family being studied.
 It further elaborates on the benefits of employing the local modeling recipe specifically as a feature extraction step in a transfer learning scheme, and under what circumstances one might practically wish to employ this method.
 It also includes a discussion of potential future work in the field, and points out  work in related fields that may be of use to future work with local models.


%% file: sections/background/background.tex
\chapter{BACKGROUND} \label{chapter:background}

This section serves the purposes of 1) summarize the background knowledge necessary to be conversant in the subject at hand; 2) illustrate the current state of knowledge on the subject in the research literature with the aim of pointing out where the current inquiry fills gaps or otherwise provides a novel contribution; and 3) briefly describe existing algorithms that, although not involving the same methods as the proposed research, compete for application toward the same or similar end goals.

\section{Kernel Methods} \label{section:kernelmethods}

Kernel Methods are a deep and wide field in machine learning that range from methods such as Gaussian Process Regression (GPR\nomenclature{GPR}{Gaussian process regression}) \cite{rasmussen2004gaussian} to extensions of methods involving distance or inner product calculations, such as Support Vector Machines (SVM\nomenclature{SVM}{Support vector machine}) \cite{drucker1999support}, to simpler local-style methods such as $k$-nearest neighbors \cite{altman1992introduction}.
 The term `kernel methods' sometimes describes a means of transforming an `ordinary' inner product defined on the domain, called a `kernel trick'.
 Specifically, when an algorithm involves only calculations that can be cast as certain basic arithmetical operations and inner products in a vector space, that inner product can be replaced by another calculation that meets the definition of an inner product in some alternate space.
 The resulting calculations can be seen as having found a mapping of the original vector space into the alternate one, performed the algorithm in that alternate space, and mapped back into the original space all without the overhead of actually casting vectors into that space.
 This is especially helpful when, for example, the alternate space is of infinite dimension, but the inner product in that space involves only a finite number of calculations.
 It is not even necessary to have a constructive means of casting raw vectors into the alternate space, or to have an intuitive description of what that alternate space is.
 Any function satisfying Mercer's condition \cite{mercer1909xvi} can be viewed as an inner product of two vectors after mapping into \emph{some} vector space.
 This `kernel trick' is often very useful for creating non-linear methods from linear ones such as SVMs \cite{boser1992training}.

A separate, but related class of kernel methods, which are the subject of our research, consist, loosely, of methods for weighting a dataset based on some distance metric on the domain.
 The canonical examples of such methods are kernel smoothing (Section \ref{section:loess}) and kernel density estimation \cite{terrell1992variable}.
 Since the Euclidean distance is easily defined in terms of the inner product, the terminology and some of the intuitions of the two concepts are related.
 Some works in local models define a `kernel' in a highly specialized way.
 For the sake of generalizability, but at risk of alienating readers familiar with more specialized notation, we attempt here to use notation that applies reasonably well to both fields.
 Since our research does not require a `kernel' to satisfy Mercer's condition, we use the term `kernel' loosely to refer to any similarity measure on a space $D$:
\begin{equation} \label{equation:kernel}
    K: D\times D \rightarrow {\rm I\!R}
\end{equation}
When the objective is to transform the Euclidean distance, kernels are sometimes written as a unary operator $K:{\rm I\!R} \rightarrow {\rm I\!R}$ that accepts as input distances (which are real) between two points in some space, or sometimes $K:D \rightarrow {\rm I\!R}$ taking as input a vector difference.
 Note that several of the below defined kernel functions can be written in such a way, by replacing the Euclidean distance $||x-y||$ with a single variable $u$.
 Commonly employed kernel functions include the linear kernel \cite{rasmussen2005gaussian}:
\begin{equation} \label{equation:linearkernel}
 K(x,y) = x^Ty
\end{equation}
the Gaussian kernel where $h$ is the `bandwidth' hyperparameter \cite{rasmussen2005gaussian}:
\begin{equation} \label{equation:gaussiankernel}
 K(x,y) = \mathrm{e}^{-||x-y||^2/h}
\end{equation}
the Tricube kernel, a $C^2$ approximation to the Gaussian with finite support \cite{cleveland1979robust}:
\begin{equation} \label{equation:tricubekernel}
 K(x,y) = (1-\frac{||x-y||^3}{h^3})^3
\end{equation}
the uniform kernel \cite{epanechnikov1969non}:
\begin{equation} \label{equation:uniformkernel}
 K(x,y) = \begin{cases} 
      1 & ||x - y||^2 \leq h \\
      0 & \text{else}
   \end{cases}
\end{equation}
the Dirichlet kernel \cite{levi1974geometric}, which is periodic, and highly weights both nearby points, and points distributed at a regular interval:
\begin{equation} \label{equation:dirichletkernel}
 K(x,y) = \frac{\sin((n+1/2)||x-y||^2)}{\sin(||x-y||^2/2)}
\end{equation}
the KNN\nomenclature{KNN}{$k$th-nearest neighbor} kernel, so named as $y_k$ is the $k$th nearest neighbor to $x$ in some observed dataset (a special case of variable-bandwidth kernels \cite{terrell1992variable}):
\begin{equation} \label{equation:knnkernel}
 K(x,y) = \begin{cases} 
      1 & ||x - y||^2 \leq ||x-y_k||^2 \\
      0 & \text{else}
   \end{cases}
\end{equation}
and many other options including sums, products and compositions of the above.
In order to ease the use of vector and matrix representations, we will adopt the following convention to apply kernels to multiple data points at once.
 If $X$ and $Y$ are matrices containing $n$ and $m$ row vectors, respectively, $K$ is an arbitrary kernel, and $x_0$ is the first row of $X$:
\begin{equation} \label{equation:matrixkernel}
 K(X,Y) := 
 \begin{bmatrix}
  K(x_0, y_0) & K(x_1, y_0) & \ldots & K(x_n, y_0) \\
  K(x_0, y_1) & K(x_1, y_1) & \ldots & K(x_n, y_1) \\
  \vdots      & \vdots      & \ddots & \vdots      \\
  K(x_0, y_m) & K(x_1, y_m) & \ldots & K(x_n, y_m) 
 \end{bmatrix}
\end{equation}      

The definition of `similarity' can vary significantly depending on the domain, which need not even consist of real vectors.
 Such considerations are addressed by choosing an appropriate kernel function for a given application, and by adjusting the kernel hyperparameters.
 Depending on the circumstances, kernel hyperparameters can be learned with techniques such as cross-validation.
 
Even for methods that do not inherently rely on an inner product or a distance calculation in their formulation, kernel functions can often be used to `localize' data processing methods to a query point $q$ in the domain.
 This hinges on the fact that these methods often admit a procedure for weighting the input data in the calculation.
 We can use a kernel function to provide these weights, weighting data that are closer to $q$ more highly than those at a distance.
 Such a method can be employed to compute, for example, moving averages.
 It can even be used to train `localized' versions of many machine learning and statistical models.
 `Localizing' machine learning models in such a way is often called `local learning'.

\section{Local Learning} \label{section:locallearning}

`Local learning' refers loosely to a collection of machine learning and statistical methods that ``locally adjust the capacity of the training system to the properties of the training set in each area of the input space.'' \cite{bottou1992local}
 Machine Learning methods described as `local' often fall under the purview of kernel methods.
 The term `local' is also sometimes used to refer to other methods such as splines.
 We will not be addressing such methods in this work, and use the term `local learning' to refer to kernel methods generally.

 The intuition of learning simple local properties of data at individual points by focusing on data in local neighborhood of each point is quite old. 
 Techniques such as `moving'(`running'/`rolling') averages on time series, which apply a weighting scheme to obtain an average of data sampled at `nearby' times, have been used since at least the 1800s \cite{altman1992introduction}.
 Many local learning techniques involve such simple summarizing statistics computed at a local neighborhood of a point.
 The Nadaraya-Watson estimator \cite{nadaraya1964estimating}, for example, computes a local weighted mean, and the $k$-nearest neighbors method computes a local mode \cite{altman1992introduction}.

We focus on a broad class of local methods that train individual models at various points $\{q_i\}$ across the domain, weighting data according to its proximity to the query point $q \in \{q_i\}$.
 Local regression techniques such as locally weighted scatterplot smoothing (LOESS\nomenclature{LOESS}{Locally-estimated scatterplot smoothing}) \cite{cleveland1979robust} described in Section \ref{section:loess} fall into this category.
 Note that `kernel trick' methods described in Section \ref{section:kernelmethods} do not, in general, fit this paradigm, and investigation of the properties of `kernel tricks' are not part of our current work.

For most parametric model families, we can `localize' model training and prediction according to a straightforward formula.
 Given a model family $F$ parametrized by $\theta \in \Theta$, trained via a real-valued loss function $L$ over a dataset $X \subset D$, the optimal parameters $\theta_{\text{opt}}$ are obtained by minimizing the total loss:
 $$ \theta_{\text{opt}} = \text{argmin}_\theta \sum_{x\in X} L_F(\theta, x) $$
 This gives a particular model $f_{\theta_{\text{opt}}} \in F$.
 Such a training procedure can be `localized' to a query point $q \in D$ by inclusion of a data weighting scheme $K$ prioritizing data `near' $q$, giving many such models, especially suited to prediction near the data on which they were trained:
\begin{equation} \label{equation:locallearning}
 \theta_q = \text{argmin}_\theta \sum_{x\in X} L_F(x, \theta) K(x, q)
\end{equation}
 A lengthy review of the literature on learning local models, including methods and applications, is given in \cite{atkeson1997locally}.

The resulting set of models are typically used to perform a prediction of some set of response variables $D_r$ from some set of predictor variables $D_p$, where $D_p \times D_r = D$.
 Since the individual models are only trained locally, and a global model is not obtained, these models are usually used to interpolate points `near' the training dataset $X_p \subset D_p$, and less often to extrapolate to other points in $D_p$.
 Despite the fact that the model families used to train the local models may be parametric, these methods are generally considered non-parametric since they do not involve any \emph{global} parameters.
 
When used for predictive purposes, and $x$ is some combination of dependent ($x_p$) and independent ($x_r$) variables, $K$ typically does not depend on the independent dimensions of $x$ and $q$.
 We make no such restrictions on the definition, but when appropriate will abuse the notation of Equation \ref{equation:locallearning} with functions $K$ defined over vectors that include only a subset of the dimensions of $D$.
 
Typically, $K$ is chosen as a non-increasing function of some distance metric between the points $x$ and $q$, which has the form of the `kernel' functions described in Section \ref{section:kernelmethods}.
 If $K$ ignores the query point $q$, then we call the model `global', and Equation \ref{equation:locallearning} degenerates to the `ordinary' training procedure above.
 Viz. the model is still weighted, but is no longer `locally' weighted.
 Otherwise we use the term `local', although we do not explicitly restrict $K$ to highly weight `nearby' points in the usual sense.
 This allows interesting notions of `local' as in, for example, periodic kernels such as the Dirichlet kernel (Equation \ref{equation:dirichletkernel}).

If we further insist that $K$ has finite support, i.e. that $K(x, q) := 0$ for a large proportion of the $x\in X$, we receive an added benefit of computational efficiency due to a vast reduction in the size of our training dataset for each local model.
 This has been a key feature of local learning methods for `lazy learning' applications \cite{aha1997lazy}, whereby a local model can be learned on-demand to make individual predictions.
 Kernels with finite support are also useful for online applications \cite{atkeson1997control}, since models are able to learn easily from new examples, essentially utilizing only the most recent data. 

The kernel function $K$ itself usually depends on some set of hyperparameters $\phi$ which may themselves be learned via, e.g. cross-validation \cite{dellaert1996recognizing}.
 The Gaussian kernel (Equation \ref{equation:gaussiankernel}) has the hyperparameter $h$, the `bandwidth' of the kernel.
 The KNN kernel (Equation \ref{equation:knnkernel}) has the hyperparameter $k$, the number of neighbors representing a `local' neighborhood.
 Thus, it is frequently the case in local learning where the primary focus is on learning the kernel function $K$ itself.
 For example, cross-validation is sometimes used for the $k$-nearest neighbors algorithm to learn the hyperparameter $k$ of the KNN kernel (Equation \ref{equation:knnkernel}).
 Locally linear embedding \cite{roweis2000nonlinear} involves learning weights for a weighted KNN kernel, where the weights are learned locally at each point in a sample.
 Convolutional neural networks \cite{krizhevsky2012imagenet} learn weights for what is essentially an asymmetric KNN kernel using the distance induced by the $L^\infty$ norm.
 More generally, choosing optimal kernel parameters has a long history in kernel density estimation \cite{terrell1992variable}, which is, in some sense, one of the most basic applications of kernel methods.

Note that kernel parameters are almost always learned \emph{globally}.
 Models based on simple local means using batteries of kernels with learned global parameters can be incredibly expressive.
 This is evidenced by recent work in computer vision discussed briefly in Section \ref{section:computervision}.
 Although tuning of kernel parameters will be necessary for our experiments as it is an important concern for any local learning application, it is not the primary topic of the current inquiry.

\section{LOESS} \label{section:loess}

Probably the simplest application of kernels to function interpolation is the Nadaraya-Watson estimator \cite{nadaraya1964estimating}, which essentially computes a locally weighted mean, defined at every point in the domain.
 More complex algorithms for local learning typically involve some manner of local parametric learning, and this class of algorithm are the subject of the current research.
 We include here a short section on the LOcally Estimated Scatterplot Smoothing (LOESS) procedure \cite{cleveland1979robust}, as it gives a simple illustration of the local learning paradigm.
 LOESS involves performing local ordinary least squares linear regression, and using the resulting set of models to make predictions at individual points.
 In the parlance of Section \ref{section:locallearning}, the model family $F$ consists of linear models on the dataset $X \subset D = {\rm I\!R}^{n + 2}$.
 Here $n$ is the number of predictor variables in $D_p = {\rm I\!R}^n$, ordinary least squares permits exactly one response variable $D_r = {\rm I\!R}$, and a column of 1s is added to conveniently represent the `intercept' term.
 A linear model $\hat{x_r}=\theta\cdot [x_p\ 1]$ can then be cast as a minimization problem on the squared-error loss on the response variable:

 $$L(x_r,\theta) := (x_r - \hat{x_r})^2$$

Such a model can be fit according to any number of optimization procedures, but has the convenient (and rare, amongst many machine learning model families) property that the minimum of the total loss has a closed form solution.
 By locally weighting the above loss function with a kernel function over the predictor variables according to Equation \ref{equation:locallearning}, we obtain a linear model defined at each point $q \in D_p$, parameterized by $\theta_q$.
 Using the local model defined at $q$ to make a prediction at $q$ gives a LOESS model:
\begin{equation} \label{equation:loess}
 \text{LOESS}(q) := \theta_q\cdot [q\ 1]
\end{equation}
In short, LOESS gives the set of points that lie on their own neighborhood's ordinary least squares regression line.
The general outline of this procedure can be extended trivially to many model families trained via minimization of a sum of individual losses.
 The most obvious extension is to polynomial regression, which has been researched extensively \cite{fan1996local}. 

The obvious way in which to use this local structure information is to simply make a prediction, which is the purpose of LOESS.
 Note that LOESS defines a function from the predictor to the response variables, which is convenient for making predictions.
 Although predictions are the usual target of local modeling, the individual regression models may be interesting outside of obtaining a prediction.
 The prospect of obtaining this additional information from the trained local models is the subject of current inquiry.
 We describe in the next section another way in which these local models may be useful.

\section{Local Model Parameters} \label{section:localmodelparameters}

Note that Equation \ref{equation:locallearning} defines a function $f(q) = \theta_q$, where $f:D \rightarrow \Theta$.
 In the case where $K$ does not depend on the response variables, as is the case with LOESS, this function reduces to $f:D_p \rightarrow \Theta$.
 Also, for LOESS, $\Theta = {\rm I\!R}^{n+1}$ since the regression coefficients are a length $(n+1)$ vector.
 Thus, we obtain a mapping from the predictor variables into an entirely new space that, we predict, encodes useful information about the local structure of the data.

Interestingly, the vast majority of research into local learning seeks to use models trained locally to directly obtain a prediction.
 This actually contrasts with global regression problems, wherein many statistics are often investigated besides raw predictions, e.g. distributions on residuals, p-values, etc. and the parameters are interpreted as meaningful attributes of the data under consideration.
 We argue that the parameters of \emph{local} models are also meaningful.
 Being much more numerous than their global counterparts, analysis of local model parameters presents an entirely new challenge on a new set of variables that might be as large or even larger than the original set of observations.

The authors of \cite{teukolsky1992numerical} suggest that ``Occasionally you may wish to know not the value of the interpolating polynomial that passes through a small number of points, but the coefficients of that polynomial. A valid use of the coefficients might be, for example, to compute simultaneous interpolated values of the function and of several of its derivatives or to convolve a segment of the tabulated function with some other function.''
 Taking LOESS in this light, the learned local model parameters actually represent approximations to the first partial derivatives of the estimated underlying function.
 For time-series, the LOESS slope parameter represents the rate of change of the response variable with respect to time, which is incredibly important to all manner of applications in the physical sciences and beyond \cite{newton1687philosophiae}.
 Considering Equation \ref{equation:locallearning} more generally: although a local model may result in a prediction at a point $q$ (depending on the family), the parameters of the model or some other summarizing statistic thereof may be useful in its own right.
 As a further example, Total Least Squares linear models have, as their learned parameters, the normal vectors to the surface formed by the model predictions.
 For a global 2D model on 3D data, this would give a single normal vector to a planar representation of the data.
 For \emph{local} 2D models on 3D data, this would give normal vectors to locally linear representations of the data, which might be thought of as approximations of surface normals for generic surfaces.
 Surface normals are very useful in, for example, 3D rendering shading algorithms \cite{lorensen1987marching}.
 We investigate this potentially useful property in Section \ref{section:scms}.

Furthermore, for time series, summarization of `local' patterns in data are an incredibly common analysis technique \cite{zivot2003rolling}.
 `Rolling' analysis, in the parlance of time-series methods, is a collection of techniques for learning information about a set of variables locally in the time domain.
 These range from analysis of rolling averages (local means in the time domain), to rolling statistical tests, to rolling linear regression analysis.
 Manually monitoring the model parameters as they evolve over time can be used to gain insight into a processes behavior \cite{zivot2003rolling}.
 Note that `Rolling linear regression' is equivalent to a LOESS model, but generally with a uniform kernel.
 
This intuition can be extended to other classes of models.
 Notably, AR(MA)\nomenclature{ARMA}{Autoregressive moving average} models are commonly employed in time series modeling due to the fact that time series exhibit autocorrelation and also since noise in time series is often itself autocorrelated.
 In \cite{bay2004framework}, the DARTS (Discovering Anomalous Regimes in multivariate Time-Series) algorithm defines a method whereby local vector AR model parameters are monitored over time for evidence of anomalous behavior.
 Indeed, they show that it is possible to \emph{automate} such a process by developing an anomaly score based on the distribution of local model parameters during normal operation.
 This insight is key to the hypothesis that local model parameters can be used as a feature extraction step for further learning.
 Note that there is nothing necessary about the use of linear models or AR models for such a purpose, or anything necessary about the use-case of anomaly detection.
 It would seem that the local model parameters are simply useful features of the data, and we hypothesize that they can be used for any number of purposes.

Local model parameters have also been used for learning outside of the context of time series analysis.
 Locally linear embedding \cite{roweis2000nonlinear} is a dimensionality reduction technique that attempts to preserve linear structure amongst neighboring data.
 The algorithm proceeds by 1) constructing a local linear model at each point of the dataset, where the query point is the target variable, and neighboring data are the predictors.
 The learned parameters of these local models are then representative of the local structure of the dataset, and can 2) be used to find corresponding vectors in a lower dimensional space that have the most similar local such parameters, and thus local structure.
 Note that the parameters learned in step 1) are not subsequently used to make a prediction, but are used as feature vectors in a subsequent learning step.

Despite this similarity between locally linear embedding, the DARTS method from \cite{bay2004framework} and the suggestions for manual inspection of rolling model parameters in \cite{zivot2003rolling}, there seems to be no research investigating this link.
 Furthermore, there seems to be little research investigating use cases of local model parameters outside of time series modeling or extending this work to the plethora of popular model families emerging in the field of machine learning.

\section{Computer Vision} \label{section:computervision}

The computer vision literature contains a great number of algorithms described as `local' (e.g. \cite{krizhevsky2012imagenet, lowe1999object}) and operating on local features.
 `Local' in this context usually refers to distances in the domain of pixel-coordinate vectors (e.g. for convolutional neural nets \cite{krizhevsky2012imagenet}), or less commonly in the space resulting from the concatenation of pixel-coordinates and color information (e.g. Mean Shift based image segmentation \cite{comaniciu2002mean}).
 Local features mostly refer to the output of some real-valued linear function operating on a square patch of pixel vectors, called `filters' or also `kernels'.
 The term `kernels' in this context is identical to the notion of `kernel' introduced above.
 However, the notion of `locality' is greatly simplified on the domain of pixel-coordinate vectors, since they are arranged neatly into a grid.
 Kernels in computer vision are almost always of the following type, and are unique to the grid-like arrangement of data in images.
 We will call them `image filters' or simply `filters':
\begin{equation} \label{equation:cvkernel}
 K(x,y) := \begin{cases} 
      W_{x-y} & ||x - y||_\infty \leq h \\
      0 & \text{else}
   \end{cases}
\end{equation}
Where the $\infty$ norm ensures a square-shaped window of width $h$, which is convenient for computational purposes.
 Since the pixel-coordinates are integral for images, the weights $W_{x-y}$ are finite in number, and are typically only constrained to sum to 1.
 By an appropriate choice of weights they can approximate pretty much any fixed bandwidth kernel, including the uniform, Gaussian and Dirichlet kernels defined in Section \ref{section:kernelmethods}.
 If we consider an image as a mapping from pixel-coordinates to color space, many filters have particularly pleasing interpretations, for example, in terms of the directional derivative.
 Yet other filters can be used to estimate the sharpness of a patch, to find edges, corners and various other interesting features.

The abilities of image filters is not without bound.
 Ironically, this is due partially to the fact that they are such an expressive class of kernel.
 Since the number of hyperparameters for the computer vision kernel grow quadratically with the width of the kernel, large kernels are more difficult to learn, and recent state-of-the-art work shows that $3\times 3$ kernels ($h=1$) are generally preferable \cite{simonyan2014very}.

Furthermore, it is often the case that many such kernels must be learned to represent simple concepts that can be easily captured with more complex methods.
 For example, an image filter can be used to determine the magnitude of the gradient of the color in a particular pixel-coordinate direction.
 However, they cannot directly find the direction of the gradient.
 One simple workaround for this is to simply include many filters that measure the gradient in multiple directions (horizontally, vertically, diagonally, etc).
 On the other hand, a more complex method such as principal component analysis might be used to approximate the direction of the gradient at a point directly.
 
Similarly, wavelet-based \cite{chan2005image} features can be approximated with image filters to find `texture'; i.e. to determine if the color values vary periodically in part of an image.
 As with gradient filters, texture filters are only capable of recognizing a priori fixed periods and directions, and are incapable of `finding' locally optimal periods.
 Specifically, a real Gabor filter returns a large value when a patch of an image resembles a local sinusoid with a given period and direction.
 Therefore, having a Gabor filter bank \cite{deng2005new} allows us to approximately determine the direction and frequency of any sinusoidal behavior (often interpreted as `texture' in image processing).
 All of the filters in the filter bank need be applied to an image to essentially `check' a battery of different periods and directions.
 This idea of creating a bank of filters scales poorly to higher dimensions, and it is distinctively suited to the grid-like structure of images.
 
This property is illustrated quite nicely in visualizations of modern convolutional neural network layers \cite{krizhevsky2012imagenet}, which provide a means for learning image filters.
 In Figure \ref{fig:nnlayers} it can be seen that the vast majority of learned filters resemble something like the real part of Gabor filters (Figure \ref{fig:gaborfilters}).
 It is entirely possible to fit a local sinusoidal model (i.e. fit a Gabor filter at each pixel) to determine the most likely direction and frequency for a given size patch, obtaining essentially the same information.
 In the case of images it is actually significantly faster to apply a Gabor filter bank than to `fit' a Gabor filter.
 For other types of data, finding the optimal local filter might not only be a viable solution, but might be preferable due to the poor scalability of filter banks.
 
This example is demonstrative of two important things.
 First, that the period and direction of local texture are \emph{incredibly valuable} as features for the remainder of the model, which is itself a neural network.
 These features might also be extracted as the parameters of an appropriately-chosen local model (specifically, a fitted Gabor filter).
 Second, that image filters have difficultly representing certain types of information, being forced to perform what amounts to a grid-search over the direction/frequency space in order to fit a function to a neighborhood at each pixel.
 Although incredibly efficient in low dimensions (pixel-coordinates are 2D) and convenient for grid-shaped data, this exact algorithm does not apply carte blanche to arbitrary datasets.
 
As a side note, we anticipate the argument that with convolutional neural networks, learning local kernel parameters at each pixel in an image is often avoided.
 This is because the total number of learned parameters grows with the size of the dataset, and this presents an opportunity for overfitting, as suggested in, e.g. \cite{nowlan1992simplifying}.
 Note that this problem hinges on the expressivity of image filters.
 Specifically, a Nadaraya-Watson estimator with an unconstrained image filter comprises a very large class of functions from the space of images of a particular size into itself.
 If the weights are permitted to vary at each point, then the number of parameters not only greatly exceeds the number of pixels in the image, but our model is largely unconstrained in any way.
 Thus, the resulting model falls hard on the variance side of the well-known bias-variance tradeoff, and will be a victim of overfitting.
 If a more restrictive model family is learned at each point, then this problem can be avoided, as evidenced by the success of, e.g. LOESS.
 Further evidence is the a priori expectation that the parameters of locally-learned Gabor filter at each point should not give significantly different information than application of a Gabor filterbank.
 A locally-learned set of models and parameters might therefore be sufficiently constrained so that models based on such features need be no more a victim of overfitting than convolutional neural networks.

\begin{figure}[!t]
    \centering
    \includegraphics[width=\textwidth]{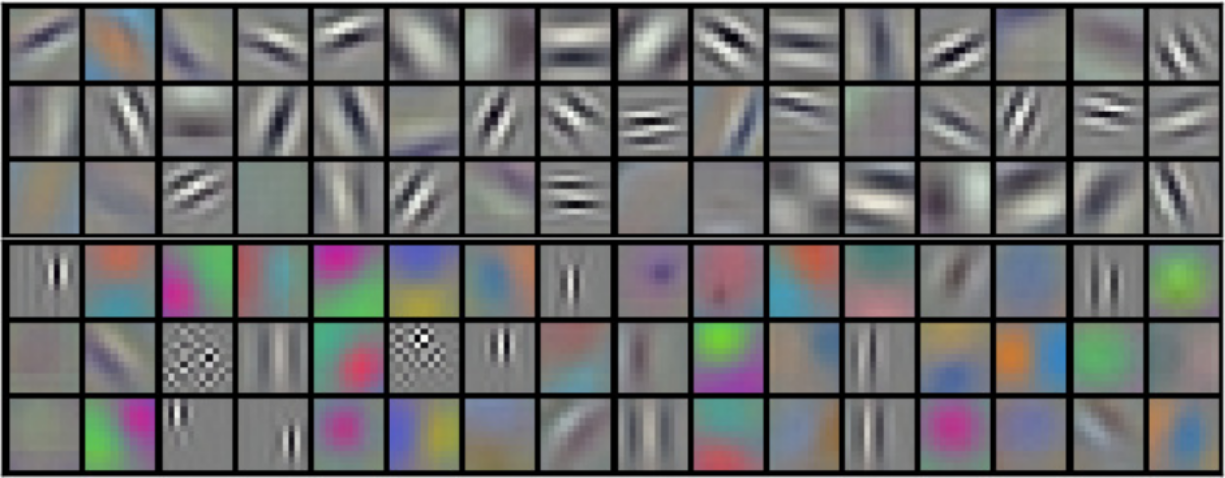}
    \caption[
        Convolutional neural network filters
    ]{
        Convolutional neural network filters.
        The top layer of the network in \cite{krizhevsky2012imagenet}.  
        Each square is a visualization of a learned filter matrix, where lighter pixels represent larger values, the entire matrix being constrained to sum to 1.
        Note that most of these filters resemble Gabor filters (Figure \ref{fig:gaborfilters}).
    }
    \label{fig:nnlayers}
\end{figure}

\begin{figure}[!t]
    \centering
    \includegraphics[width=3.5in]{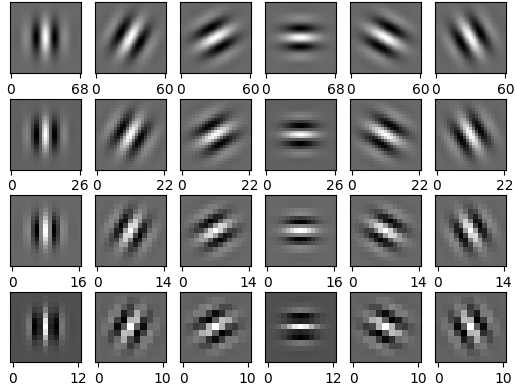}
    \caption[
        A Gabor filterbank
    ]{
        A Gabor filterbank.
        The various rotations (columns) provide a texture `direction'.
        The various frequencies (rows) provide a texture `width'.
    }
    \label{fig:gaborfilters}
\end{figure}

\section{Feature Learning} \label{section:featurelearning}

The engineering of useful features is often a key ingredient in any successful model, and is frequently the sink of a majority of effort in any machine learning project \cite{domingos2012few}.
 This is sometimes because raw data may not be amenable to direct application of many machine learning algorithms (e.g. raw text data).
 More commonly, the engineering of better features is simply a key ingredient to a more predictive model \cite{halevy2009unreasonable}.
 The input of subject matter experts and extensive trial-and-error are the bedrock of most feature engineering efforts.
 As such, much effort has been expended in techniques to perform automatic feature engineering, also called feature learning or representation learning.
 We describe here a handful of important methods for learning features, especially as related to local learning.

Many model families are sensitive to feature colinearity, others scale poorly with the sheer number of features in terms of training time, and visualization tools are optimized for low-dimensional data.
 Thus one of the most important forms of feature engineering consists of reducing the number of features while still retaining useful information from the raw data.
 Dimensionality reduction algorithms are therefore some of the earliest techniques for feature learning, but are still the subject of ongoing research.
 
Many dimensionality reducing techniques such as principal component analysis and linear discriminant analysis rely on global properties of the dataset.
 Additionally, a number of dimensionality reducing techniques are designed to take local information into account.
 Most of these attempt to model the data as an `embedding' of a lower dimensional surface into the original space, so that local structure is preserved in some sense.
 Locally linear embedding, for example, uses local model parameters as a feature to model such an embedding.
 Yet other algorithms exist for performing dimensionality reduction leveraging local structure considerations, including stochastic neighbor embedding \cite{hinton2003stochastic}, self-organizing maps \cite{kohonen2012self}, multidimensional scaling \cite{cox2000multidimensional} and others.

More algorithms exist that employ the extraction of local features via various means for various purposes including local outlier factor \cite{breunig2000lof} and, by many interpretations, the output of intermediate layers in neural networks described in Section \ref{section:computervision}.
 However, the use of local model parameters as a feature extraction step has seen relatively little use, and the wide array of model families now available provide many options for such applications.

Although locally linear embedding uses local model parameters as a feature for subsequent learning, it is limited to a specific application and a very narrow class of local model family.
 The success of locally linear embedding and other algorithms at using locally-derived features to ultimately perform dimensionality reduction suggests that features learned locally, and specifically from local models may be useful more broadly.
 Furthermore, it shows that the use of local model parameters as features has real applications, and suggests that such a feature extraction technique may find uses elsewhere.

Finally, although not always thought of as feature extraction, `chaining' models can be seen as form of feature learning.
 By `chaining', we mean using the output of one model as the input for another model.
 The chained models can be trained either independently (as in ensemble, or stacked methods, e.g. \cite{fast2008stacked}) or jointly (as in neural networks, or index models e.g. \cite{ichimura1993semiparametric}).
 Outputs of intermediate models can thus be thought of as `features' used as input to the final model.
 Although this is not directly analogous to the current work, this shows that the idea of using results from one round of modeling as a precursor to another is a viable one.

\section{Principal Surfaces and Subspace-Constrained Mean Shift} \label{section:scms}

\cite{hastie1989principal} introduce a theoretical construct for scatterplot smoothing known as principal curves.
 The original algorithm contains a number of deficiencies, including the fact that, although higher-dimensional principal surfaces are theorized, no practical algorithm is provided to obtain them \cite{ghassabeh2013some}.
 More recently, the Subspace-Constrained Mean Shift (SCMS\nomenclature{SCMS}{Subspace-constrained mean shift}) algorithm gives an iterative method to find arbitrary points on a principal curve or surface in a manner that easily generalizes to arbitrary dimensions \cite{ozertem2011locally}.
 SCMS is based on local principal component analysis (PCA\nomenclature{PCA}{Principal component analysis}).
 Upon convergence, the top principal components of the local principal component analysis at the resulting point are tangent to, and the least components orthogonal to the principal surface.
 Thus, SCMS provides a means to simultaneously project a dataset in arbitrary dimensions onto a lower-dimensional manifold and to obtain the normal vectors to that manifold at the computed points.

The SCMS algorithm is related to the Mean Shift algorithm \cite{fukunaga1975estimation}, which can, in fact, be considered a special case of SCMS.
 From an arbitrary initial point $q$, given a data design matrix $X$ and a kernel $K$, Mean Shift proceeds to the local mean centered at that point, and then repeats until convergence.
 Using the following formula for the local mean\nomenclature{$\bar{x}_{K,q}$}{The local mean centered at $q$ weighted as kernel $K$}:
\begin{equation} \label{equation:localmean}
    \bar{x}_{K,q} = \frac{\sum_{x\in X} x K(x,q)}{\sum_{x\in X} K(x,q)}
\end{equation}
We can express a single iteration of Mean Shift by:

\begin{equation} \label{equation:meanshift}
    q \gets \bar{x}_{K,q}
\end{equation}

SCMS extends this by shifting towards a kind of local ``multidimensional mean" in the form of a local principal component analysis.
 Let $proj(q,S)$\nomenclature{$proj(x,S)$}{The orthogonal projection of $x$ onto $S$} be the projection of the point $q$ onto the subspace spanned by the column vectors of $S$, and let $PCA_n(X, w)$ be the top $n$ eigenvectors of the weighted covariance matrix of data design matrix $X$ with weight vector $w$.
 SCMS can be expressed as the projection of $q$ onto a local linear model passing through the local mean:
\begin{equation} \label{equation:scms}
    q \gets proj(q - \bar{x}_{K,q}, PCA_n(X, K(X,q))) + \bar{x}_{K,q}
\end{equation}
If we adopt the convention that $PCA_0(\cdot,\cdot)$ gives back the $0$ vector, then it is easy to see that SCMS reduces to Mean Shift when $n = 0$.

We note here that principal component analysis is a convenient way to compute a total least squares linear regression.
 An $n$ dimensional total least squares regression surface is given by the span of the $PCA_n$ vectors, shifted to pass through the mean of the data.
 This interpretation in terms of total least squares regression is easy to see in Equation \ref{equation:scms}.
 This allows the process to be viewed from the perspective of being an extension of LOESS to situations where there is no distinction between the predictor and response variables.
 In such a situation, a `prediction' is not well-defined, and the prediction step required by LOESS is not possible.
 
To see the analogy, note that LOESS is defined in such a way so that a point lies on the LOESS surface in the space precisely when that point lies on its own local ordinary least squares regression line (see Equation \ref{equation:loess}).
 Similarly, a point $q$ will be contained in the set of points to which SCMS converges (the principal surface) if $q$ lies on its own local total least squares regression surface.
 This can be seen since, if $q$ lies on its regression line, then: 
 $$proj(q - \bar{x}_{K,q}, PCA_n(X, K(X,q))) = q - \bar{x}_{K,q}$$
 and Equation \ref{equation:scms} reduces to $q \gets q$.
 SCMS has been shown to converge for certain families of kernels that approximate the Gaussian \cite{ghassabeh2013some}.
 Unfortunately, it is currently unknown whether SCMS will converge for many kernels, and a generic classification of kernel functions guaranteeing convergence for even Mean Shift has been elusive for over 40 years.
 
Since the SCMS algorithm has only recently made the computation of principal surfaces feasible, research into their use in high-dimensional applications remains sparse.
 \cite{ozertem2011locally} study many of the properties of SCMS for various datasets, including kernel selection procedures, but only briefly touch on their use for 3D point clouds, and then only for contrived data.
 Specifically, appropriate kernel selection procedures, surface reconstruction accuracy and especially surface normal accuracy for real-world 3D point clouds is not currently well understood. 

\section{Unsupervised Learning for Time Series}

Time series are a large application area for unsupervised learning methods.
 These include clustering of entire time-series and clustering of subsequences of a time-series \cite{aghabozorgi2015time}, segmentation \cite{gionis2003finding} and change-point detection \cite{takeuchi2006unifying}, and anomaly and outlier detection \cite{cheng2009detection}.
 Although it is often technically possible to apply to time-series problems unsupervised methods designed for generic data, time-series usually have unique considerations that are important to account for, such as autocorrelation and a distinction between past and future.

The applications detailed in Chapters \ref{chapter:localgpr} and \ref{chapter:localw2v} fall mostly under the descriptions of change-point detection and clustering of subsequences in time-series.
 Subsequence clustering of time-series has recently been the subject of some controversy. 
 Specifically, some previously common methods of performing subsequence clustering involving simple clustering methods over sequences extracted with a sliding window have been shown to have inherent shortcomings \cite{keogh2005clustering}.
 Part of this problem hinges on the casting of a time-series subsequence as a fixed-length non-time-series vector to allow straightforward vector analysis, and clustering in the resulting vector space.
 Our proposed method does not do this, and we therefore expect it to be largely immune to these issues.
 Because of this, there has recently been a reluctance to study time-series subsequence clustering in the literature, and we think that methods that allow such an analysis, and which are robust to the problems mentioned in \cite{keogh2005clustering}, might be opportune.

Many of the methods for time-series analysis in the literature are `localized' in a sense, although the formalization is somewhat different from that presented here.
 One method for creating `localized' models for time series are via a simple `sliding-window' or, sometimes, a `Parzen-window'.
 This is tantamount to localization via the uniform kernel \ref{equation:uniformkernel}.
 The standard term for non-uniform localized models in time series analysis is `discounted', where the `discount' is on the effect of past observations on a model, which are weighted down in proportion to how distant they are in time from the point at which a model is created.
 This might be formalized as localized learning with a kernel weighting scheme where the kernel is one-sided, since time series are considered to be observed into the past from the present point in time.
 In other words, the current time observation is weighted highly, observations in the past are weighted down, and observations in the future from the query point at which a model is built are weighted zero.

The idea of performing a prediction with a model and classifying outlying or anomalous observations based on the residual of this prediction or of some subset of predictions is a common means of anomaly detection in time-series.
 This method can be used with localized, or `discounted' models \cite{cheng2009detection} in a straightforward manner.
 \cite{laptev2015generic} use a time-series of residuals of locally-learned models to form a set of features on which anomalies can be detected.
 In \cite{zivot2003rolling}, it is claimed that observation of changes in the coefficients of a `rolling' linear regression model can be used to test the validity of the assumptions underlying a global linear regression model.
 Although this is not specifically designed to detect change-points or anomalies, it can potentially be used for this purpose, insomuch as certain classes of change-points and anomalies might manifest as violations of the global assumptions of a model.
 \cite{bay2004framework} seem to notice just this property, turning a time series of data into a new time series of local model parameters (specifically ARMA model parameters) which can be monitored for anomalous observations.

However, it has so far been the case that features extracted as the parameters of local models has seen relatively little use in time series for the vast majority of model families.
 Specifically, for Gaussian Processes, no such research has been done.
 \cite{nguyen2009model} use `local' Gaussian Processes to model the time evolution of robot motion control dynamics, but their procedure is more akin to a spline as opposed to the local modeling procedure in Equation \ref{equation:locallearning}.
 Furthermore, the individual GPR models are used strictly for prediction, and not to extract more complex features from the individual models.
 As such, we think that there remains a broad need to explore the concepts from Section \ref{section:locallearning} as they apply to time series domains.



%% file: local_gpr/abstract.tex
Time series obtained from clinical sensors such as EEG and Accelerometers often contain mixtures of complex signals that can be difficult to interpret and require advanced methods for automated analysis.
 In this work, we introduce a flexible and extensible method for transforming clinical time series into new, simpler features that may aid in automated analysis and interpretation.
 The method is based on information extracted from locally-trained Gaussian Processes.
 In this work, we describe the procedure, illustrate the idea on contrived data and demonstrate its effectiveness at improving existing methods for both epilepsy detection and activity classification.
 We show that our feature extraction technique generalizes, and can be used as a preprocessing step for arbitrary machine learning and statistical methods.
 We further illustrate that the techniques evaluated in our experiments form only a small subset of a very broad class of feature extraction methods on arbitrary data that might form the basis for further study.

%% file: local_gpr/intro/intro.tex
\section{Introduction} \label{localgprsection:introduction}

Electronic sensors are used in a variety of clinical contexts.
 These include electroencephalography (EEG\nomenclature{EEG}{electroencephalography}), electrocardiogram, accelerometry, and many others.
 The signals obtained from some of these sensors traditionally require specialized training to interpret due to their intrinsic complexity \cite{gil2012electroencephalography}.
 Furthermore, the quantity of data obtained via bodily sensors has increased exponentially in recent years as we enter the era of ``big data'' \cite{luo2016big}.
 This has created an incentive for the creation of automated methods to process these signals and to extract meaningful information to cheaply and efficiently replace a human analyst.
 Unfortunately, the wide variety of electronic sensors used in clinical settings gives rise to many one-off solutions for individual use cases \cite{luckett2017dissimilarity, casale2011human}. 

In this work, we introduce a novel nonparametric time-series feature extraction method based on locally-trained Gaussian Processes.
 The described algorithm is intended solely as a feature extraction technique, and is designed to be tailorable to a very broad class of problems.
 The algorithm can be employed unsupervised, or can employ trainable global parameters when labeled data are available.
 The general idea is to train a locally-weighted Gaussian Process Regression (GPR) at desired points across the entire domain, in a manner similar to the popular LOESS \cite{cleveland1979robust} procedure.
 Instead of obtaining a new time-series of local model predictions, as in LOESS, we extract the optimal parameters of the locally-trained models.
 These parameters, we hypothesize, provide a representation of the state of the system in time, and can be subsequently analyzed using arbitrary statistical or machine learning methods.
 We demonstrate that this technique is powerful enough to improve the accuracy of naive methods for detection of epileptic events in EEG data, and flexible enough to also be able to detect activity changes in accelerometry data.

The paper is structured as follows:
 Section \ref{localgprsection:background} gives a brief overview of local methods for regression and especially methods similar to the proposed.
 It also provides a very brief overview of methods used for time-series classification which we will employ to validate the usefulness of the extracted features.
 Lastly, it provides an introduction to Gaussian Processes.
 Section \ref{localgprsection:methodology} provides a description of the proposed algorithm, the datasets used and preprocessing steps involved, and the method of validating the procedure against the described datasets.
 Section \ref{localgprsection:results} describes the results of the experiments, and provides a discussion of our observations.

%% file: local_gpr/background/background.tex
\section{Background} \label{localgprsection:background}

\subsection{Local Regression}

Local regression methods are generally classifiable as nonparametric \emph{smoothing} procedures.
 Much work has been done on the empirical and theoretical properties of local polynomials \cite{fan1996local}.
 The special polynomial cases of local linear regression \cite{cleveland1979robust} and simple local mean smoothers \cite{nadaraya1964estimating} have long been standard tools for the statistics community.
 These methods center around the idea of fitting a model to a small subset of data, and then making a local prediction.
 Applied across the domain, the resulting predictions of these local models result in a ``smooth'' version of the original data.
 
In \cite{bay2004framework}, the authors notice that not only are the \emph{predictions} of local models informative, but so too are the \emph{parameters} of those local models.
 However, they examine only a small subset of possible model families in the very narrow context of time-series anomaly detection.
 We expand upon their work by providing a general characterization for features extracted from local models, and apply the technique to new model families and to clinical time-series analysis.
 
\subsection{Time-Series Classification}

A time-series classification algorithm is one that takes in an entire time-series and outputs a discrete scalar specifying the class to which that time-series is proposed to belong \cite{bagnall2016great}.
 There are a great number of time-series classification algorithms in the literature, of which a select few are reviewed here.

One class of time-series classification methods learn a fixed number of features from an individual time-series.
 The fixed-length feature representation can then be used as a vector in an arbitrary machine learning classification algorithm.
 Shapelets \cite{ye2011time} are an example of such a feature extraction technique, and are typically paired with a decision-tree classifier.

A second class relies on the fact that, given a distance metric between any two time-series on some domain, many existing machine learning methods for classification can be applied.
 As such, many time-series classification algorithms are concerned with contriving a distance metric that places time-series belonging to the same class `near' each other, and series belonging to different classes `far' from one another.
 The most commonly used class of distance metrics for time-series classification is DTW\nomenclature{DTW}{dynamic time warping} (dynamic time warping) \cite{velichko1970automatic}, and is often paired with a nearest-neighbors classifier.
Although there have been a great number of new algorithms for time-series classification introduced in the literature in recent years, empirical evidence suggests that DTW together with 1-nearest-neighbors (1NN) classification is difficult to beat on most tasks \cite{bagnall2016great}.

Relatively little work has been done on evaluating the relative efficacy of time-series classification methods paired with time-domain features.
 \cite{gorecki2013using} investigate the use of the first difference of the time-series as a feature to be used with DTW/NN style algorithms.
 The evaluation of more complex features is lacking in the literature.  
 We will therefore employ DTW/1NN for our experiments with time-series classification, and expand upon this method by pairing it with the proposed feature extraction procedure.

\subsection{Gaussian Process Regression} \label{localgprsection:gpr}

In the remainder of the paper, the family of models under consideration will consist of Gaussian Processes with problem-appropriate covariance functions.
 A Gaussian Process ``is a collection of random variables, any finite number of which have a joint Gaussian distribution'' \cite{rasmussen2005gaussian}.
 For those unfamiliar with Gaussian Process Regression, it is helpful to consider the case of two neighboring points in a time-series.
 If we know the value of the first variable, we might reasonably expect the second to be approximately normally distributed with a mean very close to the first measurement, and for the variance to be commensurate in how far apart the two measurements are in time.

For our purposes, these variables are the dependent variables in a time series taken at various times.
 Such a Gaussian Process is completely defined by two functions:
 A mean function that gives the expected value at each point in time, and a covariance function that describes the relationship between pairs of points.
 Often, the mean function is taken to be $0$, an assumption which can, in practice, be approximately satisfied by differencing on the mean.
 This assumption may not hold generally, and can be abandoned by explicitly modeling the mean function, but we will adopt it for computational simplicity.
 The covariance function is a function whose choice is problem dependent \cite{rasmussen2005gaussian}.

Common choices for covariance function are the radial basis function (RBF\nomenclature{$RBF_l$}{A radial basis function with mean 0 and standard deviation $l$}), in which neighboring points have high covariance, 
  periodic functions, e.g. \cite{mackay1998introduction} in which seasonal neighbors have high covariance,
  noise functions in which neighboring points are independent of one another
  or sums and products of the above options.
 Gaussian Process Regression is the process of determining the optimum parameters for the chosen covariance function family, given some dataset.
 
This requires some notion of `optimal', which can be defined using any arbitrary loss function.
 We will compute optimal parameters by using the negative marginal likelihood as a loss function.
 The marginal likelihood for a Gaussian Process is given by:
\begin{equation} \label{localgprequation:gprobjective}
  p(y|X,\theta) = \frac{
        \exp\left(-\frac{1}{2} y^\mathrm{T}{C_\theta(X)}^{-1}y\right)
    }
    {
        \sqrt{(2\pi)^k|C_\theta(X)|}
    }
\end{equation}
Where $C_\theta(X)$ is the covariance matrix given by our chosen covariance function family, parametrized by $\theta$, over the index data $X$ with given labels $y$.

The covariance functions that we will be using (all of which are described at length in \cite{rasmussen2005gaussian}) are sum and product combinations of: 
The constant covariance function\nomenclature{$CN_c$}{A constant function parametrized by $c$}:
\begin{equation}\label{localgprkernel:constant}
CN_c(x,x') = c \ \forall x,x'
\end{equation}
The RBF covariance function, to model trends:
\begin{equation}\label{localgprkernel:rbf}
RBF_l(x,x') = \exp\left(-\frac{||x-x'||_{L_2}^2}{2l^2}\right)
\end{equation}
The white noise (WN\nomenclature{$WN$}{The indicator function for equality in the arguments}) covariance function, to model pointwise noise:
\begin{equation}\label{localgprkernel:wn}
WN_\epsilon(x,x') = 
\begin{cases}
    \epsilon & x = x' \\
    0 & x \neq x' \\
\end{cases}
\end{equation}
The exponential sine-squared covariance function\nomenclature{$SS_{p,l}$}{The exponential sine-squared function, parametrized by $p$ and $l$}, for periodicity:
\begin{equation}\label{localgprkernel:ss}
SS_{p,l}(x,x') = \exp\left( \frac{-2\sin^2(\pi/p ||x-x'||_{L_2})}{l^2}   \right)
\end{equation} 

One of the primary motivations for using Gaussian Process Regression in the current study is the interpretability and flexibility made possible by piecing together different covariance functions.
 By analyzing the time-series of parameters of an exponential sine-squared component from a locally-trained GPR, we hope to gain insight into changes in periodicity of a time-series.
 By analyzing the time-series of parameters of the white noise component, we hope to gain insight into changes in the noise level of the data, etc.
 
Although there exist methods in the literature described as `local' Gaussian Process Regression \cite{nguyen2009local}, these methods are more akin to splines than to the local methods we are emulating in this work, such as LOESS.
 Furthermore, the use of local GPR to perform time-series change-point detection expands the work of Bay et al. \cite{bay2004framework} to a new set of model families.
 Our contribution in this domain, therefore, consists of both a method for training local GPR models, and the novel use of locally-trained GPR models as a feature extraction method for time-series.

%% file: local_gpr/methodology/methodology.tex
\section{Methodology} \label{localgprsection:methodology}

\subsection{Algorithm} \label{localgprsection:algorithm}

We propose a method of automatic feature extraction that maps the original data to new vectors representing local relationships between data points.
 This method, based in part on the algorithm in \cite{bay2004framework}, consists of training a model at each point in a domain where data are weighted in such a way as to obtain a `local' model.
 Various properties of these local models, such as the model parameters, can then be extracted as features for further analysis.
 
We employ the concept of locally-weighted learning \cite{atkeson1997locally}.
 This involves weighting data that are `near' a query point $q$ to obtain a local model at the point $q$ using a kernel $K$.

 Common choices for $K$ include a uniform kernel, which, for time-series, is equivalent to a Parzen window:
 \begin{equation}
    U_h(x,x') =
        \begin{cases}
            1 & ||x-x'||<h \\
            0 & \text{else} \\
        \end{cases}
\end{equation}
 And also a tricube kernel, which transitions the weights of distant points smoothly to 0:
 \begin{equation} \label{localgprequation:tricube}
    T_h(x,x') = (1-||x-x'||^3)^3
 \end{equation}
 The general idea of learning a local model using a kernel-weighting scheme can be found in expositions of the LOESS procedure \cite{cleveland1979robust}.
 
We then extract the parameters or other features of the local model to obtain a representation of the state of the process at the point $q$.
 Our analyses may involve preprocessing of data and further processing to obtain specific results, but the primary contribution of this paper is the extraction of features from these locally-trained models.

More formally, and generally, consider a dataset $X$ on some domain $D$, and a model family $F$ parameterized by $\theta$.
 Suppose that $A$ is an algorithm for training a model that takes in $X$ and returns $f \in F$, and that $A$ accepts a scheme for weighting the contributions of the individual datapoints.
 Points $x\in X$ are weighted against the query point $q$, usually as some function of their distance to $q$, as $K(x,q)$.
 We denote the weighted training procedure at query point $q$, weighted using kernel $K$ as $A_{K,q}$.
 If $P$ is some model postprocessing algorithm, $P: F\rightarrow {\rm I\!R^k}$, then we can compose $A$ and $P$ on the dataset $X$ to obtain a feature transformation from the original variables to a new space that encodes local structure of the data.
 We term this transformation a `Local Model Feature Transformation' (LMFT):
\begin{equation} \label{localgprequation:lmft}
  LMFT(q) = P(A_{K,q}(X))
\end{equation}
 This broad class of functions includes local model predictions, local model parameters and other summarizing statistics for individual local models.
 The manner of local structure that is encoded depends on the specific choice of $F$, $A$, $K$ and $P$, but primarily on $F$ and $P$.
 
For the purposes of this paper will will be studying the family of functions $F$ defined by Gaussian Process Regression (GPR) using a covariance function that is context-specific (see Section \ref{localgprsection:gpr}).
 We weight our GPR algorithm according to a weighting scheme similar to that described in \cite{hong2017weighted}.
 Except for the addition of a weighting scheme, we employ the `ordinary' training procedure ($A$) of optimizing the marginal likelihood over available model parameters \cite{rasmussen2005gaussian}.
 We will employ the tricube kernel for $K$ (Equation \ref{localgprequation:tricube}) throughout our experiments, since it is smooth and also since it has a finite support which is more computationally efficient.
 Lastly, our postprocessing procedure $P$ will consist of simply extracting the parameters of the locally-trained models.
 Since these parameters are real, we obtain a new time-series in the space of model parameters.

\subsection{Experiments}

In this section, we describe the datasets used, and the methods employed to provide evidence for the utility of the algorithm described in \ref{localgprsection:algorithm}.
 Since the intent of the given procedure is to provide a generic framework for extracting signals from data, we have attempted to provide a variety of contrived and real world datasets to illustrate the algorithm's flexibility.

\subsubsection{Contrived Data} \label{localgprsection:contriveddata}

We have contrived datasets that reflect both changes in levels of white noise and also in periodicity. 
 These two datasets can be seen in Figures \ref{localgprfigure:contrivedvariancecherry} and \ref{localgprfigure:contrivedperiodgpr}, respectively.
 We provide these to illustrate the general idea, and to provide evidence enough to warrant exploration on real data.
 We further attempt to provide here justifications for some of the decisions made in later experiments regarding choice of various hyperparameters, and other design decisions.
 Thus, we will not perform any validation in the ordinary sense, and will instead merely provide visualizations of cherry-picked results intended to provide the reader with a `feel' for the algorithm.

\begin{figure}[!t]
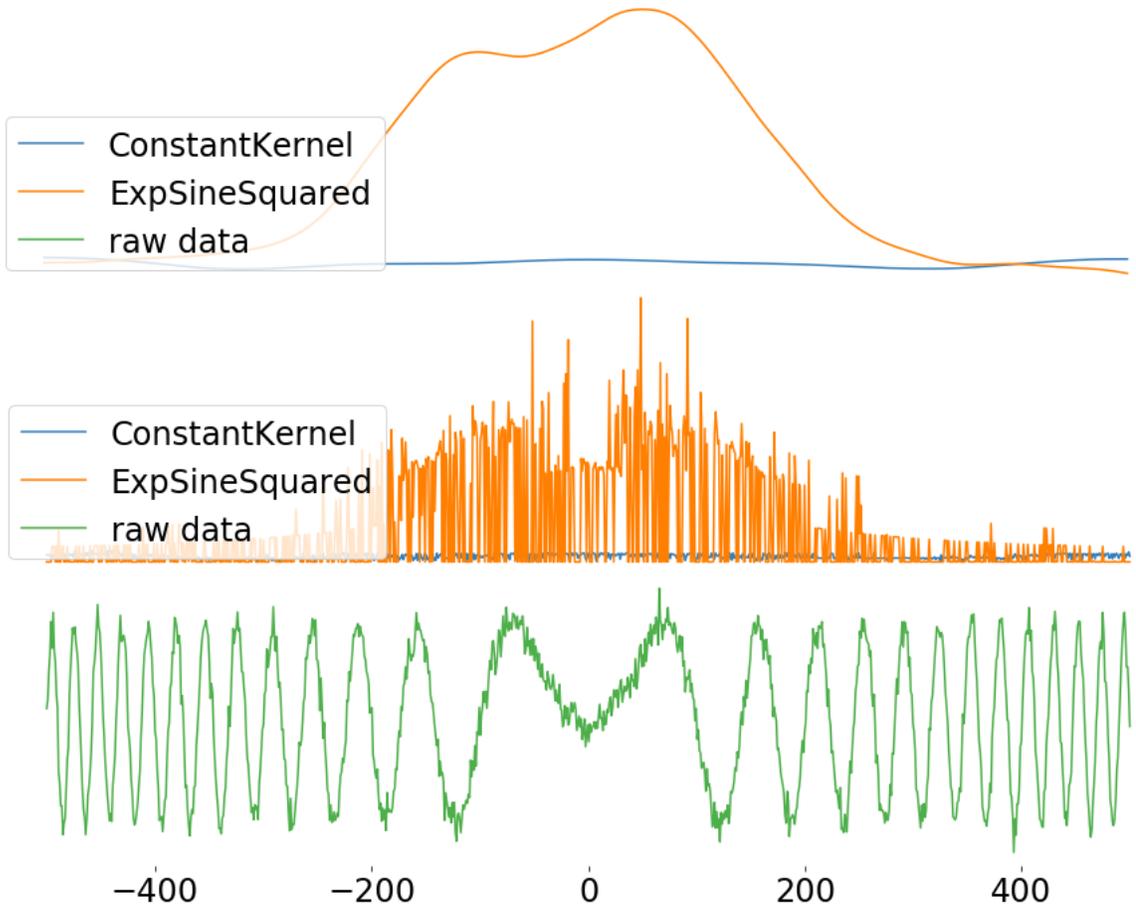

    \centering
    \includegraphics[width=\columnwidth]{local_gpr/static/{{contrived_period_gpr_both}}}
    \caption[Gaussian process period feature for toy data]{Gaussian process period feature for toy data. 
        Variable period data below in green.
        Extracted GPR parameters are period for the exponential sine squared term in orange, and the constant term value in blue.
        100 seeds are chosen log-uniformly on $[10^{-10}, 10^{10}]$.
        Above, extracted parameters have been smoothed.}
    \label{localgprfigure:contrivedperiodgpr}
\end{figure}

\begin{figure}[!t]
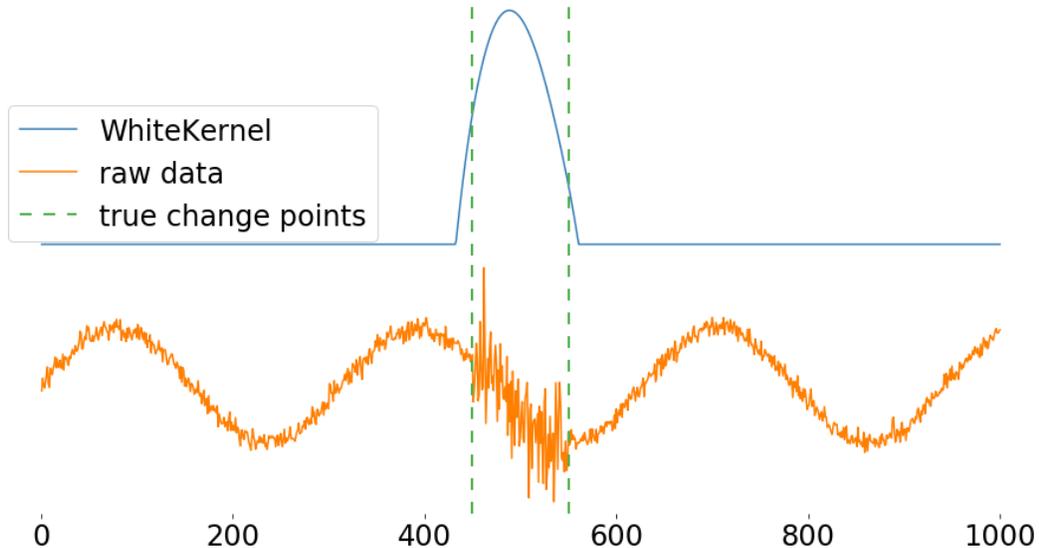

    \centering
    \includegraphics[width=\columnwidth]{local_gpr/static/{{contrived_variance_cherry}}}
    \caption[Gaussian process noise feature for toy data]{Gaussian process noise feature for toy data.
        Contrived variable noise level data together with extracted GPR parameter representing local noise level}
    \label{localgprfigure:contrivedvariancecherry}
\end{figure}
 
These two datasets are chosen to target the `noise' parameter in Equation \ref{localgprkernel:wn} and the `periodicity' parameter in Equation \ref{localgprkernel:ss}.
 These are not the extent of parameters available to a Gaussian Process Regression, and certainly do not cover the extent of parameters available when wider classes of model families are taken under consideration.
 Investigations on these two parameters therefore only show a small segment of the possible uses of LMFT that can be acheived with minimal change to the general structure of the process.

Analysis of the first dataset will employ a covariance function of the form:
\begin{equation} \label{localgprequation:variablevariancecovariance}
 C = CN_c RBF_r + WN_w
\end{equation}
 Analysis of the second dataset will employ a covariance function of the form:
\begin{equation} \label{localgprequation:variableperiodcovariance}
 C = CN_c SS_{p,l} + WN_w
\end{equation}

We employ a tricube kernel to locally weight our data during training of individual models, with a bandwidth of 120.
 We have chosen the value 120 a priori, since this results in a `window' of size $240$ and the length of the section we wish to detect is of length $100$ in the variable noise dataset.
 We have not attempted to evaluate the effect of bandwidth, but the reader may wish to keep in mind that, in general, a higher bandwidth tends to produce smoother results, as is typical with kernel methods \cite{turlach1993bandwidth}.

\subsubsection{Accelerometry Data} \label{localgprsection:accelerometrymethodology}

In order to demonstrate that the algorithm is useful for real-world applications, we further provide analyses of a set of uncalibrated accelerometer readings previously used for activity classification \cite{casale2011human}.
 An example of one axis of one set of observations from this dataset is given in Figure \ref{localgprfigure:accelerometer5}.
 We will employ this same dataset, but we will only try to detect when there has been a change in activity.
 This is an important sub-problem in the task of classifying individual activities, and is a simplification of the problem suitable for initial explorations of the proposed algorithm.

\begin{figure}[!t]
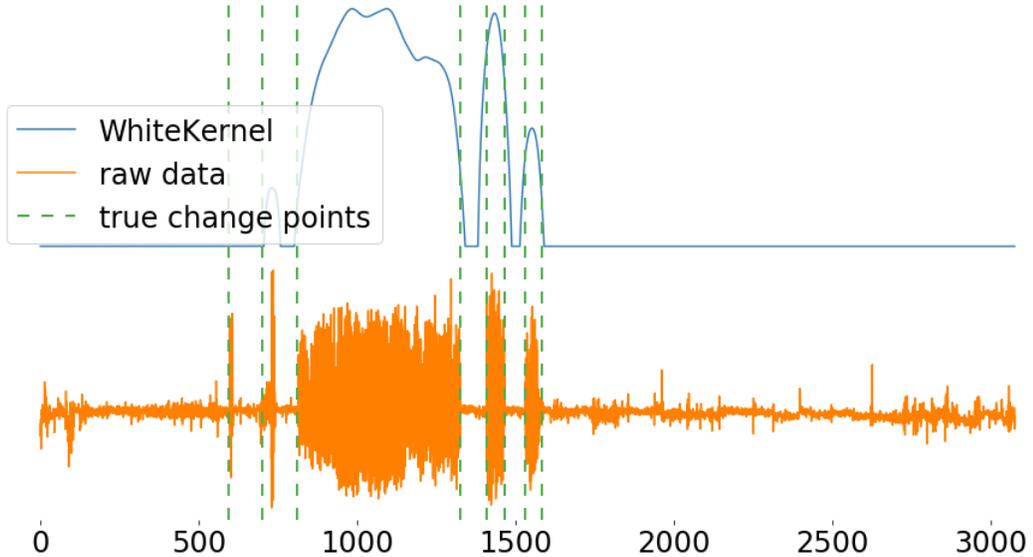

    \centering
    \includegraphics[width=\columnwidth]{local_gpr/static/{{accelerometer_05_y_axis}}}
    \caption[Gaussian process noise features on accelerometer data]{Gaussian process noise features on accelerometer data.
        y-axis acceleration from one of the sequences in the accelerometry data.
        Above is the extracted set of features from the learned white noise parameter of a local GPR.
        The kernel bandwidth is set at $80s$, and fixed parameters in equation \ref{localgprequation:accelerometergpr} set as $c=8000$ and $r=7s$.}
    \label{localgprfigure:accelerometer5}
\end{figure}
 
Upon inspection, the labels in the dataset provided on the UCI machine learning repository do not seem to align perfectly with the signal in the time series.
 Since the dataset is relatively small, we have manually aligned the ground-truth labels by applying a constant shift to the entire label sequence.
 This was performed as a preprocessing step, before the training or validation steps.

It is easy to see from Figure \ref{localgprfigure:accelerometer5} that an important feature for detecting a change in activity is the amount of noise in the signal.
 This makes intuitive sense as well: the amount of `bounce' picked up by the accelerometer will be significantly different depending on whether one is sitting at a computer or running up a flight of stairs.

From the figure, we can see that the accelerometer doesn't seem to hold its zero very well, and there may therefore be longer-term trends or shifts in the mean of the accelerometer values.
 Although these sudden mean shifts occasionally align with a change in activity, it is just as frequent that they occur when there is no change in activity.
 Since there is no label for such things, it is only clear that such shifts are not necessarily indicative of a change in activity.
 We hypothesize that it would be possible to extract these changes in noise level while simultaneously accounting for the longer-term trends and shifts in mean with a local GPR with a covariance function of the form:
\begin{equation} \label{localgprequation:accelerometergpr}
    C = CN_c RBF_r + WN_w 
\end{equation}
The RBF term will account for the observed trends and shifts to the mean.
By monitoring the changes in the learned parameter for the white noise component, which is obtained conditional upon the RBF term, we hope to extract a signal that clearly indicates when there has been a change in activity.

We impose restrictions to force the RBF term to ignore very small scale changes, and to instead model the long-term changes in accelerometer calibration mentioned above.
 This can be achieved by fixing the length scale parameter of the RBF and constant term components of the covariance function.
 Some exploratory analysis lending evidence to our decision to fix these values rather than allowing them to vary in the model is discussed in Section \ref{localgprsection:accelerometerresults}.
 Fixing various parameters in our local models has further benefits which we describe in Section \ref{localgprsection:contrivedresults}.
 
For each combination of hyperparameters defining our model family, we propose to extract local GPR features from each channel of an accelerometer recording, resulting in a $N_t\times N_c \times N_f$ tensor, where $N_t$ is the recording length, $N_c$ is the number of channels (3), and $N_f$ is the number of features extracted from each channel.
 We then flatten the last two axes to re-obtain a time-series of vectors allowing traditional analysis, resulting in a transformed dataset of shape $N_t\times N_cN_f$.
 If we constrain the problem so that the white noise level is our only free parameter, this results in a new time series with 3 channels, each representing the local noise-level of the accelerometer readings in each axis.

\subsubsection{EEG Epilepsy Data} \label{localgprsection:eegmethodology}

Finally, we provide an analysis of algorithm performance on the task of epilepsy detection from EEG readings.
 A dataset of 39 EEG recordings each involving an epileptic event was obtained from the authors of \cite{luckett2017dissimilarity}.
 Recordings are in the 10-20 international system, resulting in 21 channels of unfiltered surface EEG signals sampled at 250Hz.
 The approximate point at which the epileptic event occurs is also noted.

Since EEG recordings are notoriously noisy, we have preprocessed the data in an attempt to clean up some of these artifacts.
 We have followed advice from \cite{gil2012electroencephalography} in reviewing the raw signals.
 In order to reduce the influence of artifacts such as mains hum, we apply a filter with both low and high pass components.
 The applied filter limited the frequencies of the signal to $<\sim 50$Hz.
 After filtering the data, we subsample the data for computational efficiency to reduce the sample rate to 50Hz.
 Since most of the high frequency signals have already been removed, there is no good reason to retain a high sample rate.
 
From visual inspection of the signals, non-normal EEG signals appear within less than a minute of the recorded event time.
 Thus, we propose to analyze 200s long fragments of each recording.
 \cite{gil2012electroencephalography} suggest that frontal lobe seizures do not typically last for multiple minutes, so that 200s of signal should be sufficient to capture both a period of non-ictal and ictal activity.
 39 `positive' fragments are taken to include 100s of EEG signals before and after the label of the epileptic event.
 39 `negative' fragments are taken to include 200s of EEG signal sampled during a section equidistant from the beginning of the recording and the noted time of the epileptic event.
 Each recording is at least an hour long, so this results in a balanced dataset containing signals representative of the onset of an epileptic event and signals representative of ordinary neural activity.

We propose to extract local GPR features from each channel of an EEG recording, resulting in a $N_t\times N_c \times N_f$ tensor, where $N_t$ is the recording length, $N_c$ is the number of channels (21), and $N_f$ is the number of features extracted from each channel.
 We then flatten the last two axes to re-obtain a time-series of vectors allowing traditional analysis, resulting in a transformed dataset of shape $N_t\times N_cN_f$.
 
Brainwaves are often interpreted in terms of their various length scales, with delta waves on the order of 0.5-4 Hz, up through beta waves at 12-30 Hz and including other, even higher frequency waveforms.
 Our downsampled data is incapable of revealing any frequencies above the beta range.
 In addition to altered brainwave patterns, ictal activity is sometimes associated with motor activity, which is on the order of delta waves or somewhat lower \cite{gil2012electroencephalography}.
 
In order to detect changes in these various brainwave patterns, we employ a GPR covariance function of the following form:
\begin{equation} \label{localgprequation:eeggpr}
    C_{\text{EEG}} = CN_c RBF_r + WN_w
\end{equation}
 The $CN * RBF$ component is designed to model changes in brain wave activity and also motor activity.
 In order to constrain the model to detect such changes, we fix the $RBF$ term with a length scale $l$ representative of delta waves.
 We accomplish this by fitting an unconstrained GPR with the above covariance function to a simple sine wave with a period in the delta wave range, and a scale comparable to the EEG signals.
 We further fix the $WN$ term with a noise level at $1$, large enough to encourage the covariance matrix to be positive definite, but small enough to have a relatively small effect on the model.
 To see that this should have a small effect on the model, we note that the scale of the signals is on the order of hundreds of units.
 We further note that our low-pass filter preprocessing step removes most of the signal that might be considered white-noise, and so fixing the white noise level to be rather low is not unreasonable.

We further fix the $RBF$ term and $WN$ term for the optimizer when training the local GPR models over the constant term.
 This fixed $RBF$ term parameter is obtained by first fitting an `exemplar' model to an arbitrary section of each signal from a non-ictal region.
 The exemplar model is fit using 400 initial seeds to obtain approximately globally optimal parameters.
 The trained $RBF$ of the exemplar model is then constrained to be constant during local model training.
 The $WN$ term is a priori constrained to be 1, chosen based on the scale of signal units, and our experience with exploratory investigations.
 The kernel bandwidth is fixed at $2s$, chosen based on the range of commonly studied frequencies in brain waves, and also our experience with exploratory investigations.

By fixing the parameters of the $RBF$ and $WN$ terms, we hope to interpret the locally-trained parameters of the \emph{CN} term as representing changes in the pattern of brain waves or of motor activity.
 We hypothesize that the features so extracted will be useful in the detection of the onset of epileptic events.

In order to classify a fragment as containing ictal activity or not, we will observe changes in the extracted local GPR features.
 Specifically, we anticipate that the extracted features will be significantly different depending on whether the query time is during or near an ictal period.
 We hypothesize that the complex patterns exhibited by ictal activity in the raw signal will translate to simpler patterns in the extracted features.
 We predict that this will improve the power of classical methods used to classify arbitrary time series.

We propose, therefore, to employ 1NN with DTW variously on the raw signals, on the LMFT-extracted features, and on the concatenation of the two.
 We split the above-described dataset of segments representing both non-ictal activity and those representing the onset of ictal activity $\sim 50/50$ into train ($n=20$) and validation ($n=19$) sets.
 We then classify the validation set by 1NN with DTW against the training set.
 We will then evaluate the results of the three possibilities by comparing the confusion matrices, and particularly precision, recall and accuracy.

DTW is based in part on calculating a Euclidean distance between points at similar times.
 Therefore, it is sensitive to the relative scale of the various features, and features with smaller scales are effectively weighted down in importance.
 Therefore, in order to effectively apply DTW to the concatenation of the raw data with the extracted features, we scale each channel and each feature to have mean 0 and standard deviation 1.
 Since the various channels of the raw data are all of a comparable scale, and so too with the extracted features, we do not perform any scaling when applying DTW to only raw data or only extracted features.

Since the dataset is quite small, we set aside 2 of the training samples for selection of the kernel bandwidth $b_{\text{GPR}}$ for the local GPR models, and also for selection of fixed parameters in the covariance function.
 Our process of `training' to obtain these hyperparameters consists of visualization of the extracted signals for various setting and also a priori knowledge of the nature of the data.
 Note that in the presence of additional data, it might be possible to train over available bandwidths and other global hyperparameters to obtain optimal values for the task at hand (differentiating ictal vs non-ictal signals.)

%
%
%
%
%
%

%% file: local_gpr/results/results.tex
\section{Results} \label{localgprsection:results}

\subsection{Contrived Data} \label{localgprsection:contrivedresults}

In this section we investigate a number of practical issues related to obtaining local modeling with Gaussian Processes, including the stability of features over time, hyperparameter selection, and the interpretation of any extracted features.
 We use these initial investigations to inform the setup for later experiments.

Depending upon the covariance functions and data, the sought-after optimum might be in a very ``narrow" valley.
 To the point that finding it using a random seed is a rather unlikely event, and so a great number of seeds must be employed to reliably find good optimal parameter estimates.
 This is illustrated in the middle graph of Figure \ref{localgprfigure:contrivedperiodgpr}.
 These parameters are extracted by training over the constant and period parameters and fixing the length scale and white noise parameters in Equation \ref{localgprequation:variableperiodcovariance}.
 100 random seeds are chosen log-uniformly on $[10^{-10},10^{10}]$ and the optimizer is run on each of these, with the best result taken in the end.
 Approximately half of the time, a period on the order of tens is found, which is representative of the approximate period of the local data at that point.
 The top graphs in Figure \ref{localgprfigure:contrivedperiodgpr} show a Nadaraya-Watson-smoothed version of the extracted parameters.
 Note that the period generally increases toward the middle of the data, and decreases toward the edge, which is reflective of the properties of the underlying data.
 This suggests that the algorithm is capable of extracting meaningful parameters, but has a problem with stability of the optimal GPR parameters over time.

These problems are not uniform across all covariance functions and datasets.
 We have identified a number of means by which the stability of extracted features can be improved.
 First, reducing the number of free parameters improves output stability.
 Second, seeding the optimizer with a single fixed value that is `near' the expected operating range encourages the algorithm to remain in a particular local optimum.
 Lastly, the `plain' covariance function in Equation \ref{localgprequation:variablevariancecovariance} seems to be much more stable than the periodic one.
 These three things can be seen operating in tandem on the contrived data with variable variance in Figure \ref{localgprfigure:contrivedvariancecherry}.

From these observations, we believe that seeding the optimizer with a fixed value and limiting the number of free parameters available to the local models by training them \emph{globally} to be a good partial solution to the discontinuity problem.
 The proceeding experiments will therefore all employ these options.

\subsection{Accelerometry Data} \label{localgprsection:accelerometerresults}



Our initial explorations of the proposed feature extraction method on this data focused on the signals from 2 different individuals.
 A portion of the signal containing the ground truth change points for one sequence can be seen in Figure \ref{localgprfigure:accelerometer5}.
 The extracted features corresponding to the local model feature extraction procedure detailed in \ref{localgprsection:accelerometrymethodology} is also shown.

Note the length scales of the x-axis in Figure \ref{localgprfigure:accelerometer5}.
 Although the signal appears to have a very prominent noise component when viewed at a resolution on the order of minutes, it is significantly less noticeable when we zoom in to see changes at the level of seconds.
 This is because there is non-trivial small-scale structure to the accelerometer readings.
 We have discovered that this small-scale structure in the readings is not as useful as the large-scale structure in distinguishing between various activities.
 Complicating matters, as the RBF component goes to 0, the $CN_c RBF_r$ term simply becomes a $WN$ component.
 Allowing the optimizer to choose the parameter of the RBF component of our covariance function \emph{locally} may therefore not be an optimal method for extracting the noise component.

For this reason, and for the purpose of stability as mentioned in Section \ref{localgprsection:contrivedresults}, we choose the RBF length scale and constant term parameters \emph{globally}.
 Choosing $r$ and $c$ via a grid search to give a good signal, we obtain the extracted signal shown in Figure \ref{localgprfigure:accelerometer5}.
 In particular, small values of $r$ and/or $c$ tend to produce a very noisy signal, and the extracted features shown reflect much larger values for $r$ and $c$.


\subsection{EEG Epilepsy Data} \label{localgprsection:eegresults}

\subsubsection{Exploratory Results}

Our initial explorations of the proposed feature extraction method on this data focused on the positive (including ictal onset) subsamples from 2 different individuals.
 The positive samples contain both ictal and non-ictal activity, and were expected to therefore illustrate the differences in the extracted features for these two types of EEG signals.
 A portion of the signal containing the labeled point of ictal onset for one sequence can be seen in Figure \ref{localgprfigure:ictalonset5}.
 The extracted features corresponding to the local model feature extraction procedure detailed in \ref{localgprsection:eegmethodology} are also shown.
 Since these are our `training' data, the parameters have been specifically selected in such a way as to obtain a `nice' signal.

\begin{figure}[!t]
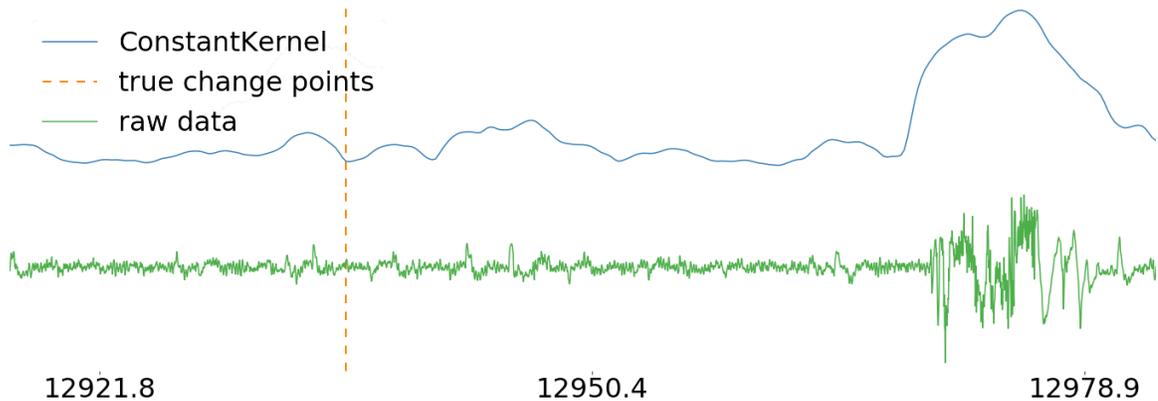

    \centering
    \includegraphics[width=\columnwidth]{local_gpr/static/{{F00019_ictal_onset_gpr_illustrative}}}
    \caption[Gaussian process RBF features on EEG data]{Gaussian process RBF features on EEG data.
        Approximately 60s of channel FP1 from one of the recordings, filtered and including the labeled point of ictal onset.
        Above is the extracted set of features from the learned constant value parameter of a local GPR.}
    \label{localgprfigure:ictalonset5}
\end{figure}

In Figure \ref{localgprfigure:ictalonset5} it can be seen that sections of the signal with prominent low-frequency components give high values for the extracted feature, and vice-versa.
 In this figure, it can be seen that the onset of ictal activity might involve a move from lower-valued LMFT features to higher-valued ones, although perhaps not in a uniform fashion.
 This observation has informed our hypothesis that the performance of DTW will increase when applied to the LMFT features, since it need only discover the relatively simple transition over time from low to high value, rather than an intricate series of ups and downs associated with the raw EEG signals.

\subsubsection{Confirmatory Results}

Having decided upon a covariance function family and a setting for the various hyperparameters in the exploratory analysis, we proceed to extract features from every channel of every sequence.
 We then employ the validation process as described in Section \ref{localgprsection:eegmethodology}.

The confusion matrices for the 3 proposed datasets - raw data, LMFT features and the concatenation of the two - are given in Table \ref{localgprtable:eegconfusion}.
 It can be seen that 1NN using DTW for classifying sequences of non-ictal EEG activity vs the onset of ictal activity using the raw data is ineffective.
 Indeed, the true label and the predicted label are entirely independent.
 Upon inspection, it seems that 17/19 of the true positive and 15/19 of the true negative validation sequences were matched to two of the training sequences.
 Interestingly, although one of those training sequences was true positive and one true negative, there was no correlation to the label of the validation sequences to which they were matched.
 It is not clear what the cause of this strong preference for certain training sequences was, but the lack of correlation reveals that the raw data are not useful for input directly into a DTW/1NN classifier.
 
Similarly, when the raw data and the extracted features were concatenated, 34/38 of the validation sequences were matched to two of the training sequences.
 These two sequences were both negative, and were both different from the two popular sequences from the 1NN with DTW on only the raw data.
 We had expected that the combined information of the raw data and the extracted features would be \emph{more} informative for a classification algorithm, but this did not turn out to be the case.

1NN with DTW on only the LMFT features performed better.
 The gains were modest, giving an accuracy, precision and recall of only 0.66, 0.75 and 0.47 respectively. 
 Still, the intention of the experiment was to show that the LMFT extracted features could be useful for a later classification step, and so these results support our hypothesis.
 We remind the reader that there exist opportunities for training over hyperparameters with larger datasets, and also that the proposed algorithm accepts arbitrary model families.
 This suggests that many improvements to the proposed method might still be made, even for this particular task of epilepsy detection.

\begin{table}[!t]
  \begin{center}
    \caption{EEG classification confusion matrices.}
    \label{localgprtable:eegconfusion}
    \begin{tabular}{llll|ll|ll}
        &   & \multicolumn{6}{c}{\textbf{Predicted}} \\
        &   & \multicolumn{2}{c}{\textbf{Raw Data}} & \multicolumn{2}{c}{\textbf{LMFT}} & \multicolumn{2}{c}{\textbf{Both}} \\
        &   & -    & +    & -    & +    & -    & +    \\
      \multirow{2}{*}{\textbf{True}}
        & - & 14   & 5    & 16   & 3    & 0    & 19   \\
        & + & 14   & 5    & 10   & 9    & 0    & 19
    \end{tabular}
  \end{center}
\end{table}


%% file: local_gpr/conclusion/conclusion.tex
\section{Conclusion}

We have investigated a novel feature extraction method, \emph{local model feature transformations}.
 Via experiments on synthetic and real-world datasets we have shown that this feature extraction method can transform time-series into a variety of time-domain features revealing various local properties of the underlying signal.
 We have specifically shown that local model feature transformations based on Gaussian Process Regression can be used to extract interpretable signals from clinical time-series, and can further be used to improve the performance of off-the-shelf classification algorithms.
 Future and ongoing work includes experiments with local model feature transformations based on additional model families and further applications to bioinformatic and other domains. 

\section{Acknowledgements}

We would like to thank Dr. Lee Hively, who granted us permission to utilize the electroencephalography recording dataset.

%% file: local_SVM_projection/abstract.tex
Local learning algorithms are a very general class of nonparametric lazy learners obtained by stitching together systems of locally weighted parametric models.
 For a system of locally-learned classifiers, there is currently no convenient method to find points on the decision surface.
 In this paper, we introduce a generic algorithm for finding the decision surface for systems of localized classifiers using arbitrary model families.
 The decision surface of a classifier is often useful for obtaining pseudo-probabilistic output from the orthogonal distance of a point to the decision surface.
 We therefore extend our method to find this orthogonal projection of an arbitrary point onto the decision surface for a broad class of classifier families.
 We specifically derive the necessary equations for computing the orthogonal projection onto the decision surface of systems of locally linear support vector machines.
 We demonstrate how this can be used for pseudo-probabilistic calibration, and by extension for multiclass classification strategies such as one-vs-rest or one-vs-one.
 Lastly, we demonstrate the efficacy of this method to obtain more accurate multi-class classifiers on popular datasets.

%% file: local_SVM_projection/intro/intro.tex
\section{Introduction} \label{localsvmsection:introduction}

Parametric classification methods form a broad class of statistical and machine learning algorithms for data with categorical dependent variables.
 Typically, parametric methods require stringent assumptions about the distribution of the underlying data, making them inflexible, but providing a limited set of parameters that can be used to interpret the model.
 On the other hand, nonparametric methods forego these assumptions, making them more flexible, but less interpretable.

Methods for creating a nonparametric model family from a parametric one via model localization are common in the statistics community.
A common scheme when the independent variables are continuous is via kernel weighting methods \cite{bottou1992local}, which are a category of ``lazy learners" \cite{atkeson1997locally}.
 Nonparametric regression methods evolving from the kernel-weighting paradigm include the Nadaraya-Watson estimator \cite{nadaraya1964estimating}, LOESS \cite{cleveland1979robust}, and local polynomial models \cite{fan1996local}.
 This scheme is not limited to regression, and the same general idea has been extended to clustering problems \cite{fukunaga1975estimation}, density estimation \cite{terrell1992variable}, anomaly detection \cite{bay2004framework} and surface reconstruction \cite{ozertem2011locally}.
 More recently, learning local representations of data has become an important part of supervised learning in many domains, including image-processing \cite{krizhevsky2012imagenet} and time-series analysis \cite{ye2009time}.

A broad class of local models can be formed by localization of classification models.
 Many properties of systems of local models employing specifically linear SVMs have been investigated.
 Work on the topic begins circa 1992 with the seminal work by Bottou and Vapnik \cite{bottou1992local}, who employ linear SVMs with the square KNN kernel.
 Local modeling remains a popular topic, and improvements to the base algorithm include extensions using the ``kernel trick'' \cite{zhang2006svm}, kernel and bandwidth selection \cite{ralaivola2001incremental}, feature weighting \cite{kecman2010locally}, efficiency heuristics \cite{do2015random} and others.
 Interestingly, despite the seemingly large interest in constructing new classifiers by means of systems of local SVMs and other models, to our knowledge no attention has been paid to the discovery of or the properties of the decision surface of the resulting classifier.
 This is problematic, since many applications of a classification model family depend on the decision surface.
 One obvious example is obtaining scores from a classifier, which are often some function of the orthogonal distance of a point to the decision surface.
 We anticipate that the topic of this investigation will be of use to any future investigations into local SVMs and other varieties of local classifiers. 
 
Although a distance to the \emph{local} decision surface can be readily obtained from the local model at a test point, it is unclear whether these distances can be meaningfully compared between two test points.
 This problem has repercussions for techniques that rely on score comparisons, such as extending classifiers to multi-class problems using the one-vs-rest or one-vs-one schemes.
 It is therefore important to be able to find the decision surface of a system of localized classifiers, as distinct from the decision surfaces of the individual local models.

Unfortunately, although the individual local models can be used to obtain predictions from many classification model families, individual models at arbitrary query points $q$ cannot be used to directly compute the decision surface of the system.
 Recently, algorithms such as Subspace-constrained Mean Shift \cite{ozertem2011locally} have been developed that use local models to perform the superficially similar problem of surface reconstruction in point-cloud problems.
 We will show that this approach can be modified to obtain the decision surface for localized classifiers such as systems of local support vector machines.

The modified Subspace-constrained Mean Shift that we introduce can be used to find points lying on the decision surface, but it is difficult to predict precisely where on the surface the algorithm will terminate.
 To remedy this, we introduce a simple enhancement to the algorithm that allows resampling of the surface along a particular direction.
 For well-behaved classifiers, we demonstrate that this resampling technique can be used to find the orthogonal projection of an arbitrary point onto the decision surface.
 
Armed with the decision surface, and the orthogonal projections, we can readily apply calibration techniques to obtain pseudo-probabilistic output from the resulting model. 
 We then use the resulting probabilities to tackle multi-class classification problems.
 To demonstrate the effectiveness of the algorithm, we apply this technique to a simple toy dataset, and also to a handwritten digits classification task \cite{alpaydin1998cascading}.

%% file: local_SVM_projection/background/background.tex
\section{Background} \label{localsvmsection:background}

\subsection{Local Models}

Local regression methods are frequently employed as nonparametric \emph{smoothing} procedures.
 Such popularly used methods include local constant models (Nadaraya-Watson) \cite{nadaraya1964estimating}, local linear models (LOESS) \cite{cleveland1979robust}, and local polynomials \cite{fan1996local}.
 The procedure used to localize these models can be applied to arbitrary models that allow a weighted learning procedure, and that make predictions \cite{vapnik1992principles}.
 The authors of \cite{bottou1992local} apply this same concept to obtain predictions from localized linear classifiers.
 
The general idea of local modeling is to fit a model to a locally-weighted neighborhood of the training data, and then to make a prediction with the obtained local model.
 If the weighting scheme and loss function of the model family are in $C^m$, and the loss function is convex, then the parameters of the local models will generally be $C^m$ in the query point as well, resulting in well-behaved predictions across the space.
 Although this process can apply to many model families, simple models are the most popular candidates for localization, since convexity is a desirable property of the candidate model family, and also since the reduction in training data caused by the localized-weighting step often opposes the construction of complex local models.
 
Still, other applications for local models than mere predictions have suggested more complex models for localization. 
 In \cite{bay2004framework}, the authors use the model parameters of localized ARMA models for anomaly detection.
 In \cite{brown2019local}, the authors use the model parameters of localized Gaussian Process models for time series classification.

We formalize the notion of local models as follows: 
 Suppose a model family $F$ together with some training algorithm $A$ that admits a weighted training scheme.
 Each training point $x_i\in X \subset D$ is weighted as $w_i \in {\rm I\!R}$, so that $A: \mathcal{P}(D\times {\rm I\!R}) \rightarrow F$ for some domain $D$.
 We define the corresponding system of localized models $F'$ to give a single model in $F$ for each point in the domain $D$, $F' = \prod_{q \in D} F$. 
 The localized algorithm $A'_K: \mathcal{P}(D) \rightarrow F'$ extends $A$ by considering the weights $w_i = K(q,x_i)$ to be a function of the data $X \subset D$ and some query point $q \in D$.
 Thus, $A'_K$ takes in a dataset $X \subset D$, and returns a set of models, one $f_q \in F$ for every point $q \in D$.
 Some local algorithms, such as LOESS, only require that $q \in D^*$ where $D^*$ is some subspace of independent variables of $D$, but we will not formalize this distinction in our analysis.
 In the proceeding, the weight function $K$ and other functions are assumed to handle the division between dependent and independent variables appropriately, and we will abuse the notation accordingly.

If the training algorithm $A$ minimizes a loss function that is the sum of individual model errors: 

\begin{equation} \label{localsvmloss}    
    L(X) = \sum_{x \in X} E(x,f)
\end{equation}

Then the localized algorithm is easy to obtain as: 
 
 \begin{equation} \label{localsvmlocalized_classifier}
    A'_K(X) = \mbox{argmin}_{f_q \in F} \sum_{x \in X} E(x,f_q)K(q,x)
 \end{equation}

Support vector machines happen to be of this type.
 $K(q,x)$ is usually called the ``kernel", and many choices of kernel are common in the literature.
 One popular choice is a square KNN-kernel, as employed with locally linear SVMs in \cite{bottou1992local} and \cite{kecman2010locally}.

\begin{equation} \label{localsvmknn_kernel}
     K(x,y) = \begin{cases} 
          1 & ||x - y||^2 \leq ||x-y_k||^2 \\
          0 & \text{else}
       \end{cases}
\end{equation}
where $y_k$ is the $k$th-nearest training point to $x$.
 We will employ the Gaussian Kernel with a fixed radius (the ``bandwidth'') in our experiments, because it is smooth:

\begin{equation}\label{localsvmgaussian_kernel}
  K(q,x) = \frac{1}{\sqrt{2\pi}} e^{-\frac{||x-q||^2}{2h^2}}
\end{equation}

Yet other functions can provide a both smooth and variable-width kernel if desired.
 Although optimal bandwidth selection is an important problem in local modeling, it is not the focus of this work.

In this notation, the Nadaraya-Watson estimator can be obtained by taking $F$ to be constant functions, and letting $E$ be ordinary least-square errors against the dependent dimensions of $x$.
 LOESS can obtained by taking $F$ to be linear functions of the independent dimensions of $x$, and $E$ as ordinary least-square errors.

Often, a particular use case of a system of localized models has the convenient property that the desired quantity is directly obtainable from the learned model $f_q$.
 For example, LOESS is commonly used for scatterplot smoothing, requiring predictions at arbitrary independent variable inputs.
 Since the $f_q$ returned by training an ordinary least squares model represents a linear function of the independent parts of $x$, we can easily obtain a prediction from $f_q$ at the query point $q$ by simply evaluating the function.
 
Unfortunately, it is not always the case that the desired quantity can be directly obtained from the learned model $f_q$.
 For example, points on the decision surface of a system of localized classifiers cannot be obtained from the individual $f_q$ at arbitrary points $q$.
 This is the problem of the current investigation, the solution to which is inspired by another class of local modeling methods that also suffer from this issue.

\subsection{Mean Shift}

Although they incorporate local modeling methods, the Mean Shift \cite{comaniciu2002mean} and Subspace-constrained Mean Shift \cite{ozertem2011locally} algorithms are more typically considered to be a subset of the theory surrounding kernel density estimation \cite{terrell1992variable}.
 This is to some extent because, although Mean Shift is often used as a clustering algorithm, it is theoretically convenient that it happens also to find the modes of the kernel density function.
 Likewise, it is theoretically convenient that Subspace-constrained Mean Shift happens to find the ``ridges'' of the kernel density function \cite{ozertem2011locally}.

We can also view Mean Shift and variants as algorithms that operate on systems of localized models.
 The primary difference between these methods and the methods described in the previous section is that the use case does not possess the convenient property that the desired output is directly obtainable from the learned model $f_q$.
 MeanShift and variants, rather than seeking properties of arbitrary query points $q$, seek to find query points that satisfy a particular property.
 Namely, they seek points $q$ that have $0$ error evaluated against their own local regression model.
 For LOESS, the distinction between dependent and independent variables causes this to be a property at \emph{all} query points. 
 Since the model families of MeanShift variants do not distinguish between independent and dependent parts, these points are slightly more difficult to obtain.
 For these models, the points that lie on their own regression surface can be found by projecting a point iteratively onto its local regression surface, ``chasing down'' the regression surface in a sense.
 It is easy to see that such an iterative method can only terminate at a stationary point of the local modeling/projection algorithm.
 Whether or not the algorithm ever terminates is another question entirely, and whether or not Mean Shift converges in the general case remains an open question \cite{ghassabeh2013some}.

In the notation introduced above, Mean Shift operates on the family $F$ of constant functions returning a vector of the same dimensionality as $D$, and with $E$ being total least-square errors against $x$.
 We can then obtain a projection from $q$ onto the image of $f_q$.
 By iterating this projection scheme, we obtain an algorithm that terminates (if it terminates) in a $0$-dimensional set, which is convenient as a clustering algorithm, and which we call Mean Shift.
 Subspace-constrained Mean Shift can be obtained by taking $F$ to be $m$-dimensional linear embeddings in $D$, and letting $E$ be total least-square errors against $x$.
 By iteratively projecting a point onto the surface of the local total least-squares regression model, we obtain a point on an $m$-dimensional surface reconstruction of the input data $X$ \cite{ozertem2011locally}.
 We will replicate this intuition of ``chasing down'' the local surface and apply it to systems of localized classifiers in the proceeding section.

%% file: local_SVM_projection/methodology/methodology.tex
\section{Methodology} \label{localsvmsection:methodology}

In the following, we will refer to properties of an ``(individual) local model", in which we mean a single model trained at some query point $q$.
 For example, for local classifiers, individual local models each possess a unique decision surface.
 On the other hand, if we form a new classifier $g$ by combining the predictions of the individual models so that $g(q) = f_q(q)$, then $g$ has a decision surface that is distinct from the decision surface of any individual local model.
 We distinguish these as properties of the ``system of localized models'', in which we mean the class of algorithms resulting from the process of training many local models and combining the results in some way.

For classification models, we can obtain a prediction at a query point $q$ from the individual local models $f_q$ directly, so that nothing further is required whenever predictions are the ultimate goal.
 It is often desirable to obtain additional properties of the trained model when employing classification models.
 For example, when extending an algorithm to multiclass problems via a one-vs-rest or one-vs-one strategy, it is necessary to obtain \emph{scores}.
 For such applications, a natural scoring mechanism is to obtain the orthogonal distances to the decision surface.
 The decision surface of a system of localized classifiers is not naively given by the decision surfaces of or any other property of the individual $f_q$ at arbitrary query points $q$.

However, it is easy to see that a point $q$ will lie on the decision surface of the system of localized classifiers if and only if it lies on the local decision surface at $q$.
 To be clear, if we let $A$ be some family of simple binary classifiers, then the combined predictor $g$ described above predicts $1$, $0$ or ``decision surface'' precisely when $f_q$ does, by definition.
 If we adopt an iterative projection scheme analogous to Mean Shift variants, then the decision surface of the system is precisely the stationary points of the projection scheme.

Letting $proj(x,Z)$ be the orthogonal projection of $x \in D$ onto $Z \subset D$, and $DS(f)$ the decision surface of $f$:

\begin{algorithm}[!t] 
    \caption{Localized Classifier Decision Surface Projection}
    \label{localsvmscms_algorithm}
    \alg{Localized Classifier Decision Surface Projection}
\begin{algorithmic}
  \STATE given $A, K, X, q$
  \STATE $y \leftarrow q$
  \REPEAT
    \STATE $f_q \leftarrow A_{K(q,X)}$
    \STATE $y \leftarrow proj(y, DS(f_q))$
  \UNTIL {convergence}
  \RETURN $y, f_q$
\end{algorithmic}
\end{algorithm}

Although this algorithm will converge to a point on the decision surface if it converges, the fact that there is currently no general proof for the convergence of Mean Shift \cite{ghassabeh2013some} causes us to suspect that a proof of convergence of this algorithm is non-trivial.
 We will therefore forego an attempt at a proof and provide empirical evidence instead, provided in Section \ref{localsvmsection:results}.

Unfortunately, this scheme generally will not proceed toward the \emph{nearest} point on the decision surface.
 Therefore, we introduce some additional mechanics to obtain the orthogonal projection. 
 The general idea is to iteratively project a query point toward the decision surface of a single local model defined at that point along some initial vector $v$.
 This algorithm terminates (if it terminates) at a point on the global decision surface for the localized classifier in the direction of $v$ if one exists.
 $v$ is then iteratively adjusted toward the surface normal at that point.
 The result is the orthogonal projection of the query point onto the decision surface, which can then be used toward, e.g. Platt calibration of the full model.

Suppose that the kernel $K$ is continuously differentiable, the error function $E$ is continuously differentiable, and the loss function $L$ is convex.
 Then the localized classification algorithm $A'$ in Equation \ref{localsvmlocalized_classifier} is a composition of continuously differentiable functions and is therefore itself continuously differentiable.
 This is a convenient property, since it implies that the scores obtained from the orthogonal distance from $q$ to $DS(f_q)$ can thus be stitched together into a single continuously differentiable function $f^*(q) = ||q - proj(q, DS(f_q))||$.
 The implicit function theorem gives that the level set $f^*(q) = 0$ is locally a differentiable function of one of the components of $q$ whenever the gradient is not the zero vector.
 Note that although some loss functions such as hinge loss are not everywhere differentiable, we only require differentiability along the decision surface. 
 For hinge loss, this will only occur when one of the training data lie exactly on the local decision surface.
 This is almost guaranteed to occur at some points, but for finite training sets, this set of points will be limited.
 Thus, even for common non-differentiable loss functions, the normal vectors to the level set exist and are differentiable almost everywhere, allowing us to apply a wide range of optimization strategies to the direction of our projection.
 
In order to perform various optimization algorithms against the direction of projection, we require a means to resample the decision surface in a orderly fashion.
 We amend algorithm \ref{localsvmscms_algorithm} to allow constraint to a particular direction vector $v$.

\begin{algorithm}[!t]
    \caption{Constrained Decision Surface Projection}
    \label{localsvmconstrained_scms_algorithm}
    \alg{Constrained Decision Surface Projection}
\begin{algorithmic}
  \STATE given $A, K, X, q, v$
  \STATE $y \leftarrow q$
  \REPEAT
    \STATE $f \leftarrow A_{K(q,X)}$
    \STATE $y \leftarrow proj(y, DS(f) \cap \{x | x = y + \lambda v\})$
  \UNTIL {convergence}
  \RETURN $y, f$
\end{algorithmic}
\end{algorithm}

It is important here to choose $v$ so that there exist points on the decision surface in that direction.
 However, it is easy to check if this algorithm has converged to the decision surface, by simply evaluating $proj(y, DS(f))$ on the result, and ensuring that it gives $y$.

For classifier families with a linear decision surface, the constrained projection is given by:

$$ proj(y, DS(f) \cap \{x | x = y + \lambda v\}) = y - \frac{y \cdot n}{v\cdot n} v $$
where $n$ is the normal vector to the local decision surface.
 For non-linear classifiers, a means to project onto the decision surface along a particular direction may or may not have such a convenient closed-form.

By varying the vector $v$ along which we make our projection, we can search for a point on the decision surface that is closest to our query point $q$.
If we further require that $K$ and $E$ are twice differentiable, then the normal vectors to the decision surface will also be differentiable, and we can therefore perform ordinary gradient-dependent optimization algorithms on $v$ to minimize its difference with the surface normal.
 Again, we note that we only require this property along the decision surface, so that loss functions that are twice differentiable almost everywhere are likely to not exhibit problems for non-contrived datasets.

First, we run the above described iterative projection procedure to find \emph{some} point on the decision surface, $y_0$.
 Since the level set at that point is a differentiable function within some $\epsilon$ ball by the implicit function theorem, we consider the level surface at that point in the reference frame where $n_i$, the normal vector to the surface at $y_i$, is the independent variable. 
 The level surface near $y_i$ can thus be written as a function $g(a) \propto n_i$, where $a$ is a vector in the null space of $n_i^T$.
 The gradient of the distance between $q$ and $y_i$ can be taken with respect to this new domain to obtain:

$$ \frac{\partial}{\partial a} ||q-y_i||^2 = 2 rej_{n_i}(y_i - q) $$
where $rej_{n_i}(x)$ is the vector rejection of $x$ onto $n_i$.
 Since actually writing our decision surface as a closed-form function is infeasible, we would like to make a small step along the surface in this direction instead.
 We can accomplish this by constraining the projection step of our iterative algorithm from $q$ to be some convex combination of the vector $n_i$ and $(y_i - q)$.
 If $y_{i+1}$ is forced to lie along such a line from $q$, then:

 $$\begin{array}{l} rej_{n_i}(y_{i+1}) = \\
    rej_{n_i}(q + \lambda((1-\alpha)(y_i - q) + \alpha n_i)) = \\
    rej_{n_i}(q) + C_1 * rej_{n_i}(y_i - q) = \\
    rej_{n_i}(y_i) + C_2 * rej_{n_i}(y_i - q) = \\
    rej_{n_i}(y_i) + C_2 * \frac{\partial}{\partial a} ||q-y_i||
    \end{array}
 $$
where $C_j$ are some constant values.
 We therefore obtain a simple gradient descent algorithm.
 This assumes that our constrained iterative projection algorithm actually converges to some point $y_i$.
 For an arbitrarily chosen direction, this is not guaranteed.
 Still, since there exists some $\epsilon$ on which our level surface is a function, as long as we constrain $\alpha$ to be small enough that our algorithm does not leave that $\epsilon$ ball, the existence of $y_i$ is assured.
 It is not clear how to find the size of $\epsilon$, so we will satisfy ourselves with simply choosing $\alpha$ to be ``small", in the same sense that ordinary gradient descent chooses a ``small" step size.

We can now apply the gradient descent process described above to obtain an orthogonal projection onto the decision surface:

\begin{algorithm}[!t]
    \caption{Orthogonal Decision Surface Projection}
    \label{localsvmorthogonal_scms_algorithm}
    \alg{Orthogonal Decision Surface Projection}
\begin{algorithmic}
  \STATE given $A, K, X, q, \alpha$
  \STATE $y, f \leftarrow $ Algorithm \ref{localsvmscms_algorithm}$(A, K, X, q)$
  \REPEAT
    \STATE $n \leftarrow (\nabla S)(y)$
    \STATE $v \leftarrow (1-\alpha)(y - q) + \alpha n$
    \STATE $y,f \leftarrow $ Algorithm \ref{localsvmconstrained_scms_algorithm}$(A,K,X,q,v)$
  \UNTIL {convergence}
  \RETURN $y$
\end{algorithmic}
\end{algorithm}
where $S$ is the decision surface of our system of local classifiers.
 Note that if algorithm \ref{localsvmorthogonal_scms_algorithm} fails to converge to the decision surface due to having ``missed" it via poor selection of $v$, a simple heuristic to allow the algorithm to continue is to simply apply algorithm \ref{localsvmscms_algorithm} to the result.
 We have found that this generally gives good results.

What remains is a convenient way to compute $(\nabla S)(y)$.
 One possibility is to perform a finite difference by taking a grid in some small neighborhood of $y$ and applying algorithm \ref{localsvmscms_algorithm} to obtain points on the decision surface $S$.
 We can then compute a hyperplane estimation from those points with, for example, total least-squares regression, which gives an approximation to $(\nabla S)(y)$.
 
Finite difference approximations are generally inferior to direct differentiation, which can be computed for most loss functions with a nicely-behaved second derivative.
 For classifiers with linear decision functions, the derivative has a particularly nice form.
 Although we limit our proceeding discussion to linear classifiers, much of the reasoning can be applied to classifiers with non-linear decision functions, but must be handled on a case-by-case basis.

Suppose that our classifier family has a linear decision function.
 Then the distance from $q$ to the projection of $q$ onto the decision surface of the local model centered at $q$ can be written in terms of the parameters of the linear model: 

\begin{equation} \label{localsvmfstarq}
 f^*(q) = q^T\mathbf{n}^* - \lambda^*
\end{equation}
where $\mathbf{n}$ and $\lambda$ are the normal vector to the local decision surface and the offset from the origin, respectively.
$\mathbf{n}^*$ and $\lambda^*$ are the optimal values computed by our training algorithm localized at $q$.
 The derivative of this w.r.t $q$ is:

\begin{equation} \label{localsvmdfstar_dq}
    \frac{\partial f^*}{\partial q} = \mathbf{n}^{*} + q^T \frac{\partial \mathbf{n}^*}{\partial q} - \frac{\partial \lambda^*}{\partial q}
\end{equation}

It is very often the case that $\mathbf{n}^*$ and $\lambda^*$ are the result of an argmin operation on some loss function $L$:

$$ \mathbf{n}^*, \lambda^* = \mbox{argmin}_{\mathbf{n}, \lambda} L(X,q,\lambda,\mathbf{n}) $$

Unfortunately, argmin does not have a convenient derivative.
 A common trick to differentiate such functions is to rely on the fact that, for differentiable loss functions $L$, the minimum is obtained only when the derivative w.r.t the argmin variables is $0$.
 If $L$ is differentiable and convex, then:

$$ \frac{\partial}{\partial \mathbf{n}} L(X,q,\lambda,\mathbf{n}) |_{\mathbf{n} = \mathbf{n}^*} = 0 $$
where $\cdot|_{a=b}$ denotes substitution of $a$ with $b$.
 Taking a derivative of both sides of this formula w.r.t $q$ yields:

$$ \frac{\partial^2 L}{\partial q \partial \mathbf{n}} |_{\mathbf{n} = \mathbf{n}^*} + \frac{\partial^2 L}{\partial \mathbf{n}^2} |_{\mathbf{n} = \mathbf{n}^*} \frac{\partial \mathbf{n}^*}{\partial q} = 0 $$

The $\partial^2 L / \partial \mathbf{n}^2$ term is the Hessian of $L$ w.r.t $\mathbf{n}$.
 If the Hessian is invertible at $\mathbf{n}^*$, we can continue to solve this equation for $\partial \mathbf{n}^* / \partial q$:

\begin{equation} \label{localsvmdn_dq}
    \frac{\partial \mathbf{n}^*}{\partial q} = -\left(\frac{\partial^2 L}{\partial \mathbf{n}^2} |_{\mathbf{n} = \mathbf{n}^*} \right)^{-1} \frac{\partial^2 L}{\partial q\partial \mathbf{n}} |_{\mathbf{n} = \mathbf{n}^*}
\end{equation}

A parity of reasoning holds for $\partial \lambda^* / \partial q$:

\begin{equation} \label{localsvmdl_dq}
    \frac{\partial \lambda^*}{\partial q} = -\left(\frac{\partial^2 L}{\partial \lambda^2} |_{\lambda = \lambda^*} \right)^{-1} \frac{\partial^2 L}{\partial q\partial \lambda} |_{\lambda = \lambda^*}
\end{equation}

Specifically for SVM classifiers with squared hinge loss, the Hessian above is invertible whenever none of the training data lie on the decision surface of the local model centered at $q$. 
 We can then compute a nice closed form for $\nabla S$.
 If we let $SV(X)$ be the set of support vectors for $X$ (i.e. $1-y(x^T\mathbf{n} - \lambda) > 0$), then the squared hinge loss is given by:

\begin{equation} \label{localsvmsquared_hinge_loss}
    L(X,\lambda,\mathbf{n}) = \sum_{x,y \in SV(X)} (1-y(x^T\mathbf{n} - \lambda))^2
\end{equation}

Applying the localization process in equation \ref{localsvmlocalized_classifier} gives the following loss function for our system of localized SVMs:

\begin{equation} \label{localsvmlocal_svm_loss}
    L(X,q,\lambda,\mathbf{n}) = \sum_{x,y \in SV(X)} (1-y(x^T\mathbf{n} - \lambda))^2 K(q,x)
\end{equation}

The derivatives of this w.r.t $\mathbf{n}$, $\lambda$ and $q$ exist and are straightforward.
 Employing these and equation \ref{localsvmfstarq} directly in equation \ref{localsvmdn_dq}, we obtain:

\begin{align*} 
   & \frac{\partial \mathbf{n}^*}{\partial q} = 
         \left( \sum_{x,y \in SV(X)} x x^T K(q,x) \right)^{-1} * \\
   &     \left(\sum_{x,y\in SV(X)} yx^T \frac{\partial K}{\partial q} -
         x(x - q)^T \frac{\partial K}{\partial q} \mathbf{n}^{*T}\right)
\end{align*}

And similarly in equation \ref{localsvmdl_dq}:

\begin{align*} 
   & \frac{\partial \lambda^*}{\partial q} = 
        \left( \sum_{x,y \in SV(X)} K(q,x) \right)^{-1} * \\
   &     \left( \sum_{x,y\in SV(X)}(x - q)^T \frac{\partial K}{\partial q} \mathbf{n}^{*T} -  
          y \frac{\partial K}{\partial q} \right)
\end{align*}

These can be employed in equation \ref{localsvmdfstar_dq} to exactly compute $\nabla S$ at $q$ satisfying the above conditions, and lying on the decision surface of the system of localized classifiers.
Note that the inverse term in $\partial \mathbf{n}^*/ \partial q$ is the inverse of the weighted covariance of the independent parts of $X$, which is generally expected to be invertible for non-contrived data.
 Also, if we choose a kernel $K$ that is not everywhere 0, the denominator in $\partial \lambda^*/ \partial q$ exists as well.

If the data are well-balanced near $q$, then these two inverse terms should be comparatively large.
 On the other hand, the sum-over-$y$ terms should be relatively small since they involve the sum of many comparably-sized positive and negative terms.
 The sums involving an $(x-q)$ term should also be relatively small, since there will similarly be much canceling out of vectors pointing in various directions around $q$.
 Thus, at points that are within the bounding box of the training data, the largest contributor to $\partial f^* / \partial q$ will generally be the unit $\mathbf{n}^*$.
 We can therefore approximate $(\nabla S)(y)$ with the normal vector of the local decision surface at $y$.
 All of our experiments in section \ref{localsvmsection:results} employ this approximation.
 Comparing this approximation with the finite difference formula proposed above, we have found that for squared hinge loss SVM, for points on the decision surface within the bounds of the training data, this approximation holds to a high degree of accuracy.
 
We also yet again note that we only require a non-zero second derivative along the decision surface, so that even loss functions that are not differentiable everywhere may still be differentiable almost everywhere on the decision surface.
 Our experiments show that the approximation of the normal vector of the decision surface of a system of localized SVMs with hinge loss is reasonably well approximated by the normal vector of the local model at a point lying on the decision surface.
 We demonstrate this in the proceeding section.

%% file: local_SVM_projection/results/results.tex
\section{Results} \label{localsvmsection:results}

In this section we apply the algorithms described in the previous section with a linear support vector classifier on both a toy dataset and a handwritten digit classification task \cite{alpaydin1998cascading}.
 We employ the LinearSVC class of the scikit-learn library, which provides a wrapper for the liblinear \cite{fan2008liblinear} support vector classifier implementation.
 Since the loss functions available to SVM classifiers are of the sum-of-individual-errors form, we employ the localization scheme described in equation \ref{localsvmlocalized_classifier}.

\subsection{Toy Data}
 
The toy dataset shown in Figure \ref{localsvmfigure:noisy_moons} was generated from the scikit-learn ``make\_moons'' method, using a noise level of $0.06$, mean-centered and scaled to have unit variance in both $x$ and $y$.
 For our local models, we employ a Gaussian kernel (equation \ref{localsvmgaussian_kernel}) with bandwidth parameter $h=0.6$.
 This parameter was cherry-picked to give pleasing results.
 Note that as $h \rightarrow \infty$, the system of local linear SVMs approaches a single global linear SVM.
 We found that for very low values of $h$, the resulting vanishing weights caused problems with liblinear, and the individual models would fail.
 For the $\alpha$ parameter of algorithm \ref{localsvmorthogonal_scms_algorithm}, we have found that a wide range of values are viable.
 We have found that $\alpha$ values up to $0.5$ result in pleasing convergence properties, and we note that setting $\alpha$ as high as possible will generally enable the algorithm to converge more quickly.

\begin{figure}[!t]
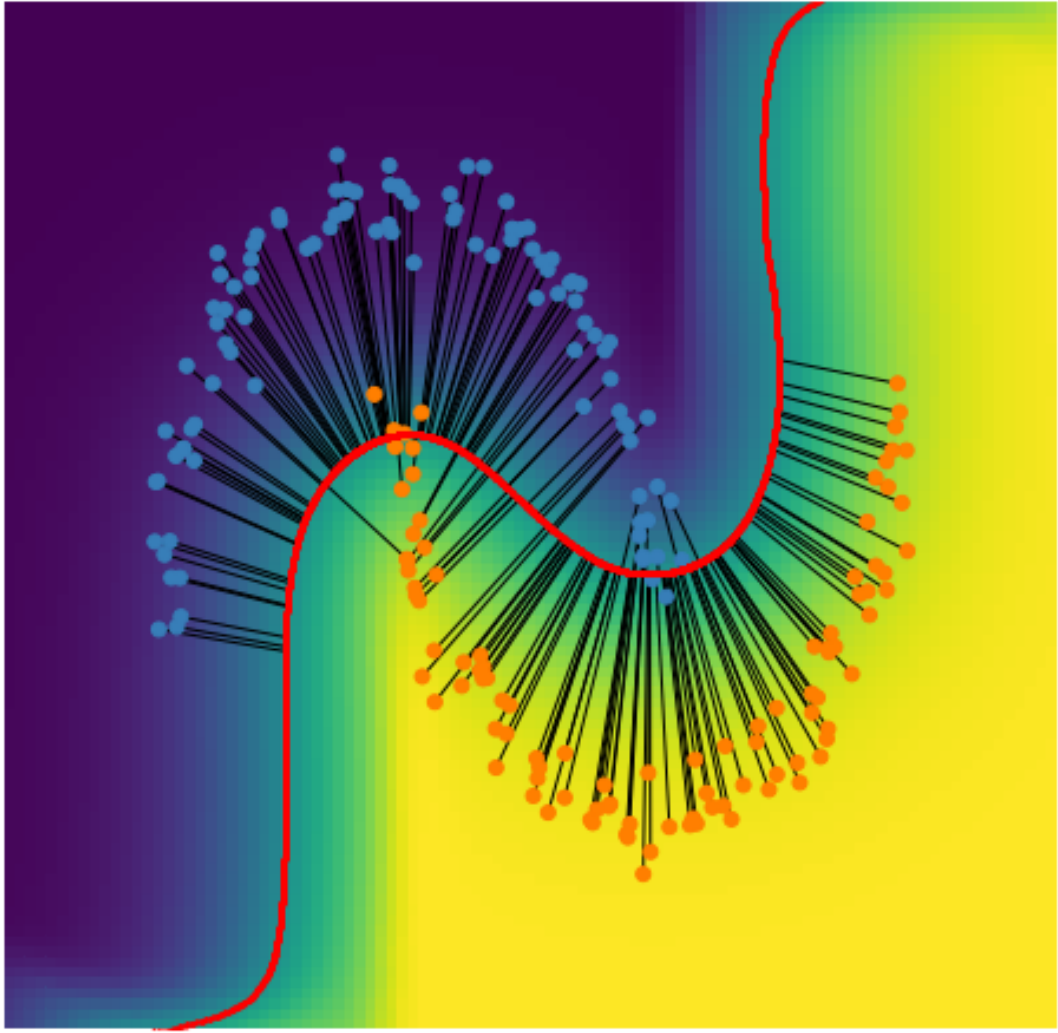

    \centering
    \includegraphics[width=\columnwidth]{local_SVM_projection/static/{{noisy_moons_local_svm}}}
    \caption[Local SVM decision surface on toy data, square hinge loss]{Local SVM decision surface on toy data, square hinge loss.
        The scikit-learn ``noisy moons'' dataset in blue and orange dots.  
        Dot color represents the ground-truth binary label.
        Decision surface in red.
        Global scores evaluated on a regular grid in background color gradient.
        Orthogonal projections of training data onto decision surface in black.}
    \label{localsvmfigure:noisy_moons}
\end{figure}

In addition to the training data, Figure \ref{localsvmfigure:noisy_moons} shows the decision surface of the system of localized classifiers in red, found via algorithm \ref{localsvmorthogonal_scms_algorithm} applied on a grid.
 The squared hinge loss is employed, resulting in a differentiable decision surface.
 Also shown are the orthogonal projections of the training data onto the decision surface found via algorithm \ref{localsvmorthogonal_scms_algorithm}, shown as black vectors.
 Lastly, the scalar field of orthogonal distances from each point in a dense grid to the red decision surface, as the background color gradient.
 Since these orthogonal distances are globally relevant, we can scale these to the unit interval to obtain pseudo-probabilistic output, as we will show in the next section.
 In these experiments, we have approximated the surface normal at a point $q$ on the surface by the normal of the local SVM centered at $q$.
 We also computed a finite difference using a neighborhood of points on the surface for comparison.
 For query points within the bounding box of the training data, the average cosine similarity between the approximate normal and the finite difference normal was $0.993$, with a standard deviation of $0.0079$.
 $99.9\%$ of approximated normals within these bounds had a cosine similarity of at least $0.968$.
 Points outside of the bounds of the training data were less likely to be well approximated.
 It is well-known that kernel methods tend to have trouble extrapolating, and when moving significantly outside of the range of the training data, the decision surface itself is probably not trustworthy.

We repeat the experiment for the hinge loss, which is not differentiable, to show how this affects the results of the algorithm.
 Results are in Figure \ref{localsvmfigure:noisy_moons_hinge}.
 Note that even though the decision surface is not everywhere well-behaved, it is well-behaved enough so that algorithm \ref{localsvmorthogonal_scms_algorithm} produces reasonable results for the orthogonal distances.
 For query points within the bounding box of the training data, the average cosine similarity between the approximate normal and the finite difference normal was $0.911$, with a standard deviation of $0.216$.
 $80\%$ of approximated normals within these bounds had a cosine similarity of at least $0.943$.
 Thus, the approximation is not nearly as good for the hinge loss.
 Nevertheless, algorithm \ref{localsvmorthogonal_scms_algorithm} seems to employ it to good effect, as can be seen in the figure.
 As expected, since the hinge loss is not differentiable, the decision surface is not differentiable either.
 Despite appearances, the decision surface is continuous, but our algorithm for finding orthogonal projections is imperfect for non-differentiable loss functions, and small crags in the surface are therefore difficult to discover.
 Although the hinge loss provides satisfactory results for individual linear SVMs, because of these problems with systems of localized SVMs utilizing hinge loss, we recommend that researchers and users of localized systems of SVMs consider employing squared hinge loss instead.

\begin{figure}[!t]
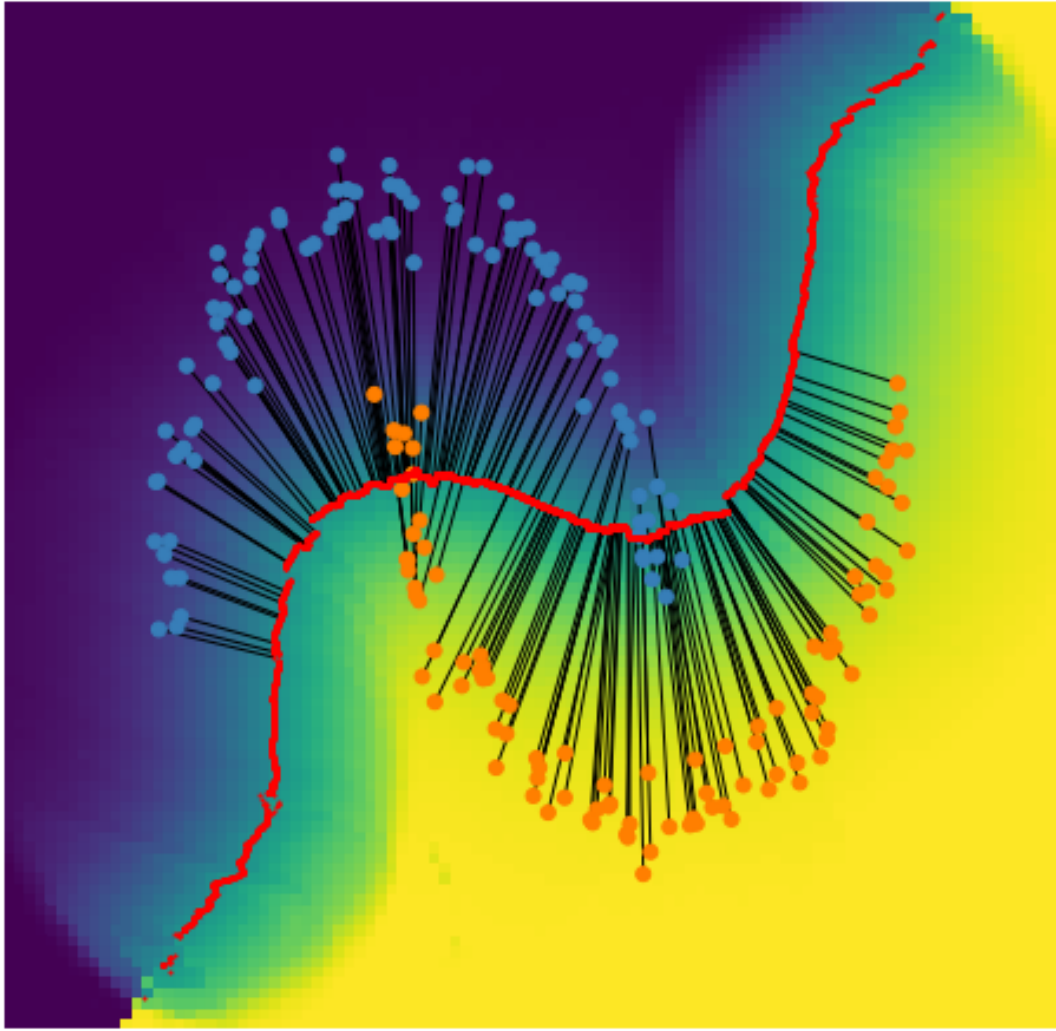

    \centering
    \includegraphics[width=\columnwidth]{local_SVM_projection/static/{{noisy_moons_local_svm_hinge}}}
    \caption[Local SVM decision surface on toy data, hinge loss]{Local SVM decision surface on toy data, hinge loss.
        As Figure \ref{localsvmfigure:noisy_moons}, but employing the hinge loss.
        The decision surface is not very well-behaved.
        Still, algorithm \ref{localsvmorthogonal_scms_algorithm} gives reasonable estimates for the orthogonal projection.
    } 
    \label{localsvmfigure:noisy_moons_hinge}
\end{figure}

\subsection{Handwritten Digits}

For this dataset, we employ a one-vs-rest strategy to obtain a multiclass classifier.
 To this end, we apply Platt Scaling \cite{platt1999probabilistic} in the ordinary fashion.
 Due to the small size of the dataset, we have simply set a regularization factor for Platt Scaling a priori at $0.01$ rather than using cross-validation to select it automatically.
 We also set the bandwidth to be fixed a priori at $20$ times the average 1-nearest neighbor distance across all training points.
 We set the SVM regularization to be the library default, which is $1$.
 Choosing better hyperparameters might reasonably improve the results, but we stress that the primary motivation with these experiments is to illustrate a use case of the decision surface for a system of localized classifiers, rather than to earnestly tackle the problem of handwritten digit classification.

We employ here the squared hinge loss to define a system of localized SVMs for each digit label against the rest.
 For each of these, we compute the orthogonal distance of the training points to the decision surface for the system of localized SVMs.
 We then train a logistic regression with this orthogonal distance multiplied by the sign of the local prediction at the training point as input.
 The ground-truth label is taken as output, and the resulting model is used to obtain pseudo-probabilistic scores for each digit.
 To perform multi-class classification, we compare the probabilities of each model, and choose the digit that scores the highest.

We validate this process on a $30\%$ test split in Table \ref{localsvmtable:handwritten_digit_results}.
 Algorithms \ref{localsvmscms_algorithm} and \ref{localsvmorthogonal_scms_algorithm} were seen to converge for all of our training and test points.
 We see that this produces modest improvements upon the linear classifier, which we might reasonably expect a non-linear classifier to accomplish.
 These results demonstrate that computing scores via orthogonal distance to the decision surface of the system of localized classifiers is viable for high-dimensional real world datasets. 
 
\begin{table}[!t]
  \begin{center}
    \caption{Local SVM one-vs-rest Classifier.}
    \label{localsvmtable:handwritten_digit_results}
    \begin{tabular}{l|l|l|l}
                    & Accuracy & Precision & Recall \\
        Local SVMs  & 0.961    & 0.960     & 0.960  \\
        Linear SVM  & 0.948    & 0.947     & 0.948    
    \end{tabular}
  \end{center}
\end{table}

%% file: local_SVM_projection/conclusion/conclusion.tex
\section{Conclusion} \label{localsvmsection:conclusion}

In this paper we describe an algorithm that, if it converges, converges to points lying on the decision surface of a system of localized classifiers.
 To our knowledge there currently exists no other algorithm to find points on the decision surface.
 Although we have provided no proof of convergence, and suspect from relationships to other algorithms that such a proof would be difficult to come by, these same relationships and our experiments suggest that the algorithm has nice convergence properties nonetheless.

We further extend this algorithm to find points on the decision surface along an arbitrary direction.
 This improves upon the naive algorithm which only finds an arbitrary point on the decision surface.
 We have demonstrated that this algorithm can be used to resample points on the decision surface in an orderly fashion.
  
We use this algorithm to develop a gradient-descent technique for finding the orthogonal projection of a point onto the decision surface.
 We demonstrate how the orthogonal projection, and the corresponding orthogonal distance, can be used as scores for pseudo-probabilistic regularization on both toy data and real-world datasets.

For the future, we conjecture that the intuition behind both algorithm \ref{localsvmconstrained_scms_algorithm} and algorithm \ref{localsvmorthogonal_scms_algorithm} would work equally well as an extension to Subspace-constrained Mean Shift.
 These could be used to resample grids over principal surfaces so that a triangular mesh might be obtained, or to find orthogonal projections onto the reconstructed surface which ``is one of the most critical operations in computer aided geometric design and applications'' \cite{ko2014orthogonal}.
 Other future work would involve techniques to improve the efficiency of the proposed algorithms.
 For example, memoization might be employed to more quickly find nearly-orthogonal points on the decision surface.
 Also, we suspect that adapting the resampling algorithm to start at points closer to the decision surface would provide significant speed improvement.

%% file: local_W2V_v_time/abstract.tex
In this work, we introduce a means for monitoring time-localized embeddings of words over time for the purpose of detection of events in real-time as they unfold.
 We extend the general concept of local learning as used in the LOESS algorithm to word embedding models to extract locally-relevant features from data for further learning.
 We apply this concept to continuous-bag-of-words word2vec models trained incrementally over text time-series data.
 We use the resulting feature set as an indicator of real-world events that correlate with changing word usage over time.
 To this end, we have used the twitter streaming API to gather a dataset of several time-series of tweets surrounding mass-shooting events.
 We make publicly available both our libraries for local learning and the dataset of tweets over time for further research.

%% file: local_W2V_v_time/intro/intro.tex
\section{Introduction} \label{localw2vsection:introduction}

Event detection via real-time social media feeds has been a popular research topic for some time \cite{atefeh2015survey,hasan2018survey,weiler2017survey}.
 Particularly attractive from a research perspective is the Twitter platform, since it provides an easy means for researchers to gather datasets.
 Many Twitter-based event detection schemes hinge on monitoring the volume of tweets \cite{hasan2018survey}, typically conditioned on a particular term or set of terms.
 Sets of related terms are often found via some model of the relationship between terms, such as semantic relationships \cite{zhang2015event} or temporal correlations in usage \cite{parikh2013events}.
 Regardless of the complexity of the semantic model used, the nature of volume-centric methods generally cannot account for changing \emph{context} of a word.
 For example, we might use a volume-based method to detect `typhoon' \cite{sakaki2010earthquake}, because the word `typhoon' is unambiguous, and is not a frequent topic of conversation except when there happens to be an active typhoon.
 However, it might be more difficult to detect a `shooting', since the term `shooting', its root `shoot' and grammatical variants can be used in multiple contexts.
 In fact, a much more common event than someone being shot by a firearm is someone being `shot' by a camera.
 This ambiguity can make volume-based approaches difficult, since the alternate usages of the term can `drown out' the signal from the intended target usage.
 To remedy this problem, we introduce a method to detect changes in the \emph{usage pattern} of a term over time.
 Although our method is not totally independent of volume, it does not depend wholly upon it, and we hypothesize that our model can detect events even when total tweet volumes or query hit volumes remain constant.
 
When a significant event happens, and the circumstances of the event are semantically related to a particular term, then the context of that term in current discussions on social media might reasonably be expected to change as the event unfolds.
 For example, when a shooting occurs, we expect that the term `shooting' might be used more frequently in the context of the name of the location of occurrence, the name of the assailant, condolences for the families, etc.
 This textual context is expected to be different from the context that occurs `normally' in day-to-day conversations.
 As such, if we compare the embeddings of the target terms from an embedding model trained on data taken during an event and the embeddings from an embedding model trained on `normal' data, we expect these embeddings to be distant in the embedding space.
 On the other hand, if we compare the embeddings of unrelated terms during and not-during the event, we expect these embeddings to remain close in the embedding space, since the unrelated terms continue to be used in their ordinary contexts.
 
To achieve a model that makes such a distinction, we propose to train a Word2Vec model \cite{mikolov2013efficient} on tweets localized in time and place.
 Since the quantity of text available from tweets during a particular day in a particular town may be insufficient to fully train the many parameters in an ordinary Word2Vec model, we first obtain a `base model' by training on a large corpus of tweets.
 We take this base model to define the ordinary usage of the various terms in the vocabulary.
 We then use this base model to define the starting parameters of many individual Word2Vec models trained on data that are localized in time and place.
 The word embeddings that result from these individual models can be compared with each other and also with the embeddings from the base model.
 We hypothesize that these comparisons can provide insight into a local geography's goings-on as they change over time, at least sufficiently so to detect very popularly-discussed one-time events.

Surprisingly, there seems to be a scarcity of public datasets that are suitable to test such a hypothesis.
 Since the majority of models involve volumetric analysis of a particular query, many public twitter datasets gathered for event detection reflect this intent by containing only tweets returned from the target query.
 This is problematic for our hypothesis, since compared to a model trained on approximately randomly-sampled tweets, \emph{all} terms in such a constrained dataset might be expected to be used in a different context.
 For example, a Word2Vec model trained on a dataset consisting only of tweets containing the term `shooting' would naturally adjust the embedding of the word `the' to reflect the fact that it seems to be used with incredible frequency in the context of the word `shooting'.
 With this in mind, we have gathered a dataset of geographically localized tweets surrounding important events in time, but without any additional query conditioning.
 This dataset allows us to test hypotheses concerning the time-dependent contextual usage of a term in a particular location, without the confounding factor of query restrictions.
 Our hypothesis that the time-series of word embeddings from Word2Vec models localized in time can be used to detect events is just such a hypothesis, and we have gathered our dataset specifically for this purpose.

In the proceeding sections, we describe the methodology of data collection and the specific details of our feature-extraction algorithm.
 We apply this algorithm to the aforementioned dataset, and compare the resulting embeddings with the embeddings of our base model using simple distance metrics.
 We then discuss the results of these experiments.

%% file: local_W2V_v_time/background/background.tex
\section{Background} \label{localw2vsection:background}

\subsection{Local Models} \label{localw2vsubsection:localmodels}

Local regression methods such as the Nadaraya-Watson estimator \cite{nadaraya1964estimating} are a popular class of nonparametric smoothing procedures.
 The local constant models employed in Nadaraya-Watson has been extended to more complex families such as local linear models (LOESS) \cite{cleveland1979robust}, and local polynomials \cite{fan1996local}.
 This recipe can be used to turn an arbitrary model family into a `local' model family, so long as it admits a weighted learning procedure \cite{vapnik1992principles}.
 The authors of \cite{bottou1992local}, for example, apply this same concept to obtain predictions from localized linear classifiers.
 
The general idea of local modeling is to fit a model to a locally-weighted neighborhood of the training data, and then to make a prediction with the obtained local model.
 Formally, for a given model family parametrized by $\theta$, and with a training scheme focusing on minimization of a loss function that is the sum of individual model errors:

\begin{equation} \label{localw2vloss}    
    L(X) = \sum_{x \in X} E(x,\theta)
\end{equation}
A localized training algorithm is easy to obtain as: 
 
 \begin{equation} \label{localw2vlocal_loss}
    L'(X,q) = \sum_{x \in X} E(x,\theta_q)K(q,x)
 \end{equation}

Word2Vec models are conveniently of this type.
 $K(q,x)$ gives a weight to each training data point $x$ depending upon the query point $q$, and is typically a function of the distance between $x$ and $q$.
 $K$ is usually called the ``kernel", and many choices of kernel are common in the literature \cite{terrell1992variable}.
 We will employ the Tricube kernel \cite{cleveland1979robust} with a fixed radius (the ``bandwidth'') in our experiments, because it is differentiable and also because it has finite support:

\begin{equation}\label{localw2vtricube_kernel}
   K(x,y) = \left(1-\left(\frac{||x-y||}{h}\right)^3\right)^3
\end{equation}

If the weighting scheme and loss function of the model family are in $C^m$, and the loss function is convex, then the parameters of the local models will generally be $C^m$ in the query point as well, resulting in well-behaved predictions across the space.
 Although this process can apply to many model families, simple models are the most popular candidates for localization, since convexity is a desirable property of the candidate model family, and also since the reduction in training data caused by the localized-weighting step often opposes the construction of complex local models.
 
Unfortunately, Word2Vec models lack these convenient properties of a low-dimensional parameter space and a convex optimization surface.
 Despite this fact, we are encouraged by similar applications for local models other than mere predictions, and using more complex models for localization which have been shown to be viable.
 In \cite{bay2004framework}, the authors use the model parameters of localized ARMA models for anomaly detection.
 In \cite{brown2019local}, the authors use the model parameters of localized Gaussian Process models for time series classification.

We employ a base-model strategy to mitigate the effects of these problems on our localized models as described in section \ref{localw2vsection:methodology}.
 Still, there is in general no theoretical guarantee that a Word2Vec model trained on `similar' data will have `similar' embeddings.
 Since our algorithm hinges on such an interpretation of the embeddings of a model having such a relationship with the data, this is a limitation of the proposed algorithm.
 Still, our experiments have shown that our mitigation strategies give good results on this front, as seen in section \ref{localw2vsection:results}.
 
\subsection{Word2Vec Models}

Word2Vec models \cite{mikolov2013efficient} are a particular type of word embedding model.
 In general, word embedding models take in a string, and output a real-valued vector.
 This is helpful for applying machine learning to text data since many off-the-shelf machine learning methods require real-valued vectors as input.
 Thus, word embeddings are commonly used as a natural language preprocessing step to transform `unstructured' text data into `structured' numerical data that facilitates further learning.
 A naive means of obtaining a word embedding is to simply use a `one-hot' encoding, which takes in a string and returns a vector containing a $0$ everywhere but a single index containing $1$, where each index is designated for a specific word in the vocabulary.
 This produces a very high dimensional, sparse vector, having no correlation between the embeddings of various words.
 Word embedding models attempt to improve upon these shortcomings by obtaining lower dimensional dense vectors where the embeddings of semantically related words are `close' in some sense.

Word2Vec models obtain such embeddings via a transfer learning strategy.
 The model is trained to predict a word from its context, or vice versa.
 This prediction task is only a surrogate, and embeddings are obtained by truncating the pipeline of mathematical operations in the model at a specific intermediate calculation.
 This intermediate calculation can be conveniently interpreted as word embeddings, and used as-is for other learning tasks.

We employ a specific Word2Vec architecture known as continuous-bag-of-words \cite{mikolov2013efficient} that takes in fixed-length sequences of text, minus the word in the middle of the sequence, and attempts to predict that middle word.
 Inputs are one-hot encoded, and outputs are a vector of the length of the vocabulary with pseudo-probabilistic values in each index, corresponding to the confidence of the model in each possible `middle' word.
 Individual errors are computed as the cross-entropy: 

\begin{equation} \label{localw2vcrossentropy} 
    E(\hat{\mathbf{y}}) = -\mathbf{y}\cdot \log(\hat{\mathbf{y}})
\end{equation}
 where $\mathbf{y}$ and $\hat{\mathbf{y}}$ are the one-hot vector for the `middle' word and the model output respectively.
 The full loss can be computed as the sum of these individual errors.
 The model formulation is log linear:

\begin{equation} \label{localw2vw2v}
    \hat{\mathbf{y}} = \mbox{logistic} (f_\theta(\sum_{w \in s} g_\phi(w)))
\end{equation}
where $g$ is a linear function, $f$ is linear with a bias term, and $s$ is the input sequence.
 The fact that the model simply sums over the words $w \in s$ and does not take into account the \emph{order} of the sequence inspires the moniker `continuous-bag-of-words'.
 After training this model, the embeddings are obtained by simply truncating the computations at $g_\phi(w)$.
 Since $w$ is a one-hot encoding, the linear model $g$ simply returns a row of its parameter matrix $\phi$.
 Thus, the rows of the parameter matrix $\phi$ contain a lower-dimensional real-valued vector corresponding to an individual word, which we interpret as embeddings of the various words in the vocabulary.
 We apply the localization scheme in equation \ref{localw2vlocal_loss} to the loss consisting of the sum of individual errors in equation \ref{localw2vcrossentropy} where our kernel $K$ operates over an additional time dimension (the timestamp of the tweet).
 This gives a time sequence of embedding matrices $\phi_q$ that we have hypothesized reflect the changing usage of terms over time.

\subsection{Twitter Event Detection}

Online event detection has roots in the dot-com boom in the late 90s, due to the proliferation of easily accessible and machine-readable electronic news feeds on the internet \cite{yang1998study}.
 As individual users are increasingly creating their own content, there has been much research in performing event detection via analysis of the time evolution of social media feeds.
 The hope is to detect developing events before traditional media and intelligence gathering are able. 
 The general hypothesis is that information from individuals who happen to be in the area will appear in social media feeds as the even unfolds.
 Since Twitter is a source of such information that is to some extent publicly available, Twitter feeds are a popular data source to test such hypotheses.
 So popular, in fact, as to have inspired a number of literature reviews dedicated exclusively to analysis of Twitter data \cite{atefeh2015survey,hasan2018survey,weiler2017survey}.

The authors in \cite{hasan2018survey} divide Twitter event detection methods into three types: Full-tweet clustering methods, Latent topic modelling, and term-based approaches.
 Our method falls squarely into the category of term-based approaches.
 A vast majority of the existing term-based approaches explicitly target changes in the volume of tweets corresponding to a particular term or set of terms.
 Very few approaches directly mine for changes in the relationships between terms in the corpus.
 
One example that does mine for term relationship changes is EnBlogue \cite{alvanaki2011enblogue}, which detects events my measuring changes in the co-occurrence of pairs of hashtags.
 This method requires some means of determining which pairs of terms to monitor, as the co-occurrence matrix is very large for even a moderately-sized vocabulary.
 Furthermore, this method is impractical for measuring changes that occur in groups of more than 2 terms, since the already large co-occurrence matrix grows exponentially with the number of terms in a group.
 Our proposed method avoids these challenges by summarizing the co-occurrence matrix entries of a single term into its corresponding embedding.
 Changes in the embedding of a term are expected to reflect changes in its co-occurrence with any other term or group of other terms.
 This efficiency is gained at the expense of measuring the exact terms for which the co-occurrence changes, as the embedding does not naively provide these contextual terms.
 However, upon discovering an event, these co-occurrences can be retroactively mined for a white-box interpretation if desired.

Yet other approaches employ term relationship mining in some form or another, but this is almost always done in the context of expanding the set of tweets used for volume-based analysis or clustering methods.
 \cite{brigadir2014adaptive} employ a Word2Vec model for exactly this purpose, wherein the embeddings of a fixed query are compared to the embeddings of a full tweet to determine whether or not to include a tweet in the timeline.
 Although they consider the possibility of an `adaptive' Word2Vec model, whose embeddings change over time, they do not directly analyze the changes in word embeddings over time.
 As part of their defense of utilizing time-dependent Word2Vec models, they note that ``retraining the model on new tweets create entirely new representations that reflect the most recent view of the world."
 This is precisely the intuition behind our proposed method.
 Their positive results for dynamically updating their Word2Vec models over time therefore lends credence to our hypothesis that the time-series of embeddings contains information about developing events.

%% file: local_W2V_v_time/methodology/methodology.tex
\section{Methodology} \label{localw2vsection:methodology}

\subsection{Data Collection} \label{localw2vsection:datacollection}

Data collection was performed during various short periods in 2019.
 The data collection process consisted of manually inspecting news sources for events, subjectively evaluated to select widely reported events with little forewarning.
 Although a few natural disasters were recorded in various locales, the most readily available source of such data points in English were terrorist attacks, especially mass shootings.
 Since the majority of our data points involve shooting events, we have limited our analysis to these data.
 This leaves us with 8 relevant events for which tweet data leading up to the event have been gathered.
 Still, we provide the full set of collected data for transparency and for use by other researchers.
 Ground truth labels for the time of each event were gathered from the Wikipedia page for the event, months after the data were gathered.

Data were collected using the free tier of the Twitter feed.
 We note that Twitter subsamples the data available in the free tier, and the exact means by which they perform this subsampling is unknown.
 The free tier allows queries up to 7 days in the past, a rate limit of 180 requests per 15 minutes, at 100 tweets per request.
 The rate limit implies a theoretical maximum 18000 tweets per 15 minutes, which we have found was sufficient to gather relatively comprehensive locality-restricted data.

Our queries are limited to tweets geotagged within a 10mi radius of the nearest population center to the event.
 For example, tweets for the Christchurch Mosque shootings are limited to those geotagged within 10mi of the `center' of Christchurch, NZ.
 Although tweets originating from mobile devices may be geotagged with the actual location of the device, many tweets are geotagged with the self-reported location of the user.
 We have made no distinction between these two possibilities, and it is possible that some tweets in the corpus originate outside of the query zone.
 No query terms are used, and the geographical location is the only restriction on the gathered data.
 
Given the above restrictions and motivations, when an event was identified via manual monitoring, we began collecting tweets, and stopped collecting when we had gathered tweets up to and for some time after the ground truth time of the event.
 Thus, our datasets generally consist of data beginning approximately 6.5 days before the time of the event, and run up through approximately 24 hours after the event.
 For most of our targeted events, the rate limits were sufficient to gather all of the tweets that Twitter provided during that time period and for that location.
 One exception is the Notre Dame Cathedral fire on April 15 2019.
 Gathering all tweets centered at Paris, France during this event overwhelmed our rate limits, and the data were unusable.
 Our other sequences are centered at smaller towns, and the data are well distributed in time.

\begin{figure}[!t]
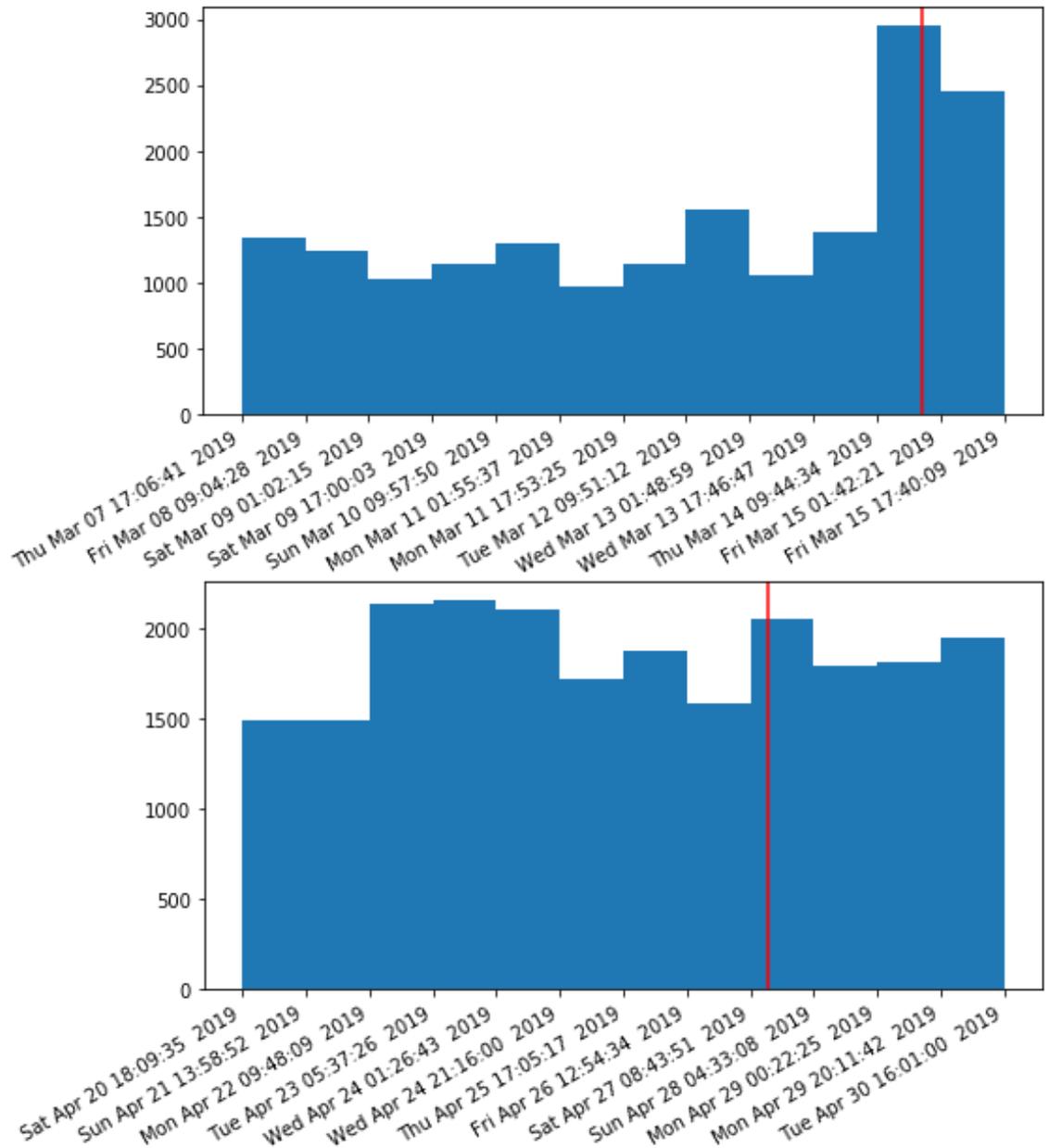

    \centering
    \includegraphics[width=\columnwidth]{local_W2V_v_time/static/{{shooting_volume}}}
    \caption[Twitter dataset volumes over time]{Twitter dataset volumes over time.
        The volume of tweets over time for the Christchurch NZ Mosque shooting (top) and Poway CA Synagogue shooting (bottom).
        The vertical red line is the approximate ground truth time of the event.
    } 
    \label{localw2vfigure:shooting_volume}
\end{figure}

Although our goal was to target situations where the overall volume of tweets was not a primary factor, we noticed that in some, but not all of the sequences, there was a volume increase near the time of the event.
 This can be seen in the differences in the volume over time histograms for the Christchurch, NZ shooting data and the Poway, CA shooting shown in Figure \ref{localw2vfigure:shooting_volume}.
 It is possible that this is a result of the Twitter subsampling with respect to past vs present data, since data gathering was started around the time of the event.
 This seems especially possible since manual inspection suggests that the magnitude of the increased volume is not entirely explained by tweets involving the event.
 In any case, this property of the data represents a limitation of our analysis, for which the increased volume is a confounding variable.

\subsection{Model Design}

In the following, we employ the continuous-bag-of-words Word2Vec model in equation \ref{localw2vw2v} trained using the cross-entropy loss in equation \ref{localw2vcrossentropy}.
 We localize the model with equation \ref{localw2vlocal_loss}, where a training point $x$ is taken to be a tweet and its associated meta-data.
 We employ the tri-cube kernel of equation \ref{localw2vtricube_kernel}, where the distance is taken between the time component of a data point $x$ and a query time $q$.
 We fix the bandwidth at 24 hours.
 Choosing such a wide bandwidth is helpful for smoothing out variations in volume between day and night.
 The result is an individual model at each point in time $q$, that is trained over a weighted time window of tweets occurring one day before or after time $q$.
 Note that the tri-cube kernel is symmetric, which may not be optimal due to the asymmetry of time with respect to the present.
 Note also that we have not experimented with changing the bandwidth.
 It is probable that a more highly tuned choice of kernel and bandwidth would be better, but such hyperparameter tuning is beyond the scope of this work.

In order to mitigate some of the potential problems pointed out in section \ref{localw2vsubsection:localmodels}, instead of initializing each of our models at random weights, we begin with a model pre-trained on a larger dataset.
 We specifically use a tweet corpus of over 5M tweets collected over 5 months made publicly available by the authors of \cite{cheng2010you}.
 We use these pre-trained weights as a `base model' from which to train each of our local models at various points in time.
 The goal here is to prevent the discovery of embeddings that are `equivalent', but not identical, as might be obtained by rotating the embedding vectors around an arbitrary axis, or by scaling all of the embeddings by a constant factor.
 This helps our individual local models find optimal parameters that are in `nearby' local minima, partially alleviating the problem with optimization surface convexity.
 This also helps alleviate the problem of low data to parameter ratio.
 Since the model has been pre-trained, it does not require data sufficient to fully learn term contexts.
 
Prior to training, we perform minimal preprocessing on the data, without `stop' words or stemming, and not removing hashtags, but removing most other special characters and removing usernames.
 We restrict the vocabulary size to the 6000 most common words in the baseline corpus.
 Words that are not members of this set are replaced with a single special `unknown' token in all datasets.
 These preprocessing steps are performed on both the baseline data and our collected data.

We further use the the base model as a word embedding baseline to which we can compare the embeddings found by our individual local models.
 By computing simple distances for specific terms in the embedding space, we measure how different the context of the term at a particular time and place is different from the `ordinary' usage of the term as defined by our baseline model.

For each of the shooting events described in section \ref{localw2vsection:datacollection}, we divide the timeframe of the data collection into 600 evenly spaced points.
 At each of these points, we initialize a Word2Vec model with the parameters of our baseline model, and train that model for a fixed number of epochs on the tweets from that time period, weighted appropriately.
 We then extract the embedding matrix of the model, and string these embedding matrices together to form a new time-series of local model parameters spanning the same timeframe as the original dataset.
 We specifically focus on the time evolution of individual rows of these matrices, corresponding to the embedding vectors of individual terms.

In order to reduce this time series to something useful for visualization and testing, we compute the distance from each of the embedding vectors in the time series and the embedding vectors of the baseline model.
 This gives us a metric of how far the embedding vector has changed after training.
 As in equation \ref{localw2vw2v}, if $g_\phi(w)$ is the baseline embedding for word $w$, and $g_{\phi_q}(w)$ is the embedding for word $w$ for the local model centered at query time $q$, then this metric is given by:
  $$ \mu_q(w) = || g_\phi(w) - g_{\phi_q}(w) ||$$
 
One possible implication of our hypothesis that is easy to test is that the embedding of a term that is semantically relevant to the event should change a great deal from baseline levels during the event compared to terms that are not semantically relevant.
 We test this prediction by comparison of the time evolution of $\mu_q(w)$ for the term `shooting' and several other randomly chosen words from the top 500 most common words in the baseline corpus.
 A further prediction of our hypothesis is that the embedding of a semantically relevant term should change more from baseline levels shortly after the event than prior to the event.
 We use the time evolution of $\mu_q(w)$ for the term `shooting' to evaluate this prediction as well.
 
A confounding factor for both of these is the volume of tweets involving the term `shooting'.
 Since our hypothesis involves identification of events independently of the total volume of tweets, or the volume of tweets including particular relevant terms, then it is important to provide metrics of this confounding factor.
 As such, for each of the events for which data were gathered, we evaluate the volume of tweets and the volume of tweets involving the term `shooting' as well.

%% file: local_W2V_v_time/results/results.tex
\section{Results} \label{localw2vsection:results}

In this section we report the results of the experiments of our local Word2Vec model feature extraction algorithm as applied to the gathered corpus of Twitter data surrounding various mass shooting events.
 Our analyses consist essentially of comparing the time series of tweet volumes, of volume of tweets involving the specific term `shooting', and of the change in the embedding vector for the term `shooting' over time.
 We recall here our hypothesis that the changes in the embedding vector for terms semantically related to the observed event should be indicative somehow of that event.
 We also recall our specific predictions that the embedding vector for the term `shooting' should change a great deal from baseline levels during the event compared to other terms, and that the embedding should change more from baseline levels shortly after the event than prior to the event.
 We had anticipated specifically that such changes might be observed even when the volume of tweets and the volume of term usage was not significantly changed.

For the Virginia Beach shooting dataset, the volumes of the term `shooting' vs various random common terms is shown in Figure \ref{localw2vfigure:term_volumes}.
 Note that prior to the event, the term `shooting' is nearly entirely absent, and after the event the term `shooting' becomes incredibly frequent.
 So frequent, in fact, as to be as commonly used as the term `hey'.
 This is the case in all of the data that we have gathered, for all of the events.
 In short, the data is insufficient to test our hypothesis with respect to situations where term usage might change independently of volume.

\begin{figure}[!t]
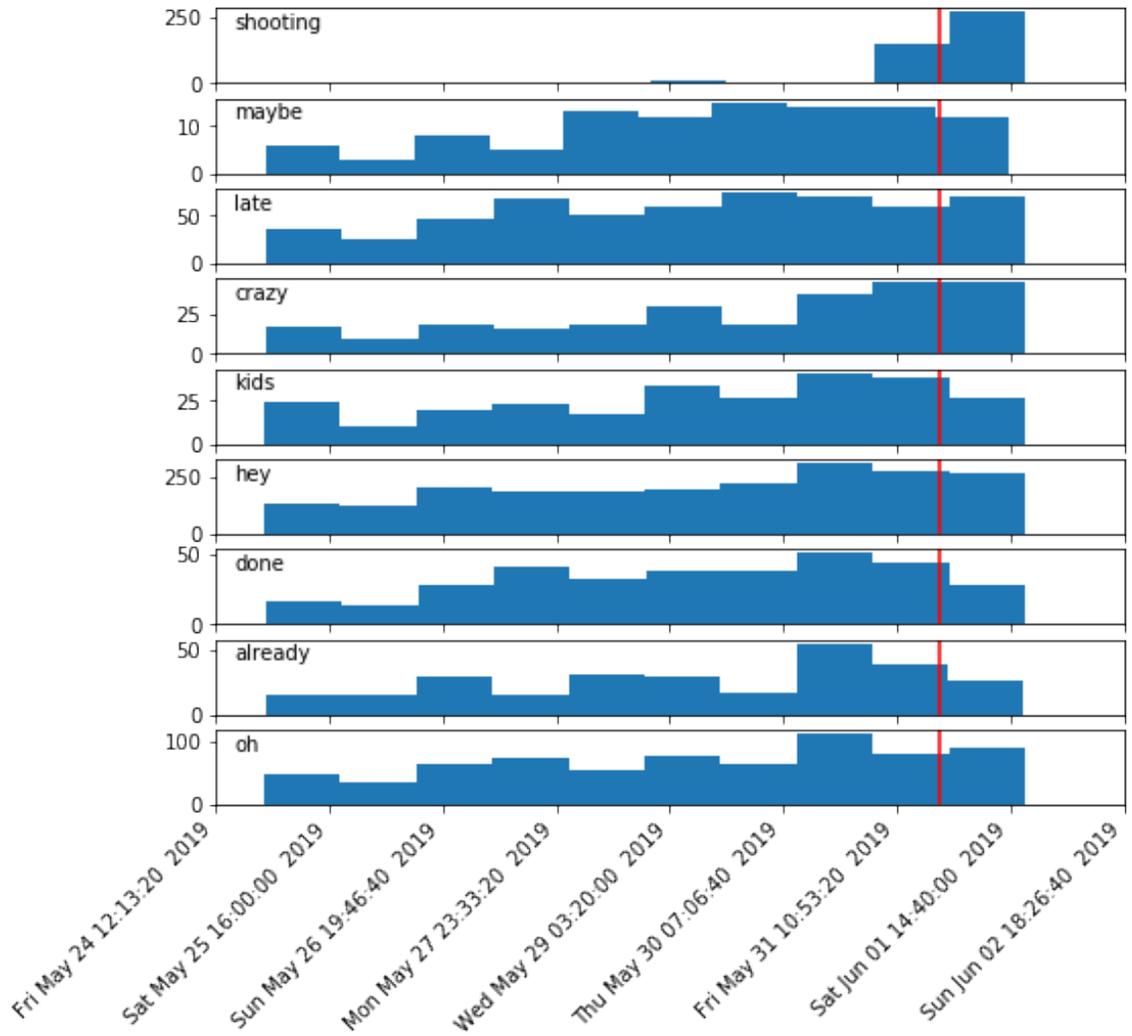

    \centering
    \includegraphics[width=\columnwidth]{local_W2V_v_time/static/{{term_volumes}}}
    \caption[Twitter dataset term volume over time]{Twitter dataset term volume over time.
        The volume of various terms in tweets for the Virginia Beach shooting dataset.
        The vertical red line is the approximate ground truth time of the event.
    } 
    \label{localw2vfigure:term_volumes}
\end{figure}

\begin{figure}[!t]
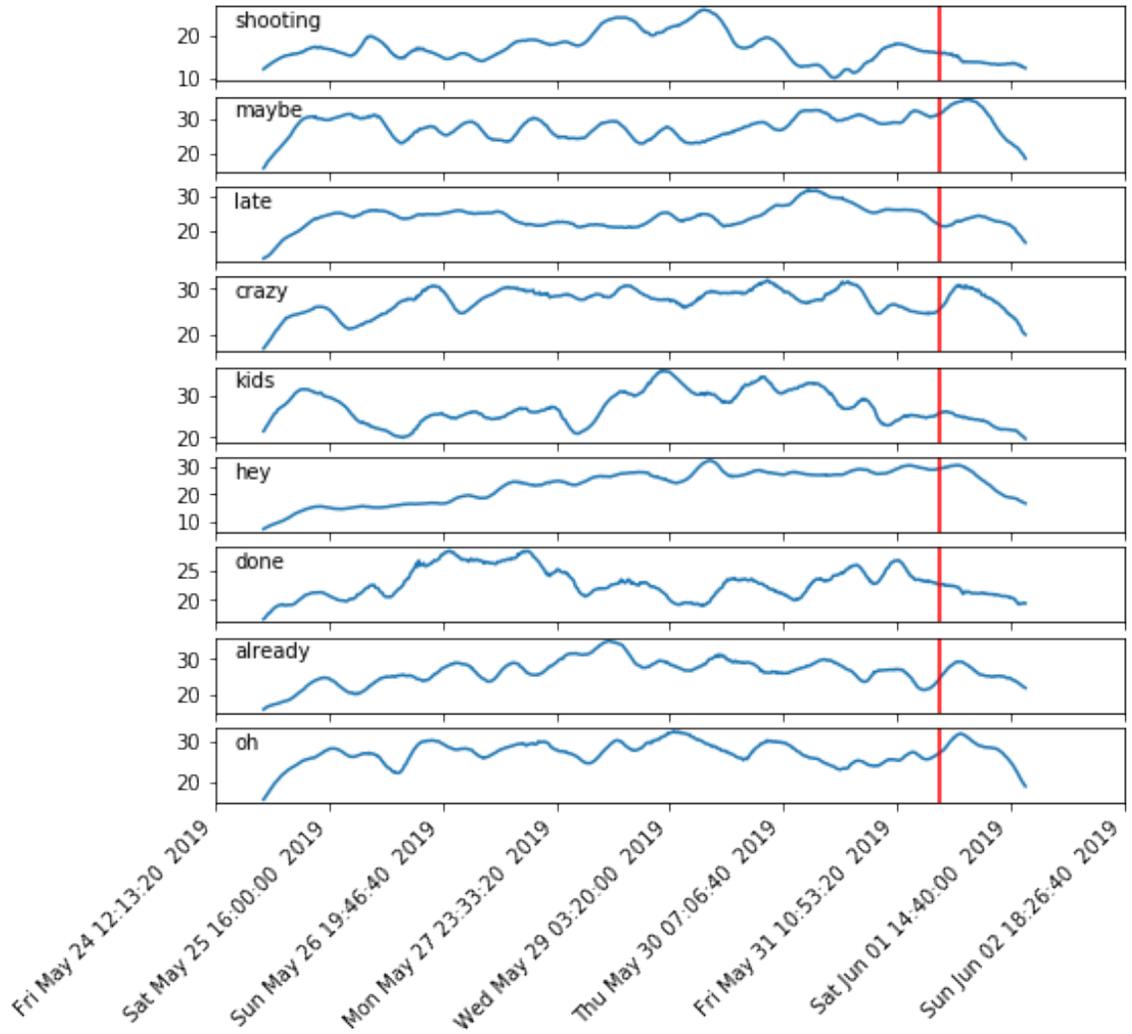

    \centering
    \includegraphics[width=\columnwidth]{local_W2V_v_time/static/{{term_embedding_distance}}}
    \caption[Term embedding changes over time]{Term embedding changes over time.
        The squared Euclidean distance of the embeddings for individual local models in the Virginia Beach shooting dataset to the embeddings in the baseline model.
        The vertical red line is the approximate ground truth time of the event.
    } 
    \label{localw2vfigure:term_embedding_distance}
\end{figure}

Still, we continued with the remainder of our analysis.
 Figure \ref{localw2vfigure:term_embedding_distance} shows the square of the Euclidean distance of the embeddings for each term to the embeddings for that term in the baseline model.
 Note that for all of the terms, the magnitude of the distance is comparatively less near the ends of the time series.
 This is almost certainly because the windows for those periods have fewer tweets in aggregate (since there is no data past the beginning and end of the dataset).
 This demonstrates that the proposed method is not entirely independent of raw tweet volume, and it is possible that adjustments might need to be made in order to compensate for this confounding factor.
 Since we have trained our local models for a fixed number of epochs, there are simply more updates to the model parameters when there is more data fed into the training algorithm.
 Possibly downsampling high-volume periods and/or upsampling low-volume periods to maintain a consistent volume of training data would help alleviate this problem with the proposed method.

In any case, since the term `shooting' occurred with such a high volume after the event, and such a low volume before the event, we would anticipate that this confounding factor would result in our predictions being \emph{more likely} to be true in our analysis.
 However, the opposite has occurred, as can be seen in the time series of changes to the embedding vector for the term `shooting', and comparing it to the other terms in Figure \ref{localw2vfigure:term_embedding_distance}.
 It seems that the volume of the term `shooting' has had little or no effect on the changes in the embedding of the term `shooting'.
 The variability in the series for the term `shooting' in the time leading up to the event is somewhat surprising, since the term was nearly entirely absent from the gathered data during this period.
 This is because updates to the embedding of the term `shooting' do not depend only on the context of the term `shooting', but also on the contexts of other terms that generally tend to appear in the same context as the term `shooting'
 This is generally touted as a feature of Word2Vec models, and word embeddings in general.
 Still, we had anticipated that the primary contributor to changes in the embedding of a particular term would be the term itself.
 Furthermore, we had anticipated that, if a particular term was semantically relevant to the targeted event, then terms with similar embeddings would be as well, so that this feature would work out in favor of the proposed algorithm.
 However, this has not turned out to be the case.

Still, the fact that the changes to the word embeddings over time looks to be continuous and relatively smooth shows that our strategies for mitigating the negative effects of the non-convex optimization surface of the individual models was successful.
 Although our particular hypotheses concerning the \emph{interpretation} of changes to the embedding vectors over time proved incorrect, this is promising that perhaps the analysis of local word embedding models over time may prove useful for some other task.
 Furthermore, the fact that these time series do not correlate perfectly with tweet volume or term volume suggests that there is some interesting underlying process, and that the general intuition may not be entirely ill-founded.

%% file: local_W2V_v_time/conclusion/conclusion.tex
\section{Conclusion} \label{localw2vsection:conclusion}

In this paper, we have investigated a novel method of extracting features from streams of text data over time.
 In order to determine possible use-cases of our algorithm, we have employed it in an event detection pipeline for the purpose of detecting important events from time series of unstructured text.
 To this end, we have gathered a new dataset of time series text data from the Twitter platform that specifically targets unforeseen events with a high geographically local impact.
 We have hypothesized that our extracted features would indicate directly the occurrence of the targeted events.
 Although this hypothesis did not ultimately prove true, it is clear that our method for extracting features is related to other relevant features such as tweet volume, but is not entirely identical to them.
 Furthermore, we have introduced methods for overcoming some principal obstacles to extraction of features via local word embedding models, paving the way for future research into their use for text stream analysis.
 Future work includes iterations upon the basic method introduced in this work to improve its applicability to real world use cases.
 Also, this lends some evidence to the feasibility of utilizing complex local models as a feature extraction technique, and therefore provides justification for further investigations into localization of various model families with interpretable parameters.

%% file: local_ODQ_regression/abstract.tex
The local modeling paradigm can be used as a basis for a broad class of non-parametric models, and also as a means for extracting local features from data.
 Many types of model-learning algorithms trivially extend to locally-learned varieties by making predictions using a locally-weighted model at each point.
 Models that do not make a distinction between independent and dependent variables cannot translate into local versions using this process, because there is no a priori concept of a ``prediction" in such contexts.
 Some work has been done recently on localizing linear models in situations where one variable is not a function of the other by leveraging orthogonal-distance regression.
 In this paper, we extend this work to localize orthogonal-distance quadric regression models.
 We perform experiments with local quadric models in 3D, showing that this class of models can be used for surface reconstruction and also for feature extraction with applications to point-cloud segmentation.

%% file: local_ODQ_regression/intro/intro.tex
\section{Introduction} \label{localodqsection:introduction}

Learning local representations of data has been a popular means of non-parameterized smoothing \cite{atkeson1997locally} starting in the 1960s with the Nadaraya-Watson estimator \cite{nadaraya1964estimating}.
 Extensions to more complex local representations for smoothing followed with LOESS \cite{cleveland1979robust} and local polynomial models \cite{fan1996local}.
 More recently, learning local representations of data has transcended simple smoothing operations, and is an important aspect of learning with large datasets.
 Specifically, local learning is an essential ingredient of state-of-the-art learning algorithms over datasets that are comprised of data that might in another context be considered datasets in themselves.
 Examples include: 
    image processing in which convolutional neural networks rely crucially on local learning \cite{krizhevsky2012imagenet};
    time-series classification, in which the local shape of a sequence is often a key feature \cite{ye2009time};
    time-series subsequence analysis, and especially anomaly detection, where local behavior is a key indicator \cite{bay2004framework};
    and many others.
 As such, methods for extracting features from data that represent local structure are important in a variety of contexts.
 Recent research has shown that one means of extracting locally-relevant features is by local modeling \cite{bay2004framework, brown2019local}.

Local modeling in certain contexts is less straightforward than in others.
 Many local modeling algorithms, such as LOESS, require that one variable might reasonably be considered as a function of the others.
 This is a simple consequence of the fact that many global models make this dependent vs independent variables assumption, and local models are typically built from the same clay.
 In circumstances where such an assumption is not reasonable, such as with 3D point-clouds, local methods often require complex preprocessing steps to attempt to align them with the local surface orientation.

Recently, researchers have developed a class of local regression models that works well for point-cloud-type data without any such preprocessing by iterating locally-linear orthogonal distance regression \cite{ozertem2011locally}.
 These are relatively well-studied, and can be used for surface reconstruction \cite{ozertem2011locally}.
 
Unfortunately, surface reconstructions using locally-linear orthogonal distance regression models have some undesirable properties.
 Firstly, these surfaces are known to be biased toward the ``inside" of a curve \cite{hastie1989principal}.
 Secondly, our own experiments suggest that these surfaces are over-eager to extrapolate nearly-linear arrangements of data from the desired surface out to infinity.
 We hypothesize that employing non-linear models as the building block for surface reconstruction will help to alleviate both of these problems.

Furthermore, although local linear models are often used to estimate surface normals, and surface normals are an important local property of a surface, they are the \emph{only} feature given by locally-linear models.
 We hypothesize that local non-linear models will provide additional features descriptive of the local shape of the surface sampled by a given point cloud.

Unfortunately, to our knowledge, no work has been done to investigate the properties of subspace-constrained mean-shift like algorithms using non-linear models.
 We therefore in this work perform initial investigations into the properties of non-linear local models on point-clouds in 3D.
 Linear models in 3D can be described using implicit polynomials of degree 1.
 This suggests that the most basic class of non-linear models in 3D are implicit polynomials of degree 2, which is the class of quadric surfaces.
 In this work, we develop a means for localizing the training of quadric models on point clouds.
 We also investigate the convergence properties of a subspace-constrained mean-shift type algorithm based on locally-trained quadrics.
 Lastly, we test the hypotheses outlined above against point clouds in the Stanford 3D scanning repository \cite{turk1994zippered} and also the Shapenet pointcloud segmentation dataset \cite{yi2016scalable}.

%% file: local_ODQ_regression/background/background.tex
\section{Background} \label{localodqsection:background}

\subsection{Local Models} \label{localodqsubsection:localmodels}

Local models are a class of nonparametric methods in which models are built using the data that is `nearby' \cite{atkeson1997locally}.
 Even very simple model families, such as a linear model, when applied locally, can create complex and highly non-linear curves that often match our intuition of how a `smooth' version of data might appear.
 This hinges on basic results in Calculus which reveal that differentiable surfaces are `locally linear' in the sense that a hyperplane approximates such a surface very well at small distances.
 How small depends on the specific qualities of the surface, suggesting that a satisfactory definition of `nearby' is of utmost importance.

Typically, `nearby' is defined by a distance metric $d$, and data are weighted according to some kernel function $K$ which approaches or is identical to $0$ as the distance between two points becomes large.
 One such kernel frequently used in local regression is the tricube kernel \cite{cleveland1979robust}, which has finite support, a convenient property for computational efficiency:

\begin{equation}
\label{localodqequation:tricubekernel}
 K_h(q,x) = \begin{cases} (1-\frac{||x-q||^3}{h^3})^3 & ||x-q|| < h \\
                            0 & \text{else} \end{cases}
\end{equation}
where $h$ is the `bandwidth' and controls the radius of our notion of `nearby'.
The Gaussian kernel is useful for theoretical treatments, since it is infinitely differentiable and has algebraically convenient derivatives.
 Since this work is largely empirical in nature, this property is less concerning.
 When the distribution of points in space is skewed, it can be convenient to use a variable bandwidth.
 One kernel that achieves this in a relatively simple manner is the $k$-nearest-neighbor (KNN) kernel:

\begin{equation} \label{localodqequation:knnkernel}
 K(q,x) = \begin{cases} 
      1 & ||x - q|| \leq ||q - x_k|| \\
      0 & \text{else}
   \end{cases}
\end{equation}
where $x_k$ is the $k$th-nearest-neighbor in the training set to the query point $q$.
It is possible to compose and multiply the various kernels to obtain new kernels.
 We combine the tricube and the KNN kernel in our experiments to obtain a tricube kernel whose bandwidth varies as the $k$th-nearest-neighbor to a query point.
 This is convenient when query points are far from the data, and the ordinary tricube kernel can give 0 weights to all of the data if none are within radius $h$ of the query.

Given a suitable kernel and a dataset $X$ we can define a training procedure for the localized model centered at point $q$ so long as the model family under consideration admits sample weights.
 When the family is trained by minimizing the sum of individual errors over the training data, this can be achieved straightforwardly by simply multiplying the weights into objective function:

\begin{equation}
\label{localodqequation:localmodel}
 \theta_q = \mbox{argmin}_\theta \sum_{x\in X} E(x, \theta) K(q,x)
\end{equation}

Many commonly employed local modeling techniques are of this form.
 This includes locally estimated scatterplot smoothing (LOESS) \cite{cleveland1979robust} and the related Nadaraya-Watson estimator \cite{nadaraya1964estimating}.
 LOESS proceeds by fitting a local ordinary least squares line at each point in the space of independent variables, using the localization procedure in Eq. \ref{localodqequation:localmodel}.
 Once this regression line is learned, it can be used to predict the dependent variable at that point.
 The resulting combinations of independent variables and predicted dependent variables gives a nonparametric surface that can be used to `smooth' noisy data.
 Notice that a key ingredient of the algorithm is that the dependent variable is implicitly assumed to be a function of the dependent variables.
 Thus, as is, LOESS is not applicable to modeling, for example, points on a sphere.
 
Fitting functions by minimization of the sum of the orthogonal distances to the training data is also conveniently fits the format of equation \ref{localodqequation:localmodel}.
 Since orthogonal distance regression does not require a distinction between dependent and independent variables, it is ideally suited for modeling sets of points that do not possess this property, such as points on a sphere.
 Subspace-constrained Mean Shift is an example of a local modeling scheme that exploits just this fact to fit an $m$-dimensional surface embedded in an $n$-dimensional space.
 This can be applied, e.g. for surface reconstruction of noisy point clouds in 3D \cite{ozertem2011locally}.
 Subspace-constrained Mean Shift consists of repeatedly fitting local linear models with orthogonal distance errors.
 Query points $q$ are iteratively projected onto the local model trained at $q$, and the stationary points of this procedure are sought.
 Little research has been done to investigate the possibility of utilizing non-linear models in such a scheme.
 This is the subject of the current inquiry.

\subsection{Local Quadric Models in 3D} \label{localodqsubsection:local_quadrics}
 
Between the applications of computer-aided design \cite{azernikov2004efficient}, 3D computer graphics \cite{chen2013scalable}, 3D medical applications \cite{yoo2011three} and others, the task of modeling 3D surfaces from point sets is an incredibly well-studied field.
 Local models are employed in many point cloud smoothing and interpolation algorithms \cite{franke1982scattered} in some form or another.
 Indeed, even simple triangular meshes can be seen as fitting local planar models to neighboring sets of exactly 3 points in the mesh graph.
 Often, since the piecewise-planar representation exemplified by a triangular mesh results in the canonical `blocky' appearance of 90s-era computer graphics, `smoother' ways of stitching together point sets are sought.

Since planes are 1st-degree polynomials, the obvious step toward a `smoother' point set mesh involves simply employing 2nd-degree polynomials, i.e. quadrics.
 Piecewise quadric representations often give this desired property of smoothness, and have been the subject of much research \cite{petitjean2002survey}.
 However, piecewise models must come together at some boundary points, and the constraint that two quadrics meet at a differentiable boundary is very restrictive.
 Thus, piecewise quadrics cannot give fully `smooth' surface reconstructions for even modestly complex classes of surfaces.
 Furthermore, over any given region, a piecewise model will approximate the surface better at some points than at others.
 Such a model may therefore be suboptimal when attempting to approximate, e.g., the curvature of the surface giving rise to a point set.

When employing a smooth kernel and the local modeling paradigm described in section \ref{localodqsubsection:localmodels}, both local linear and local quadric models can be stitched together to form a smooth surface.
 Such methods are generally more computationally demanding than piecewise models, and so are not generally used for real-time rendering applications, but are still useful for offline purposes. 
 One of the benefits of utilizing local quadrics is that a quadric model has a non-trivial second derivative, which can be used to estimate the curvature of a surface.
 In \cite{douros2002three}, the authors fit local quadric models to a point cloud for exactly this purpose of estimating the curvature of a surface at a point.
 In order to find the optimal local surface, they employ the algebraic loss between a point and a particular quadric.
 The algebraic distance, although mathematically convenient, does not generally conform to our ordinary notions of `distance' in 3D.
 Indeed, the authors themselves note that ``It would be ideal to locate the point on $F$ that is exactly the closest to $P$'', where by $F$ and $P$ they mean the quadric surface and some point, respectively.
 They point out that it is computationally expensive to compute the orthogonal distance of an arbitrary point onto the quadric surface.
 Although this is true, we will show that for certain classes of quadrics it is not prohibitively expensive. 

In \cite{levin2004mesh}, the authors employ local quadric models to create a smooth surface reconstruction.
 They employ a local orthogonal-distance linear model to first determine a `direction' for the surface, and then employ an ordinary-least squares loss in that direction for the subsequent quadric model.
 This has the benefit of constraining the orientation of the local quadric, but relies on the local linear model to correctly determine the orientation of the surface.
 However, there is no a priori reason to suppose that the local linear models produce the correct surface orientation at a point.
 Indeed, it is known that local linear models can have issues when the underlying surface contains `sharp' corners \cite{xie2003piecewise}.
 Although the quadric model family does not include a large set of members with `sharp' features, we anticipate that the increased complexity of quadric models will help alleviate some of these problems all the same.

%% file: local_ODQ_regression/methodology/methodology.tex
\section{Methodology} \label{localodqsection:methodology}

We propose a method of local orthogonal quadric regression.
 We further propose a simple iterative extension of that method for surface reconstruction, similar to subspace-constrained mean-shift (SCMS) \cite{ghassabeh2013some}, involving iteratively projecting a point onto a model fitted locally at that point.
 Lastly, we propose a means of feature extraction on 3D point clouds involving the parameters of these locally-trained models, based on the method in \cite{bay2004framework}.

Although we will not derive any theoretical convergence results of the iterative process, we note that for degenerate quadrics, this process reduces to Mean Shift (a point) and SCMS (a plane).
 These algorithms have been very thoroughly empirically tested, and shown to converge \cite{ghassabeh2013some} under some very limited conditions.
 We will instead investigate the empirical properties of the iterative algorithm and the feature extraction algorithm for $D = {\rm I\!R}^3$.
 Specifically, we will provide exploratory investigations on the ability of local quadric models to create surface reconstructions of 3D point clouds from the Stanford 3D scanning repository \cite{turk1994zippered}.
 We also provide exploratory results for the use of these local quadrics as a method of feature extraction for 3D point cloud segmentation on the Shapenet dataset \cite{yi2016scalable}.
 Although the method proposed here applies generally to local orthogonal quadric regression, in our experiments we specifically constrain the problem to parabolic surfaces.
 The reasons for this will be explained later.

\subsection{Algorithms} \label{localodqsection:algorithm}

We will use the naive definition for fitting a quadric using a least squares loss on the orthogonal distances.
 A quadric surface $Q$ embedded in ${\rm I\!R^3}$ is typically given implicitly as:

\begin{equation} \label{localodqequation:quadric}
 \mathbf{x}^T Q \mathbf{x} = \begin{bmatrix}x & y & z & 1\end{bmatrix} \begin{bmatrix}
    a & b & c & d \\
    b & e & f & g \\
    c & f & h & i \\
    d & g & i & j \end{bmatrix}\begin{bmatrix} x \\ y \\ z \\ 1 \end{bmatrix} = 0
\end{equation}

This equation defines $Q$ uniquely up to a scalar multiple.
 To obtain a unique $Q$, it is customary to constrain $Q$ in some way.
 We do this by restricting $Q$ to have unit matrix norm. 

If $X \subset D$ is our dataset, $Q$ is some quadric surface embedded in $D$, and $proj(x,Q)$ is the orthogonal projection of $x$ onto $Q$, then the least-squares loss is given by:

\begin{equation} \label{localodqoqr_loss}
  L(X) = \sum_{x \in X} ||x - proj(x,Q)||^2
\end{equation}

We further employ the concept of locally-weighted learning \cite{atkeson1997locally}.
 As described in section \ref{localodqsubsection:localmodels} this involves weighting data that are `near' a query point $q$ to obtain a local model at the point $q$ using a kernel $K$.
 We specifically employ a tricube kernel (equation \ref{localodqequation:tricubekernel}) and use a variable bandwidth based on the $k$th-nearest neighbor (equation \ref{localodqequation:knnkernel}).
 We use a constant $k=120$ to determine the bandwidth of the kernel at a given query point $q$.

Given our kernel, we can `localize' our learning procedure by simply multiplying the weights into the loss function:
\begin{equation} \label{localodqlocal_oqr_loss}
  L(X,q) = \sum_{x \in X} ||x - proj(x,Q)||^2K(q,x)
\end{equation}
Where $q$ is the query point at which our model will be centered.
 We will be interested in the optimal such $Q$:
\begin{equation} \label{localodqlocal_oqr}
  Q_\text{opt}(q) = \text{argmin}_{Q} L(X,q)
\end{equation}

By minimizing our loss over possible $Q$ at every point $q$ in the original space, we thus obtain a mapping from our original domain $D$ to the set of quadric surfaces embedded in $D$.
 Obtaining a projection of $q$ onto the optimal $Q$ then gives a mapping of $D$ back to itself:

\begin{equation} \label{localodqiterated_local_oqr}
 p_{i+1} \leftarrow proj(p_i,Q_\text{opt}(p_i))
\end{equation} 

This process can be iterated, and the path of the $p_i$ traced through space.
 Such an iterative process is very similar to the mean-shift and SCMS algorithms, but extends them to non-linear models.
 To see this, note that the mean-shift algorithm \cite{comaniciu2002mean} iteratively projects points in $n$-D onto a local 0-d mean, which can be considered a ``single point" model.
 SCMS iteratively projects $n$-D points onto a $(m<n)$-D subspace defined by a local total least squares regression, a linear model.
 The iterative application of local quadric models is a natural extension of this process to simple non-linear models.
 In fact, SCMS and mean-shift are special cases of the local quadric algorithm when the quadratic terms are constrained to be 0.
 We will investigate this iterative version of the local quadric model algorithm against data in the Stanford 3D scanning repository \cite{turk1994zippered}.

Furthermore, this process yields, at each point in the domain, a quadric model approximating the shape of the data near the query point $q$.
 The local quadric model approximates the local shape of the data, and the parameters of the local quadric model encode the shape of the local quadric in a fixed-length vector.
 This yields a mapping from the original space into the space of quadric model parameters, which can be used as a feature extraction method describing a point's relationship with its neighbors.
 We will investigate this feature extraction process against labeled segmentation data \cite{yi2016scalable}.

In order to compute $proj(x,Q)$, we employ the method used in \cite{lott2014direct}, which reduces the problem to root-finding of a 6th degree polynomial.
 Unfortunately, the roots of 6th-degree polynomials have no closed-form solution, but the solutions may still be obtained by, e.g. finding the eigenvalues of the companion matrix.
 Interestingly, if we constrain our quadric to be parabolic, then the formula given in \cite{lott2014direct} reduces to root-finding for a 5th degree polynomial.
 Quintic polynomials also do not possess a closed-form.
 Thus, iterative methods to find the orthogonal projection of a point onto a paraboloid are still necessary.
 
Furthermore, the minimization of the sum of the individual errors has no closed form, and the objective function is generally non-convex.
 Thus, optimization is non-trivial, and convergence to poor solutions can be a problem.
 Lastly, constraining our quadric to be parabolic is not trivial when we update the parameters of a quadric in the form given in equation \ref{localodqequation:quadric} during iterative optimization.
 Therefore, to achieve a parabolic constraint, after each iteration, we project the updated quadric to a `nearby' paraboloid.
 We first rotate the quadric to remove the $xy$, $yz$ and $xz$ terms in equation \ref{localodqequation:quadric}. 
 We then set the smallest quadratic term to 0 and apply the inverse rotation to regain the original orientation.
 We utilize the same rotation operation in the sequence of calculations for finding the orthogonal projection from \cite{lott2014direct}, which involves the eigendecomposition of the upper-left $3\times 3$ block of $Q$.
 Note that eigendecomposition of a $3\times 3$ matrix involves solving for the roots of a cubic polynomial, which does have a closed form.

Since the orthogonal projection, and therefore the model loss, depend on finding the eigenvalues of the upper left block of $Q$, it is worth noting that in general the eigenvalues of a matrix are differentiable with respect to the entries at points where the eigenvalues are distinct.
 This same observation applies to the root-finding calculations, regardless of the method employed, in that the real roots of a polynomial are differentiable with respect to the polynomial coefficients so long as the roots are not repeated.
 Although these observations are related to the problem of the non-convexity of the optimization surface, they are unlikely to occur during optimization, and we have not experienced any problems using gradient-based optimization methods.
 Indeed, as long as these two situations do not occur, the \emph{derivatives} of all of the computations in \cite{lott2014direct} all have closed form, even without our parabolic constraint.
 Specifically, the partial derivatives of the roots of a polynomial $P(x)$ with respect to the coefficients can be easily computed by simply differentiating the formula $P(x) = 0$ implicitly, and solving for the relevant term.
 We do not explicitly derive any of these derivatives, and instead rely on the tensorflow library to compute and propagate the derivatives appropriately.

%% file: local_ODQ_regression/results/results.tex
\section{Results} \label{localodqsection:results}

\subsection{Stanford 3D Scanning Repository}

We reiterate our hypothesis here that a mean-shift-like iterated local quadric modeling procedure will improve upon the bias problems exhibited by surface reconstructions from iterated local linear modeling (subspace-constrained mean shift).
 To test this, we apply both the iterated local quadric models and subspace-constrained mean shift to the classic Stanford Bunny.
 We note that the the creator of this dataset, Marc Levoy, on the page offering the data for download, insists that these data have been smoothed to a degree that treatment as an unorganized point cloud is perilous. 
 We therefore add noise by first estimating the normal vector at each data point using the full dataset, and then adding a Gaussian random value in the direction of the normal vector.
 We then randomly subsample the result to obtain a less regular distribution of points on the surface.
 The result is shown in Figure \ref{localodqfigure:subsampled_noisy_bunny}.
 Still, the result is not necessarily representative of errors arising naturally from 3D scanning.
 
\begin{figure}[!t]
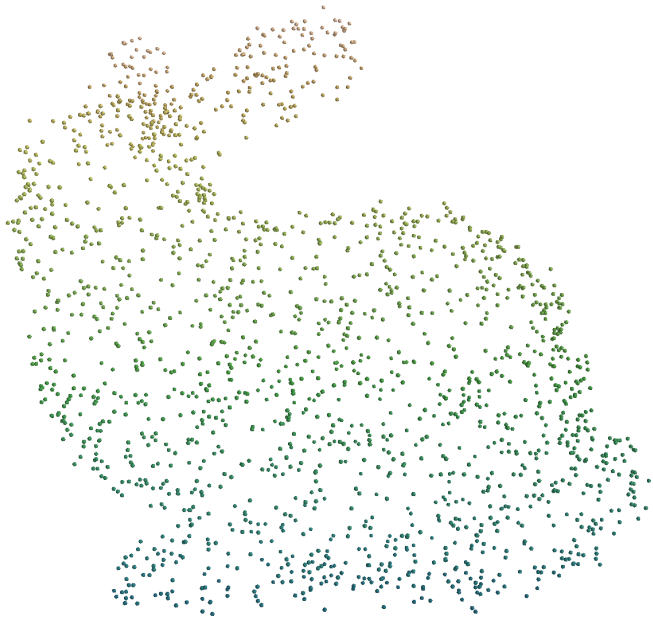

    \centering
    \includegraphics[width=\columnwidth]{local_ODQ_regression/static/{{subsampled_noisy_bunny}}}
    \caption{The subsampled Stanford Bunny with added Gaussian noise.}
    \label{localodqfigure:subsampled_noisy_bunny}
\end{figure}

\begin{figure}[!t]
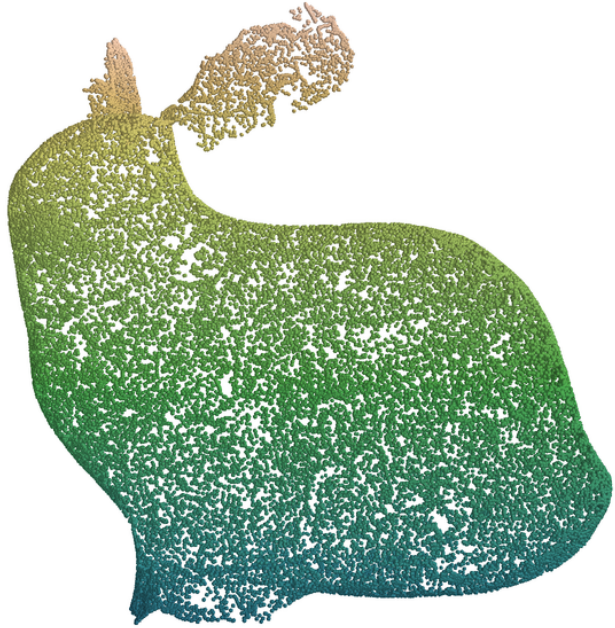

    \centering
    \includegraphics[width=\columnwidth]{local_ODQ_regression/static/{{scms_bunny}}}
    \caption[Stanford Bunny Subspace-constrained mean shift]{Stanford Bunny Subspace-constrained mean shift.
        The results of subspace-constrained mean shift trained against the noisy subsampled bunny in Figure \ref{localodqfigure:subsampled_noisy_bunny}.}
    \label{localodqfigure:scms_bunny}
\end{figure}

\begin{figure}[!t]
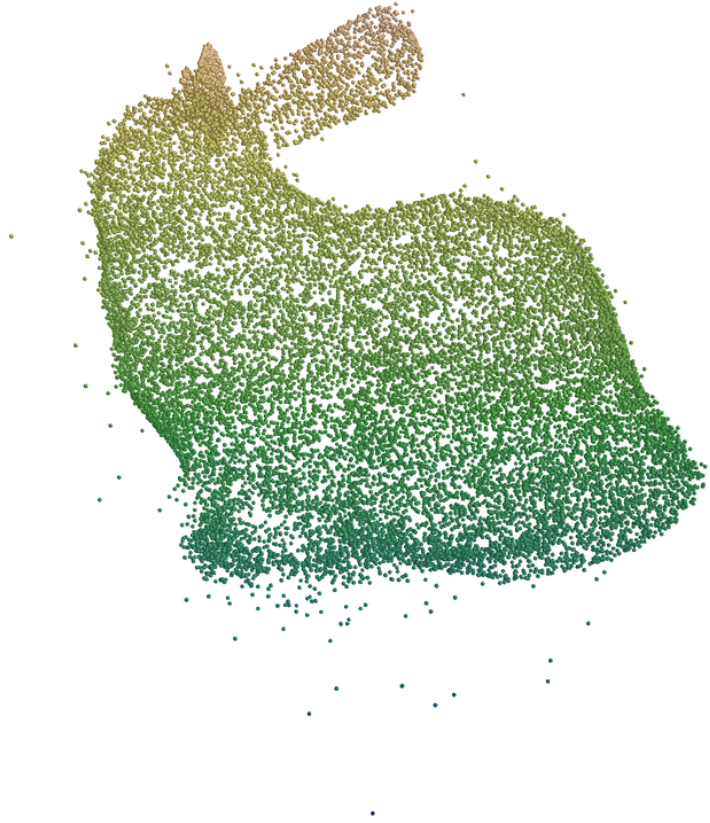

    \centering
    \includegraphics[width=\columnwidth]{local_ODQ_regression/static/{{quadric_iterated_bunny}}}
    \caption[Stanford Bunny local iterated quadric projections]{Stanford Bunny local iterated quadric projections.
        The results of the proposed iterated local quadric regression algorithm trained against the noisy subsampled bunny in Figure \ref{localodqfigure:subsampled_noisy_bunny}.}
    \label{localodqfigure:quadric_iterated_bunny}
\end{figure}

\begin{figure}[!t]
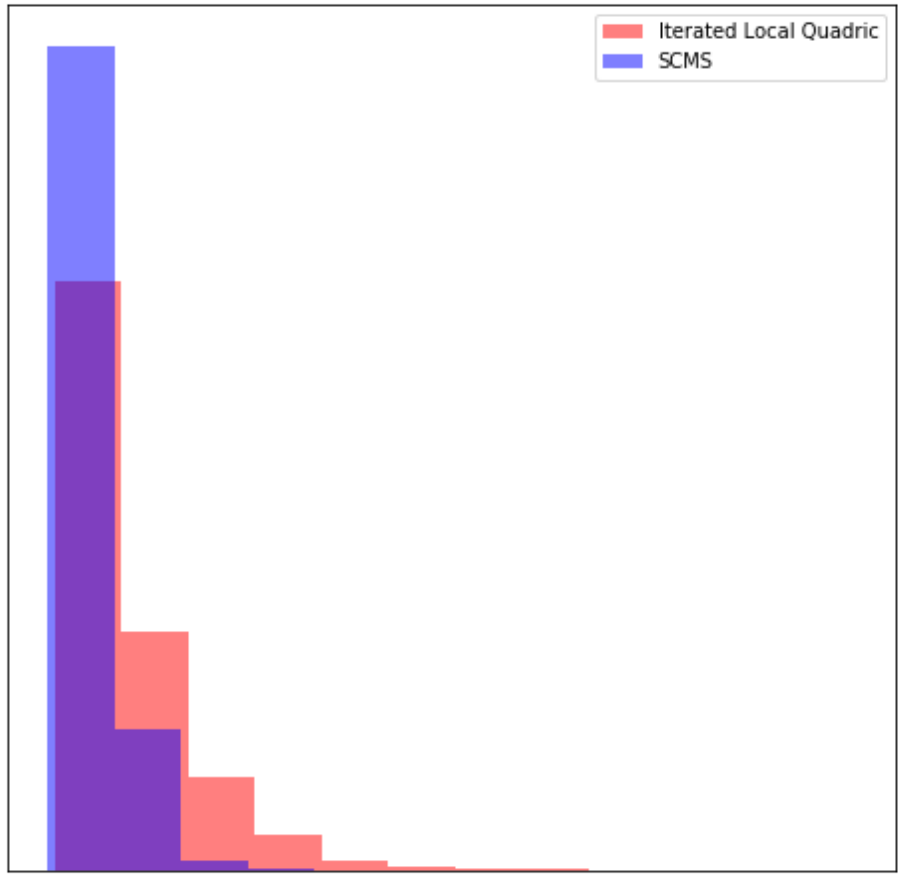

    \centering
    \includegraphics[width=\columnwidth]{local_ODQ_regression/static/{{iterated_comparison}}}
    \caption[Stanford Bunny surface reconstruction algorithm comparison]{Stanford Bunny surface reconstruction algorithm comparison.
        Histogram of the distance to the nearest point in the original bunny (ground truth) for the points found via subspace-constrained mean shift vs points found via the proposed iterated local quadric regression algorithm.}
    \label{localodqfigure:iterated_comparison}
\end{figure}

We then apply subspace-constrained mean shift with a tricube kernel and KNN-selected bandwidth $k=120$ on seed points from a subset of the non-noisy dataset.
 Subspace-constrained mean shift finds points lying on a manifold, and is fully defined by the chosen kernel and bandwidth.
 Results are shown in Figure \ref{localodqfigure:scms_bunny}.
 
It is well known that local modeling methods may behave strangely when extrapolating outside of the range of the original data.
 Thus, the representation found from these seed points may only be of a subset of the full manifold, and it is possible that it extends into space beyond what is shown here.
 The same process is repeated for the iterated local quadric regression algorithm proposed in section \ref{localodqsection:methodology}.
 Results are shown in Figure \ref{localodqfigure:quadric_iterated_bunny}.
 Since the local quadric regression modeling procedure is non-convex, some points have projected a great distance away from the surface due to poor convergence of the individual models.
 Theory and experiments to determine ways to re-seed the individual models to ensure attainment of global optima of non-localized orthogonal quadric models are therefore an important step toward improving the proposed algorithms.

Even with this caveat, it seems as though the local quadric models are severely affected by noise in the model, and have made a tradeoff from bias to variance that is not necessarily beneficial.
 Thus, our hypothesis that local non-linear models may alleviate the problem where local linear models tend to extrapolate linear arrangements away from the surface seems to have proven untrue.
 Indeed, the opposite seems like it may be happening, where the local non-linear models are now capable of also extrapolating non-linear arrangements away from the surface in an undesirable fashion.
 When we compute the distance of points on the reconstructed surface to their nearest counterpart in the original dataset, it seems as though subspace-constrained mean shift performs better, as seen in the histogram of distances in Figure \ref{localodqfigure:iterated_comparison}.
 
Still, the local quadric models have not performed uniformly less well than subspace-constrained mean shift.
 Points located near the tail of the bunny are fit much better with the iterated local quadric regression, although points in the empty space between the foot and the chest are fit less well.
 It is therefore unclear that the proposed method is strictly worse than subspace-constrained mean shift at reconstructing the original surface.
 We note that subspace-constrained mean shift has a tendency to `chop off' the relatively sharp feature of the bunny's tail - a behavior not displayed by the local quadric models.
 Thus, despite the generally poorer performance of local quadrics, one of the primary points of our hypothesis has found some support in this experiment.
 From this light, with some means to more robustly fit the individual local models to ensure discovery of global optima, and also a means to prevent fitting to noise with, e.g. model regularization, iterated local quadric regression might be a formidable surface reconstruction technique.
 However, there are currently many issues that remain to be resolved.
\section{Shapenet Segmentation}

For these experiments we employ a particular point cloud from the shapenet dataset, numbered as `02691156/000176', shown with ground truth labels in Figure \ref{localodqfigure:shapenetplane}.
 This particular point cloud was chosen mostly arbitrarily, and was not cherry-picked to give good results.
 Our hypothesis for this data was that local non-linear models would provide informative features of the surface.
 An obvious prediction of this hypothesis in the context of the current research would be that features extracted from local orthogonal quadric models would be helpful in differentiating different `parts' of a surface.
 The specific features that we use are intermediate calculations in the model fitting procedure, chosen to improve interpretability.
 Specifically, we first perform a rotation and a translation to the quadric surface by eigendecomposing the upper left $3\times 3$ block of $Q$ in equation \ref{localodqequation:quadric}.
 This gives a convenient representation as:

\begin{figure}[!t]
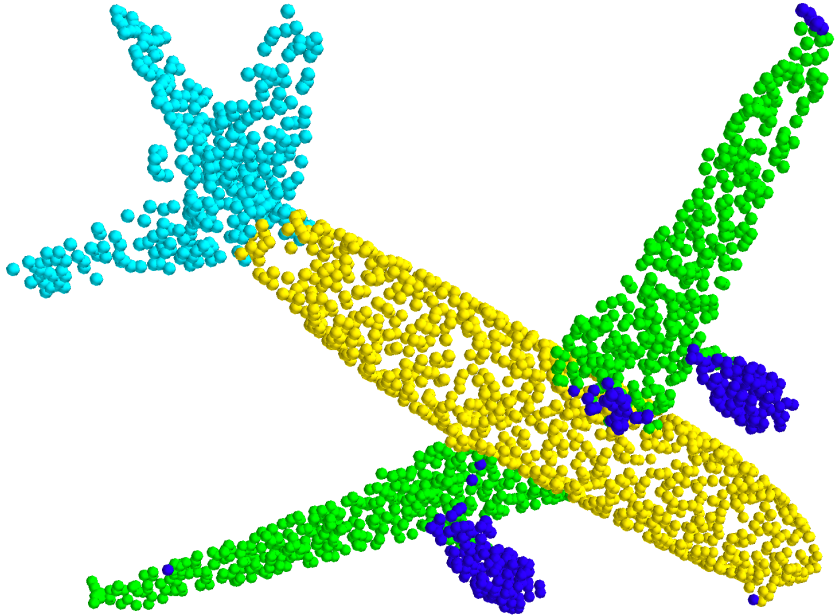

    \centering
    \includegraphics[width=\columnwidth]{local_ODQ_regression/static/{{shapenetplane}}}
    \caption[Shapenet ground-truth segmentation]{Shapenet ground-truth segmentation.
        The ground-truth segmentation of the point cloud that is the subject of our experiments.}
    \label{localodqfigure:shapenetplane}
\end{figure}

\begin{figure}[!t]
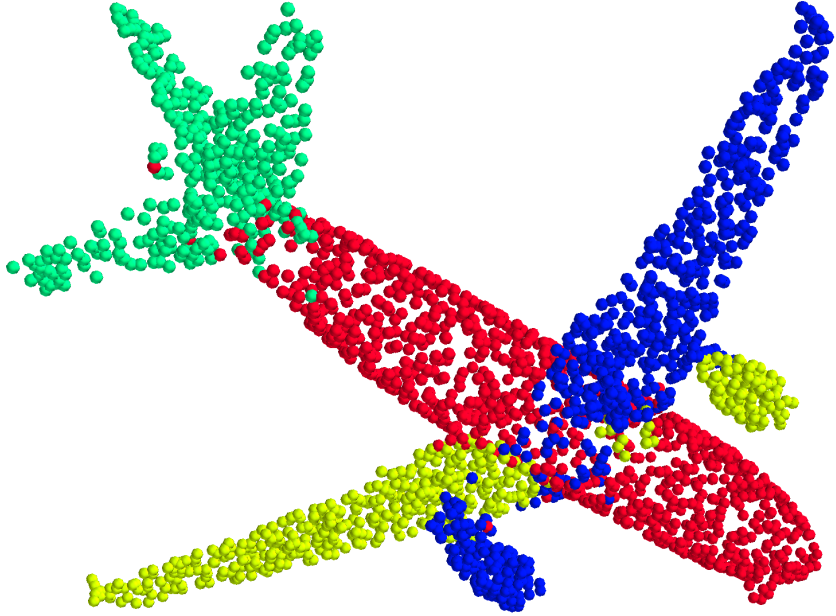

    \centering
    \includegraphics[width=\columnwidth]{local_ODQ_regression/static/{{best_gmm_pred}}}
    \caption[Shapenet GMM segmentation with quadric features]{Shapenet GMM segmentation with quadric features.
        The segmentation of the data in Figure \ref{localodqfigure:shapenetplane} with a Gaussian mixture over a concatenation of the original data and local quadric model extracted curvature features.}
    \label{localodqfigure:best_gmm_pred}
\end{figure}

\begin{figure}[!t]
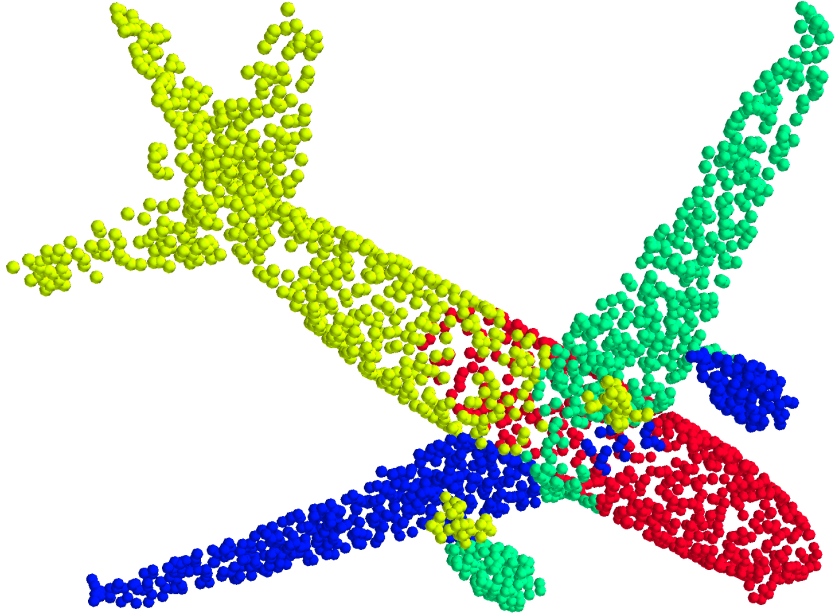

    \centering
    \includegraphics[width=\columnwidth]{local_ODQ_regression/static/{{positiononly_gmm_pred}}}
    \caption[Shapenet GMM segmentation on raw data]{Shapenet GMM segmentation on raw data.
        The segmentation of the data in Figure \ref{localodqfigure:shapenetplane} with a Gaussian mixture over only the original data.}
    \label{localodqfigure:positiononly_gmm_pred}
\end{figure}

\begin{equation} \label{localodqequation:decomposed_quadric}
 Q = E \begin{bmatrix}
    0 & 0       & 0      & \gamma \\
    0 & \alpha  & 0      & \psi \\
    0 & 0       & \beta  & \rho \\
    \gamma & \psi       & \rho      & \delta \end{bmatrix} E^T
\end{equation}
Note that the top left entry in this matrix is 0, which corresponds to the parabolic constraint that we have imposed.
$E$ is a $4\times 4$ block matrix given by:

$$E = \begin{bmatrix}
    U            & \mathbf{0} \\
    \mathbf{0}^T & 1 
        \end{bmatrix} $$
where $U$ is the matrix of eigenvectors associated with the upper left $3\times 3$ block of $Q$.
 The variables $\alpha$ and $\beta$ are the curvature of the fitted local paraboloid in the direction of the two largest components of $U$.
 Note that this is not identical to the curvature of the surface at a query point $q$, especially since the fitted local quadric may or may not be aligned with the underlying surface.
 In other words, the column of $U$ associated with eigenvalue 0 (the first column) is not necessarily normal to the underlying surface.
 Still, we expect that both these curvature features and also the first column of $U$ (which is the direction that the paraboloid `points') might still help describe the local shape of the surface generating the point cloud.
 As such, we study the effects of these features in conjunction with the original coordinates in obtaining a segmentation of the point cloud in Figure \ref{localodqfigure:shapenetplane}.

In order to evaluate the effectiveness of these features, we evaluate the performance of clustering algorithms against the inclusion of these features.
 First, we employ Gaussian mixtures of 4 components (which is the number of components in the ground-truth labeling).
 We use the variation of information metric \cite{meilua2003comparing} to evaluate the `goodness-of-fit' of a particular segmentation against the ground-truth labels.
 We variously fit the mixture model on only the original $x,y,z$ coordinate features, on the first column of $U$ (the `direction' vector), on the feature vector $<\alpha, \beta>$, and on every combination of those three sets of features.
 The segmentation attaining the minimum variation of information score of 1.41 is found when including the original $x,y,z$ coordinates and the $<\alpha,\beta>$ vector, but interestingly not the direction vector.
 This segmentation is shown in Figure \ref{localodqfigure:best_gmm_pred}.
 The segmentation given by only the $x,y,z$ components gives a variation of information score of 2.12, and is shown in Figure \ref{localodqfigure:positiononly_gmm_pred}.
 From the figures, although neither segmentation aligns perfectly with the ground truth segmentation, the inclusion of the curvature features seems to be useful, especially in distinguishing the tail piece from the fuselage.
 Of course, these two segmentations are still quite close, and it is a well-known phenomenon that inclusion of additional features in models can often improve naive performance metrics even when they are not informative.

To further evaluate the method to help rule out this possibility, we have performed mean shift clustering on only the $<\alpha, \beta>$ curvature vectors.
 Since this clustering does not take position into account in any way, the resulting imagery is difficult to interpret.
 Still, by comparing to the ground-truth labeling we can gain some sense of whether the features in any way align with our a priori notions of what makes a good segmentation.
 Mean Shift clustering gives an indeterminate number of clusters, so that comparing two clusterings with the variation of information is difficult to interpret.
 Still, for two given clusterings, we can gather some evidence that they are not totally independent using a chi-square metric \cite{kuncheva2006experimental} over contingency tables.
 When considering the contingency table of ground-truth cluster labels vs labels found via a mean shift clustering of the $<\alpha, \beta>$ vectors, a chi-square test of independence gives a $p$-value of $10^{-89}$.
 Of course, although this gives some evidence that the features are potentially useful, this gives no indication of \emph{how} useful these features are.
 Still, it seems to suggest that the improved results seen in Figure \ref{localodqfigure:best_gmm_pred} are not simply due to the addition of irrelevant features.

%% file: local_ODQ_regression/conclusion/conclusion.tex
\section{Conclusion} \label{localodqsection:conclusion}

In this paper we describe a method for local orthogonal quadric regression.
 We employ this method as a feature extraction method to obtain real values descriptive of the local shape of the underlying surface generating the point cloud.
 We have tested this method against a point cloud in the Shapenet dataset.
 In these experiments, we provide evidence that these features are descriptive of a priori human notions of shape segmentations in point clouds, and also that they can be successfully used in clustering algorithms to obtain results that more closely align with human labelings.
 
We further employ this method in an iterative method akin to subspace-constrained mean shift in order to obtain a surface reconstruction of unstructured point clouds.
 We have tested this method on the canonical Stanford Bunny from the Stanford 3D scanning repository.
 Although results are promising, it remains unclear whether iterated orthogonal quadric regression gives a better or worse surface reconstruction than subspace-constrained mean shift.
 We suspect that continued research into creating robust versions of orthogonal quadric regression will much improve the success of the model family's localization.

In the future, we anticipate extension of local methods to other forms of orthogonal-distance regression models.
 We note that the orthogonal-distance loss is an appealing metric for creating models based on arbitrary families of parametrized functions.
 We further note that the general method for projection onto quadric surfaces in \cite{lott2014direct} is applicable in arbitrary dimensions, and that therefore local quadric regression can be extended to applications on higher dimensional data.
 Such applications might include dimensionality reduction via a form of `surface-reconstruction' in $n$D, or feature extraction techniques similar to those employed in this work.
 Although the physical meaning may be lost in higher dimensional datasets, the curvature of a surface approximating the local data structure seems as though it may still be a useful feature for further learning.

%% file: sections/conclusion/conclusion.tex
\chapter{CONCLUSIONS} \label{chapter:conclusions}

We conclude this thesis by summarizing the trajectory of the research, how the individual projects have fit into this trajectory, both the successes and failures we have had and the general lessons to be learned from both.

Recall from Chapter \ref{chapter:intro} that the purpose of our research is to develop a feature learning technique that results in features summarizing local properties of the data, which is applicable to a wide variety of datasets.
 Section \ref{section:computervision} briefly describes convolutional neural networks: one massive success story of locally-derived features being useful for further learning.
 However, those methods rely on the grid-like nature of image data, and although suggestive that local features may be useful in other domains, are not directly applicable.
 One intuition that has received but small and often indirect attention in the literature, described in Section \ref{section:localmodelparameters}, is the use of local models for this task.
 Specifically, the various characteristics of, and especially the parameters of local models provide locally-derived features that often have a ready interpretation and provide corresponding information about the local structure of data.
 Applications using features derived thus, even for direct real-world applications, are relatively sparse in the literature.
 Applications using these features for further learning are even more rare.
 There currently exists no automated pipeline for their use as black-box generic feature extractors for arbitrary tasks.
 The closest such attempt is in \cite{bay2004framework}, but their investigations are limited to anomaly detection in time-series with local ARIMA models.
 In order to further our understanding of generic local-model-based feature learners based on non-trivial model families, we have, in this document, investigated in numerous ways the ability of local models to provide features summarizing the local structure of data.
 
Specifically, in chapters \ref{chapter:localgpr}-\ref{chapter:localodq} we have extended the local modeling paradigm to several model families.
 We have demonstrated that local model feature transformations are feasible for a wide range of data types.
 We have demonstrated that local model feature transformations are applicable to both supervised and unsupervised applications.
 We have shown that features extracted via local modeling often align with concepts that we might a priori expect to be useful and informative in a particular context, and that are not trivially reflected in the raw data.
 Lastly, we have provided numerous examples where local model features improve upon the results of subsequent modeling when used in conjunction with raw data.
 
In Chapter \ref{chapter:localgpr} we have extended the local modeling paradigm to apply to Gaussian process regression.
 Although Gaussian processes do not admit the naive weighting scheme typical to local learning procedure given in equation \ref{equation:locallearning}, we demonstrate that Gaussian processes do admit a weighting scheme, and that this scheme is viable as a pathway to localization of the model family.
 This provides evidence that other model families not admitting the standard weighting scheme might be amenable to local learning extensions.
 We also hypothesized that the parameters of local Gaussian processes might be useful as a feature in analysis of time series.
 In various experiments, we show that the naive interpretation of the local Gaussian process parameters translates into their usefulness as a feature in a transfer learning scenario.
 This provides evidence for the general hypothesis of this dissertation, that features extracted from complex local model families might be useful for transfer learning purposes.
 This work also identifies a number of potential pitfalls with the method, foremost of which is potential discontinuities in the argmin operator as the weights for data used to train a local model are changed as we move about the input space.
 For simple model families with convex optimization surfaces argmin is both continuous and differentiable (which we discuss in depth in Chapter \ref{chapter:localsvm}).
 Thus, when the loss functions and chosen kernel are differentiable, the proposed feature extraction method is a differentiable map over the input space, which is both theoretically and practically convenient.
 For more complex model families this is not necessarily the case.
 Still, we show in this work that it is often possible to mitigate this problem via suitable constraints on the family of models.

In Chapter \ref{chapter:localsvm} we have extended existing theory concerning localized classifiers to allow discovery of the decision surface of the system of localized models.
 We specifically focus on the extraction of the distance of a point to the nearest stationary point of a local modeling scheme based on projections to decision surfaces of locally-trained support vector machines.
 We show that features extracted in this fashion can be used as a precursor to pseudo-probabilistic scores for a system of local classifiers.
 In various experiments, we show that these features can be used to improve multi-class classification accuracy on various commonly studied toy and real-world problems.
 This also provides evidence for the general hypothesis of this dissertation, that complex features beyond simple predictions may be extracted from local model families and used for further learning.

In Chapter \ref{chapter:localw2v} we have extended the local modeling paradigm to apply to word2vec word embedding models.
 We use these local models to extract real-valued features representative of a term's usage in a particular time and place.
 We anticipated that these features could be employed naively to detect significant local events as they evolve in time.
 We furthermore gather a dataset to test our hypotheses, for which available public datasets were deemed insufficient.
 In experiments, we examine the evolution of the extracted features in time as the observed events unfold.
 Although our features were insufficient to directly provide information about the emerging events, we have shown that these features are related to other pertinent features that are typically important in event detection systems.
 Furthermore, we show that even extremely complex model families with millions of parameters may be successfully localized in a meaningful way, overcoming both extreme non-convexity in the optimization surface and extremely low data-to-parameter ratio for the individual local models.
 This provides evidence for our hypothesis that even very complex model families may be successfully localized, and that interesting features may be extracted from them.
 This work also highlights a caveat that local model features for complex model families may exhibit great complexity in their behavior, despite our naive interpretations of the `meaning' of a particular feature.
 This implies that future work on the topic must either take great care in hypothesizing the potential behavior of extracted features, or be prepared to perform a significant amount of postprocessing on the extracted features to further reduce them into a small set of usable values for a particular task.

In Chapter \ref{chapter:localodq} we have extended the local modeling paradigm to apply to orthogonal quadric regression models.
 We use these local models to extract real-valued features representative of the local `shape' of the generating surface for point clouds in 3D.
 These features are related but not identical to the local curvature and orientation of the underlying surface.
 We show that these features can be used to improve upon the alignment of the results of clustering methods with human-labeled segmentations of point clouds. 
 We also show that these models can be used to successfully reconstruct an approximation to the original surface.
 It remains unclear exactly how the method stacks up against existing methods for surface reconstruction, but our results do not seem to suggest that orthogonal quadric regression models are a good fit for this task.
 A large barrier to determining the quality of these surface reconstructions was the relative non-robustness of the non-local orthogonal quadric regression model.
 Although our model implementation proved likely to converge, it was difficult to determine if the found solution was the global or merely a local minimum, especially compared to the extremely well-studied and well-optimized PCA models against which it was compared.
 For a single individual model, this is unlikely to be problematic, but when training thousands of models across various configurations of a single dataset, it becomes incredibly likely to stumble upon situations where non-robust model families fail to converge or to find the global optimum.
 We note that we experienced this same problem to a lesser extent with our work on localized support vector machines described in Chapter \ref{chapter:localsvm}, employing well-known and respected publicly available implementations.
 Although models almost always converged appropriately, the nature of local modeling exploration over many slightly different configurations of a single dataset created a situation where if failure \emph{could} occur then it \emph{would} occur.
 Fortunately for our experiments with support vector machines, situations when the model training failed were relatively easy to isolate.

Taking all of this from a birds-eye-view, our experiments have met with modest, but not tremendous success.
 Our anticipation that features extracted from local models can be useful in a transfer learning context has largely proven true as pointed out in Chapters \ref{chapter:localgpr} - \ref{chapter:localodq}.
 Even so, we have discovered that there are a number of high hurdles for implementing such a scheme in any given problem context.
 
One observation that can be seen as both a pro and a con is the problem of choosing from an infinite set of possible model families an appropriate set of models to localize that fits the real-world context of a problem.
 Of course, this is a problem that besets any modeling problem, but the relative complexity of using local models for feature extraction vs using a model for direct predictive purposes compounds the difficulty.
 On one hand, the immensity of the space increases the likelihood that there exists a model family that gives meaningful features for any given task.
 On the other hand, the immensity of the space means that isolating a tractable model family for a particular task is not necessarily feasible.
 In our experiments, we have operated with both the model family and the task in mind, asking ``For which model family/task pairs might a local model feature extraction prove useful?''
 This is a much less difficult objective than starting with a task and finding a model family that fits.
 Still, the broad set of applications that we have found for local model feature transformations suggests that this may be easier than it sounds.

We have also found that even when our a priori interpretation of a particular feature extracted via local modeling seems to align with a particular task, good results are not guaranteed.
 Therefore, when considering the use of local modeling as a feature extraction step, it is very important to carefully examine the relationship between the model family proposed for localization, the features we propose to extract from it and the task at hand. 
 Specifically, it is crucial to consider how changes to the training data might affect those features in perhaps non-intuitive ways.

The observations above are compounded by the relative computational complexity of local modeling.
 In theory, local modeling has attractive big-O bounds.
 In fact, for a finite-support kernel with variable bandwidth set at the $k$th-nearest-neighbor distance, building a local model at each of the points in a set of training data of size $n$ is $O(nf(k) + n\log n)$, where $f(k)$ is the runtime of the particular algorithm.
 The $n\log n$ is from the construction of a ball tree \cite{omohundro1989five} and from $n$ queries at $\log n$ time each.
 Note that $f(k)$ does not depend on $n$, the size of our training data, so that if $k$ is small compared to $n$, this naively compares favorably against many global model families.
 Indeed, in our experiments with local Gaussian processes in Chapter \ref{chapter:localgpr}, the localized models were more performant than global models over the same data.
 This is partly due to the fact that Gaussian process regression is notoriously inefficient ($O(n^3)$), related to the inversion of an $n\times n$ matrix, setting a rather low bar.
 
Despite this seemingly nice theoretical runtime, in general, complex local models on complex datasets have proven to be very resource intensive.
 When the dataset under consideration is not simple (i.e. is of high dimension), the curse of dimensionality has a large effect on ball tree performance.
 Since nearest-neighbor search is a key ingredient of local modeling, this directly affects the performance of any local modeling recipe when neighbor distances are computed in a high-dimensional space.
 When the model family that we are localizing is complex, $f(k)$ becomes a major factor in the run time of fitting a full set of local models.
 Notably, whereas iteratively fitting PCA models (subspace-constrained mean shift) is not a particularly fast process, results on modestly-size point clouds are returned in two orders of magnitude less time than for our local orthogonal quadric implementation.
 This speed reduction is partly because there is approximately an order of magnitude more operations in the computation of the loss for the model.
 It is also due to the fact that the model is trained via a variant of gradient descent, which may take many iterations to converge.
 Such an implementation has fierce competition from local PCA models which are solved via a simple eigenvector decomposition which have extremely optimized iterative solutions that are known to converge rapidly \cite{anderson1999lapack}.
 
The practical runtime requirements of local modeling create a situation where, although training a set of models over a dataset for a \emph{particular} localized model family may be on the edge of feasibility, training multiple such sets is prohibitively expensive.
 This poses a problem since many model families have hyperparameters that must be selected between.
 In practical situations, we are therefore often confined to selecting hyperparameters according to a priori intuitions, which is obviously suboptimal.
 Even when the model family has relatively few hyperparameters, the local modeling recipe itself often involves at least one hyperparameter: the kernel bandwidth.
 In real applications, one might wish to choose between a set of fixed bandwidths, a set of variable-bandwidth strategies, and various kernels.
 Automating these choices is rendered impractical when the model family is of significant complexity or the dimensionality of the data is high.

Despite these difficulties, there is good reason to believe that a means for extracting robust and meaningful local information from data has the potential to be a powerful technique in machine learning.
 Our work has clearly demonstrated that local models are one way to accomplish this feat that is widely applicable to many domains.
 Still, problems remain with the technique, especially regarding efficient implementations.
 We hope, therefore, that our research demonstrating the feasibility of extracting useful local features using local modeling will pave the way for more advanced methods of summarizing local information as a general technique in machine learning.
 We further hope that the successes demonstrated here will inspire further research into efficient nearest neighbor search, and also into efficient implementations of intermediate-complexity global model families.
 For example, research into efficient and robust global orthogonal quadric regression is unlikely to prove valuable on its own - but as an ingredient of a local modeling pipeline, it is incredibly important.
 Lastly, we anticipate that the breadth of applicability of local models demonstrated in this document will inspire researchers in many fields to consider possible applications of localizing global model families that they might already use.

%% file: local_gpr/appendix1/weighted_gpr.tex
Recall the pdf of the multivariate Gaussian distribution:

\begin{equation} \label{equation:gaussianpdf}
G(\mathbf{y}|\boldsymbol\Sigma) = \frac{\exp\left(-\frac 1 2 ({\mathbf{y}}-{\boldsymbol\mu})^\T{\boldsymbol\Sigma}^{-1}({\mathbf y}-{\boldsymbol\mu})\right)}{\sqrt{(2\pi)^k\det(\boldsymbol\Sigma)}}
\end{equation}

The next few proofs show that weighting the variance of a variable in the covariance matrix by a positive integer weight is equivalent to weighting a variable by adding a number of ``independent'' copies of that variable to our design matrix, and replicating the observations of that variable.
 This gives a means of weighting the observations in a Gaussian process.

Consider a set of variables $V_i$, $1 \leq i \leq n$ with covariance matrix $\mathbf{B}_0$ consisting of entries $B_{ii}$.  
 We add $w$ independent copies of a new variable $U$ with variance $uw$. 
 If we let $\mathbf{b}$ be the column vector of covariances of $U$ with the $V_i$, the complete covariance matrix including these new variables is:

$$ \mathbf{B}_w = 
    \begin{bmatrix}
        \mathbf{B}_0         & \mathbf{b} & \ldots             & \mathbf{b} \\
        \mathbf{b}^\T         &            &                    &            \\
        \vdots               &      \multicolumn{3}{c}{uw \mathbf{I}_w}     \\
        \mathbf{b}^\T         &            &                    &            
    \end{bmatrix}
$$

\begin{lemma} \label{lemma:almostequaldeterminants}
 $ \det(\mathbf{B}_w)/(uw)^w = \det(\mathbf{B}_1)/u $ for $w \geq 1$
\end{lemma}

\begin{proof}

Note that:

\begin{equation} \label{equation:lessbs}
 \lvec{\mathbf{b}}{l} \lvec{\mathbf{b}}{l}^\T = l \mathbf{b} \mathbf{b}^\T
\end{equation}

The following formula from Schur \cite{schur1917uber}:

\begin{equation} \label{equation:schurdeterminant}
 \det\begin{pmatrix}W& X\\Y & Z\end{pmatrix} = \det(Z)\det(W - X Z^{-1} Y)
\end{equation}

Implies that:
 
\begin{eqnarray*}
	\det(\mathbf{B}_w) & = & \det(uw \mathbf{I}_w)\det(\mathbf{B}_0 -     \\
	&& \frac{1}{uw} \lvec{\mathbf{b}}{w} \mathbf{I}_w \lvec{\mathbf{b}}{w}^\T)
\end{eqnarray*}

By Equation \ref{equation:lessbs}:

 $$ \det(\mathbf{B}_w) = (uw)^w\det(\mathbf{B}_0 - \frac{1}{u} \mathbf{b} \mathbf{b}^\T) $$

By substituting $w=1$:

 $$ \det(\mathbf{B}_1) = u\det(\mathbf{B}_0 - \frac{1}{u} \mathbf{b} \mathbf{b}^\T) $$

Solving for $\det(\mathbf{B}_0 - \frac{1}{u} \mathbf{b} \mathbf{b}^\T)$ in these two equations gives the desired result. 

\end{proof}

\begin{lemma} \label{lemma:Binverse}
If $\mathbf{B}_1$ is invertible, and we write:

$$ \mathbf{B}^{-1}_1 = 
    \begin{bmatrix}
        \mathbf{A}    & \mathbf{c} \\
        \mathbf{c}^\T  & z 
    \end{bmatrix}
$$

Letting $\mathbf{J}_m$ denote an $m\times m$ matrix of all ones, then:

$$ \mathbf{B}^{-1}_w = 
    \begin{bmatrix}
        \mathbf{A}           & \mathbf{c}/w & \ldots                                         & \mathbf{c}/w \\
        \mathbf{c}^\T/w       &              &                                                &              \\
        \vdots               &              \multicolumn{3}{c}{(zu-1)/(w^2u)\mathbf{J}_w + \mathbf{I}_w/(uw)}              \\
        \mathbf{c}^\T/w       &              &                                                &            
    \end{bmatrix}
$$

\end{lemma}

\begin{proof}

By definition:

\begin{eqnarray} \label{equation:B1inverse}
	I_{n+1} &=& \mathbf{B}_1 \mathbf{B}^{-1}_1 =
    \begin{bmatrix}
        \mathbf{B}_0 \mathbf{A} + \mathbf{b} \mathbf{c}^\T    & \mathbf{B}_0 \mathbf{c} + z\mathbf{b}      \\
        \mathbf{b}^\T \mathbf{A} + u \mathbf{c}^\T             & \mathbf{b}^\T \mathbf{c} + zu               
    \end{bmatrix} \\ &=&
    \begin{bmatrix}
        \mathbf{I}_n    & \mathbf{0}_n      \\
        \mathbf{0}_n^\T  & 1          
    \end{bmatrix}
\end{eqnarray}

By multiplying through, $\mathbf{B}_w \mathbf{B}^{-1}_w$ simplifies:

$$
    \begin{bmatrix}
        \mathbf{B}_0 \mathbf{A} + \mathbf{b} \mathbf{c}^\T   & \mathbf{B}_0 \mathbf{c}/w + z\mathbf{b}/w & \ldots & \ldots \\
	    \mathbf{b}^\T \mathbf{A} + uw\mathbf{c}^\T/w         & \frac{\mathbf{b}^\T \mathbf{c} + (zu-1+w)}{w}    & \frac{\mathbf{b}^\T \mathbf{c} + (zu-1)}{w} & \ldots \\
	    \vdots                                              & \frac{\mathbf{b}^\T \mathbf{c} + (zu-1)}{w}      & \frac{\mathbf{b}^\T \mathbf{c} + (zu-1+w)}{w}& \vdots \\
        \vdots                                              & \vdots & \ldots & \ddots
    \end{bmatrix}
$$

Applying the equalities from Equation \ref{equation:B1inverse}:

$$ \mathbf{B}_w \mathbf{B}^{-1}_w =
    \begin{bmatrix}
        \mathbf{I}_n   & \mathbf{0}_n & \ldots & \ldots \\
        \mathbf{0}_n^\T & 1            & 0      & \ldots \\
        \vdots         & 0            & 1      & \vdots \\
        \vdots         & \vdots       & \ldots & \ddots \\
    \end{bmatrix}
    = \mathbf{I}_{n+w}
$$

\end{proof}

\begin{lemma} \label{lemma:equalproducts}

Suppose that $\mathbf{y}$ is a $n\times 1$ column vector.  If we take:

$$ \mathbf{x}_w = \lvec{x}{w} $$

$$ \mathbf{y}_w = 
    \begin{bmatrix}
        \mathbf{y} \\
        \mathbf{x}_w 
    \end{bmatrix}
$$

Then for all $w \geq 1$: 

$$ \mathbf{y}_w^\T \mathbf{B}_w^{-1} \mathbf{y}_w = \mathbf{y}_1^\T \mathbf{B}_1^{-1} \mathbf{y}_1  $$ 

\end{lemma}

\begin{proof}

This equality can be shown by simply multiplying out and reducing the left side using the formula from Lemma \ref{lemma:Binverse} to see that all of the $w$ terms fall out.

\end{proof}

\begin{theorem} \label{theorem:weightedvariablegaussian}

Suppose that $\mathbf{B}_1$ is actually a covariance matrix of some set of real valued random variables, then:

\begin{equation} \label{equation:weightedvariablegaussian}
 G(\mathbf{y}_w | \mathbf{B}_w) = G(\mathbf{y}_1 | \mathbf{B}_1) u^{(w-1)/2} w^{w/2} (2\pi)^{(w-1)/2}
\end{equation}

\end{theorem}

\begin{proof}

From Lemma \ref{lemma:equalproducts} we see that the numerators of the two pdfs are identical.
 In the denominator of the left-hand side of the above equality, by Lemma \ref{lemma:almostequaldeterminants} and some basic algebra we have:

$$
 \sqrt{(2\pi)^{n+w}\det(\mathbf{B}_w)} =
$$
$$
 \sqrt{(2\pi)^{n+1}\det(\mathbf{B}_1)} u^{(w-1)/2} w^{w/2} (2\pi)^{(w-1)/2}
$$

These two facts combine to show the desired result.

\end{proof}

The nice feature of this result is that it gives us a means to weight the variables of a multivariate gaussian without having to actually duplicate rows; an operation which would cause considerable overhead when computing the inverse of the covariance matrix.
 Furthermore, although the left-hand side of Equation \ref{equation:weightedvariablegaussian} is constrained to positive integral values of $w$, there is no technical reason to so constrain the right-hand side.
 Thus, we gain the ability to weight our variables by arbitrary positive real values.
 This is particularly helpful in the context of Gaussian Processes, since the `variables' are actually observations.
 If we optimize our Gaussian Process parameters by maximizing the Gaussian pdf over the covariance function parameters, $\theta$, Equation \ref{equation:weightedvariablegaussian} can be used to weight the observations.

Lastly, note that these results only require that $\mathbf{B}_1$ be a valid covariance matrix.
 Specifically, $\mathbf{B}_0$ may itself be the result of this variable replication scheme.
 We can therefore apply the result from Theorem \ref{theorem:weightedvariablegaussian} recursively.

Till now we have for convenience included the weights $w_i$ as components of the replicated variables.
 In order to use this formula, we would want to apply the weights to the \emph{original} matrix, as in the following corollary:

\begin{corollary}

Suppose that $\mathbf{w} = \begin{bmatrix} w_1 & w_2 & \ldots & w_n \end{bmatrix}$ is a vector of weights, and a sequence of covariance matrices $\boldsymbol\Sigma_i$, $0 \leq i \leq n$ is of the following form:

$$ \boldsymbol\Sigma_i = 
    \begin{bmatrix}
        \boldsymbol\Sigma_{i-1}         & \mathbf{b}_i & \ldots             & \mathbf{b}_i \\
        \mathbf{b}_i^\T         &            &                    &            \\
        \vdots               &      \multicolumn{3}{c}{u_i \mathbf{I}_{w_i}}     \\
        \mathbf{b}_i^\T         &            &                    &            
    \end{bmatrix}
$$

Where $\boldsymbol\Sigma_0$ is a $0\times 0$ matrix. Let $\mathbf{C}$ be the corresponding `weighted' matrix:

$$ \mathbf{C} = 
    \begin{bmatrix}
        u_1/w_1             & \mathbf{b}_2 & \ldots             & \mathbf{b}_n \\
        \mathbf{b}_2^\T     & u_2/w_2      &                    &              \\
        \vdots              &              & \ddots             &              \\
        \mathbf{b}_n^\T     &              &                    & u_n/w_n           
    \end{bmatrix}
$$

Let:

$$ \mathbf{y} = \begin{bmatrix} y_1 & y_2 & \ldots & y_n \end{bmatrix} $$

and

$$ \mathbf{y_w} = \begin{bmatrix} \lvec{y_1}{w_1} & \lvec{y_2}{w_2} & \ldots & \lvec{y_n}{w_n} \end{bmatrix} $$

Then:

\begin{equation} \label{equation:multipleweightedvariablegaussian}
 G(\mathbf{y_w} | \boldsymbol\Sigma_n) = G(\mathbf{y} | \mathbf{C}) \prod_{i=1}^n u_i^{(w_i-1)/2} w_i^{w_i/2} (2\pi)^{(w_i-1)/2}
\end{equation}

\end{corollary}

\begin{proof}

This can be seen by induction on $n$.  
 Setting $\mathbf{B}_0$ to be the empty $0 \times 0$ matrix in Equation \ref{equation:weightedvariablegaussian} gives the base case.
 Setting $\mathbf{B}_0$ to be $\boldsymbol\Sigma_{k-1}$ on the left-hand side and the first $k-1$ rows and columns of $\mathbf{C}$ on the right-hand side gives the induction step.

\end{proof}

Note that when optimizing this formula over parameters $\theta$ for the covariance function $K_\theta$ of a Gaussian Process Regression, some of these terms do not depend on $\theta$ and can therefore be omitted:

$$   GPR(\mathbf{y}, \mathbf{X}) = $$
\begin{equation*} \label{equation:weightedgpr}
 \max_\theta \left(
 \begin{array}{l}
   G(\mathbf{y}|\boldsymbol\Sigma = K_\theta(\mathbf{X}) - \text{diag}\left(\frac{w_i-1}{w_i}K_\theta(x_i, x_i)\right)) \\ * \prod_{i=1}^n K_\theta(x_i, x_i)^{(w_i-1)/2}
 \end{array}
 \right)
\end{equation*}

We employ this form in order to weight the observations of a Gaussian Process Regression.
 Ordinarily, the `noise' $K_\theta(x, x) = WN$ is assumed to not vary over $x$.
 If so, the formula simplifies somewhat:

\begin{equation*} \label{equation:weightedgpruniformnoise}
 \max_\theta \left(
 \begin{array}{l}
   G\left(\mathbf{y}|\boldsymbol\Sigma = K_\theta(\mathbf{X}) - WN\text{diag}\left(\frac{w_i-1}{w_i}\right)\right) \\ * \prod_{i=1}^n WN^{(w_i-1)/2}
 \end{array}
 \right)
\end{equation*}

If we opt to optimize the log likelihood instead, this can be formulated as a sum:

\begin{equation*} \label{equation:logweightedgpruniformnoise}
   \max_\theta \left(\begin{array}{ll} \log\left(G\left(\mathbf{y}|\boldsymbol\Sigma = K_\theta(\mathbf{X}) - WN\text{diag}\left(\frac{w_i-1}{w_i}\right)\right)\right) + \\ \log(WN) (-n/2 + \sum_{i=1}^n \frac{w_i}{2}) \end{array}\right)
\end{equation*}

If we adopt a convention that the average weight is 1 (i.e., $\sum w_i = n$), then the second term falls out:

\begin{equation*} \label{equation:simplifiedlogweightedgpruniformnoise}
   \max_\theta \left(\log\left(G\left(\mathbf{y}|\boldsymbol\Sigma = K_\theta(\mathbf{X}) - WN\text{diag}\left(\frac{w_i-1}{w_i}\right)\right)\right)\right)
\end{equation*}

This is convenient because it implies that we need only divide the diagonal of the covariance matrix by the weight vector, and then perform the usual optimization process.


